\documentclass[lettersize,journal]{IEEEtran}

\usepackage[numbers,sort&compress]{natbib}
\usepackage{amsmath,amssymb,amsthm}
\usepackage{amsfonts,bm}
\usepackage{caption}
\usepackage{subcaption}
\usepackage{algpseudocode} 
\usepackage{graphicx}
\usepackage{balance}
\usepackage{enumitem}
\usepackage{textcomp}
\usepackage{xcolor}
\usepackage{color}
\usepackage{multirow}
\usepackage{url}
\usepackage{booktabs}
\usepackage{hyperref}
\hypersetup{hidelinks}
\usepackage{url}
\usepackage{booktabs}
\usepackage{amsmath,bm}
\usepackage{graphicx}
\usepackage{enumitem}
\usepackage{color}
\usepackage{makecell}
\usepackage{multirow}
\usepackage{colortbl}
\usepackage{wrapfig,lipsum}
\usepackage{array}

\usepackage{amsmath}
\newtheorem{proposition}{Proposition}
\newtheorem{corollary}{Corollary}     
\theoremstyle{definition}

\usepackage{etoc}
\usepackage{algorithm}   
\usepackage{xcolor}
\allowdisplaybreaks[4]
\usepackage{bm}
\usepackage{amsmath}
\def\BibTeX{{\rm B\kern-.05em{\sc i\kern-.025em b}\kern-.08em
    T\kern-.1667em\lower.7ex\hbox{E}\kern-.125emX}}
\usepackage{balance}
\usepackage{newfloat}
\usepackage{listings}
\usepackage{booktabs}
\usepackage{float}

\newtheorem{theorem}{Theorem}

\def\method{S$^2$PLR} 
\begin{document}

\title{Safe-Subspace Pseudo-Label Refinement for Source-Free Graph Domain Adaptation}
\author{Yingxu Wang, Xinwang Liu, Siyang Gao, Nan Yin
\IEEEcompsocitemizethanks{
\IEEEcompsocthanksitem Yingxu Wang is with Department of Computer Science and Engineering, Chinese University of Hong Kong. E-mail: yingxv.wang@gmail.com.
\IEEEcompsocthanksitem Xinwang Liu is with College of Computer, National University of Defense Technology, Changsha, 410073, China. E-mail: xinwangliu@nudt.edu.cn.
\IEEEcompsocthanksitem Siyang Gao is with Department of Data Science, City University of Hong Kong. E-mail: siyangao@city.edu.hk.
\IEEEcompsocthanksitem Nan Yin is with The Education University of Hong Kong, Hong Kong. E-mail: yinnan8911@gmail.com.
}
}


\maketitle

\begin{abstract}
Source-free graph domain adaptation (SF-GDA) aims to adapt source-trained graph models to unlabeled target graphs when source graphs are no longer accessible. A central obstacle is pseudo-label reliability: under feature and topological shifts, source-induced predictions may become confidently wrong, and indiscriminate self-training can amplify systematic errors through graph message passing. This paper studies SF-GDA from a selective pseudo-labeling perspective. Instead of assuming globally bounded pseudo-label noise over the entire target domain, we identify a confidence-consistent safe subspace on which pseudo-label noise can be controlled under restricted posterior discrepancy, and derive a target-risk decomposition that separates safe-subspace fitting error, selected-label noise, and uncertain-set risk. Guided by this analysis, we propose Safe-Subspace Pseudo-Label Refinement (\method{}), a source-free graph adaptation framework that applies hard pseudo-label supervision only to target graphs supported by both semantic and structural evidence. Specifically, \method{} estimates semantic reliability using source-committee confidence and disagreement, learns a target-intrinsic structural representation via graph contrastive learning, verifies pseudo-labels through neighborhood consistency, and exploits the remaining uncertain samples with noise-tolerant soft regularization rather than unreliable hard labels. Experiments on image and real-world graph benchmarks under different domain shifts demonstrate that \method{} achieves robust and competitive performance across diverse source-free transfer settings.
\end{abstract}


\section{Introduction}

Graph data are ubiquitous in social networks, biological systems, and molecular discovery, yet models trained on one graph domain often suffer substantial performance degradation when deployed on another because distribution shifts may arise in both node attributes and topology \citep{ long2018transferable,  wu2020comprehensive, xie2022self,jin2022empowering,jin2020graph,wang2024degree}. Graph domain adaptation (GDA) addresses this challenge by transferring knowledge from labeled source graphs to unlabeled target graphs and has become an important paradigm for robust graph representation learning \citep{liu2024revisiting,yin2023coco, wang2026dsbd, wang2026riemannian}. However, most existing GDA methods assume that source graphs remain accessible during adaptation, which is often unrealistic because of privacy, proprietary, and storage constraints. This limitation motivates source-free domain adaptation, where adaptation is performed using only a pretrained source model and unlabeled target data \citep{li2024survey,zhao2024stable, tang2024source}. Extending this setting to graph-structured data gives rise to source-free graph domain adaptation (SF-GDA), a substantially more challenging problem because domain shift affects not only feature distributions but also relational structures and message-passing behavior \citep{kim2021domain, mao2024sourcefreegraph,luo2024gala,luo2024rankalign,wang2026nested}.

A central difficulty in SF-GDA is that target adaptation must rely on pseudo supervision induced by a fixed source hypothesis. Such pseudo-labels can be systematically biased under target-side structural shift, and recent source-free adaptation analyses show that full-domain self-training may enter an unbounded-noise regime in which a source classifier is confidently wrong on part of the target domain \citep{yi2023sfda_noisy,xu2025unraveling,wang2024tackling}. This failure mode is especially consequential for graph data: an erroneous pseudo-label is not an isolated classification mistake, but can be propagated through neighborhood aggregation and may distort both local structural consistency and global representation geometry. Thus, the key question in SF-GDA is not merely how to refine pseudo-labels, but which target samples can be trusted for hard supervision and which should be handled only through softer target-side regularization.

Existing SF-GDA methods have made important empirical progress through topology-aware alignment, collaborative pseudo-label refinement, and structure-preserving self-supervision \citep{mao2024sourcefreegraph,luo2024rankalign,luo2024gala,yang2021exploiting}. Nevertheless, most of them do not explicitly characterize the reliability region on which pseudo-label training is theoretically justified. Confidence filtering alone is insufficient because a source model may make overconfident target errors, whereas neighborhood consistency alone can be misleading if the target representation is geometrically distorted. Moreover, discarding all uncertain samples wastes target distributional information, while assigning hard pseudo-labels to them risks negative transfer. These observations call for a selective adaptation framework that (i) identifies a reliable target subspace using both semantic confidence and structural evidence, (ii) learns a compact target-intrinsic geometry for neighborhood verification, and (iii) regularizes the remaining uncertain samples without treating their pseudo-labels as ground truth.

In this work, we develop such a framework by revisiting pseudo-label noise in source-free graph domain adaptation from a selective-risk perspective. Rather than claiming that pseudo-label noise is globally bounded over the entire target domain, we establish sufficient conditions under which it is bounded on a confidence-consistent safe subspace. Our analysis explicitly separates two aspects that are often conflated: source confidence and posterior discrepancy determine the population-level pseudo-label noise bound, while neighborhood consistency and geometric compactness provide target-only evidence for identifying reliable samples in practice. This distinction avoids overclaiming beyond the assumptions and leads to an actionable risk decomposition: hard pseudo-label supervision should be restricted to the safe subspace, and the residual uncertain set should be controlled by soft regularization.
Motivated by this selective-risk analysis, we propose Safe-Subspace Pseudo-Label Refinement (\method{}), a selective-risk-guided framework for source-free graph domain adaptation. \method{} first estimates reliable target supervision through multi-expert uncertainty quantification, then learns a target-intrinsic structural manifold via self-supervised graph contrastive learning, and subsequently performs dual-view nested pseudo-label refinement to identify safe target samples. For the remaining uncertain samples, \method{} applies noise-tolerated regularization so that their distributional information can still be exploited without propagating incorrect hard pseudo-labels. In this way, \method{} connects selective bounded-noise analysis with structure-aware graph adaptation, moving beyond purely empirical pseudo-label refinement or topology alignment \citep{mao2024sourcefreegraph,luo2024gala,luo2024rankalign}.
The main contributions of this work are summarized as follows:
\begin{itemize}[leftmargin=*]
\item We provide a selective theoretical analysis of pseudo-label noise in SF-GDA. Under an explicit restricted posterior-discrepancy assumption, the analysis derives a pseudo-label noise bound on a confidence-qualified target subset and decomposes the resulting target risk into safe-subspace fitting error, selected-label noise, and uncertain-set risk.
\item We propose \method{}, a source-free graph domain adaptation framework that integrates multi-expert uncertainty quantification, target-intrinsic structure learning, dual-view nested pseudo-label refinement, and noise-tolerated regularization according to the reliability and risk factors identified by the analysis.
\item We evaluate \method{} on image and graph benchmarks under different domain shifts. The results show that \method{} achieves competitive performance in most transfer settings.
\end{itemize}

\section{Related Work}

\paragraph{Source-Free Domain Adaptation and Graph Domain Adaptation}
Source-free domain adaptation (SFDA) transfers a source-trained model to an unlabeled target domain when source samples are no longer accessible during adaptation~\citep{kundu2020universal,yang2021generalized, wen2025fgplfa, qu2024lead, wang2026usbd}. Existing SFDA methods commonly rely on information maximization, self-training, pseudo-label refinement, feature consistency, or confidence calibration to adapt the source decision boundary to the target distribution~\citep{yi2023sfda_noisy,hwang2024sf,xia2021adaptive,roy2022uncertainty,wang2026usbd}. Source-free graph domain adaptation (SF-GDA) extends this setting to graph-structured data, where domain shifts may affect both node attributes and graph topology~\citep{mao2024sourcefreegraph,yu2025samgpt,luo2024rankalign,zhang2025aggregate,li2024survey, luo2026robust}. This makes adaptation more challenging than in Euclidean data because message passing can amplify local structural mismatch and propagate pseudo-label errors across graph neighborhoods. Recent SF-GDA methods address these issues through topology reconstruction, contrastive learning, graph collaborative training, and structure-aware representation alignment~\citep{luo2023source,liu2024revisiting,zeng2024ugda,chen2025smoothness}. Different from topology-alignment methods that primarily reduce representation mismatch, \method{} focuses on the reliability of the pseudo-supervision used during source-free adaptation. It explicitly separates target samples into a safe subspace for hard pseudo-label supervision and an uncertain set for soft regularization. This selective formulation differs from conventional confidence filtering because target-side structural consistency is used as an additional label-free certificate before a pseudo-label is allowed to affect supervised adaptation.

\paragraph{Learning with Noisy Labels and Pseudo-Label Reliability}
Learning with noisy labels has been widely studied in supervised learning, where common strategies include robust losses, sample selection, label correction, co-training, and uncertainty-aware reweighting~\citep{yang2023parametrical, zhu2024robust,song2022learning,han2020survey,yin2024sport}. In SFDA, pseudo-label noise differs from ordinary annotation noise because it is induced by applying a fixed source hypothesis to a shifted target distribution. Such noise can be systematic rather than random, and recent analyses show that full-domain source-free self-training may suffer from unbounded pseudo-label noise under severe domain shift~\citep{yi2023sfda_noisy,xu2025unraveling}. For graph data, the issue is further complicated by structural heterogeneity: overconfident pseudo-labels may be locally unsupported, while neighborhood agreement may be unreliable if the target representation is geometrically distorted. \method{} therefore combines multi-expert uncertainty estimation, target-intrinsic structural learning, and nested pseudo-label refinement to identify a confidence-consistent safe subspace. This design is complementary to existing noisy-label learning methods because it treats pseudo-label reliability as a selective adaptation problem under source-free graph shifts rather than as a standard label-corruption problem.

\section{Revisiting Theoretical Analysis of Label Noise in SF-GDA}

\subsection{Source-Free Graph Domain Adaptation}

We consider source-free graph domain adaptation (SF-GDA) for graph-level classification
\citep{li2024survey,mao2024sourcefreegraph,luo2024gala,luo2024rankalign}. 
Let a graph be denoted by \(G=(V,E,\mathbf{X})\), where \(V\) and \(E \subseteq V \times V\) are the node and edge sets, respectively, and \(\mathbf{X}\in\mathbb{R}^{|V|\times d}\) is the node-feature matrix. 
During source pretraining, a classifier \(h_S\) is learned from a labeled source domain
$\mathcal{D}_S=\{(G_i^S,y_i^S)\}_{i=1}^{n_S}\sim P_S(G,Y)$.
During adaptation, source graphs are unavailable, and the learner only has access to the pretrained source model \(h_S\) and an unlabeled target domain
$\mathcal{D}_T=\{G_j^T\}_{j=1}^{n_T}\sim P_T(G)$,
where the two domains share the same label space \(\mathcal{Y}=\{1,\ldots,C\}\) but follow different graph distributions due to shifts in both node attributes and topology \citep{liu2024revisiting,yin2023coco}. 
The objective of SF-GDA is to learn a target model \(h_T\) that adapts \(h_S\) to \(\mathcal{D}_T\) without revisiting source data and minimizes the target risk
$R_T(h_T)=\Pr_{(G,Y)\sim P_T}\big[h_T(G)\neq Y\big]$.

\subsection{Setup and Theoretical Question}

We analyze source-free graph domain adaptation in the induced target representation space, as commonly adopted in prior SF-GDA methods \citep{mao2024sourcefreegraph,luo2024gala,luo2024rankalign}. Specifically, each target graph \(G\in\mathcal{D}_T\) is mapped to a representation \(x=\phi(G)\in\mathcal{X}\), and we denote the set of target representations by
$\mathcal{X}_T=\{\phi(G)\mid G\in\mathcal{D}_T\}$.
In the source-free setting, a source model \(h_S\) is trained on the labeled source domain and then deployed on unlabeled target samples only. The learner has no access to source graphs during adaptation and must rely on pseudo-labels induced by the source model,
\begin{equation}
\tilde{y}(x)=\arg\max_{y\in\mathcal{Y}} p_S(y\mid x),\nonumber
\end{equation}
where \(p_S(y\mid x)\) denotes the posterior predicted by the source model. Therefore, the success of source-free adaptation fundamentally depends on whether these source-induced pseudo-labels remain reliable under domain shift \citep{li2024survey,yi2023sfda_noisy,xu2025unraveling}.

To formalize this issue, we define the instance-wise pseudo-label noise rate on the target domain as
\begin{equation}
\eta(x)=\Pr_{Y\sim p_T(\cdot\mid x)}[\tilde{y}(x)\neq Y]
      =1-p_T(\tilde{y}(x)\mid x),\nonumber
\end{equation}
where \(p_T(y\mid x)\) denotes the target posterior. In conventional learning with label noise, effective training typically relies on the bounded-noise premise that the corrupted label remains positively correlated with the ground truth. In source-free adaptation, however, the pseudo-label \(\tilde{y}(x)\) is produced by a fixed source classifier on a shifted target distribution, so the resulting noise is a structured consequence of domain mismatch rather than random corruption \citep{yi2023sfda_noisy,xu2025unraveling}. For graph data, this issue is further entangled with target geometry and local structural relations, since neighborhood consistency may either support or contradict a pseudo-label in the learned representation space.

These observations lead to the central theoretical question of this work: \emph{Can pseudo-label noise in SF-GDA be bounded under domain and structural shifts, and if so, on which subset of target samples can such a guarantee be established?} Equivalently, rather than assuming that there exists a constant \(\bar{\eta}<1/2\) such that
\begin{equation}
\eta(x)\leq \bar{\eta}, \qquad \forall x\in\mathcal{X}_T,\nonumber
\end{equation}
we ask whether one can identify a restricted subset \(\mathcal{H}\subseteq\mathcal{X}_T\) on which the bounded-noise condition is restored, while the remaining samples are handled more cautiously. Answering this question is essential for justifying pseudo-label-based adaptation in SF-GDA. In the following, we first revisit the existing unbounded-noise theory in source-free adaptation, then analyze its scope, and finally establish a bounded-noise result within a confidence-consistent safe subspace.

\subsection{Existing Unbounded-Noise Theory}

To answer the question raised above, we revisit recent theoretical analyses of label noise in source-free adaptation through the lens of learning with label noise \citep{yi2023sfda_noisy,xu2025unraveling}. In conventional learning with noisy labels, a standard premise is that the noisy supervision remains bounded, namely, there exists a constant \(\bar{\eta}<1/2\) such that
\begin{equation}
\eta(x)\leq \bar{\eta}, \qquad \forall x\in\mathcal{X}_T.\nonumber
\end{equation}
This condition ensures that the observed label is still positively correlated with the ground-truth label, which underlies the validity of robust loss correction and noise-tolerant optimization. In source-free adaptation, however, the pseudo-label \(\tilde{y}(x)\) is not generated by random corruption. Instead, it is induced by applying a fixed source classifier to a shifted target distribution, making the resulting noise structurally determined by domain mismatch rather than stochastic perturbation \citep{yi2023sfda_noisy,xu2025unraveling}.

Following \citep{xu2025unraveling}, consider a binary classification setting in which the source domain is modeled as a Gaussian mixture:
\begin{equation}
x \mid y=k \sim \mathcal{N}(\mu_k,\sigma^2 I_d), \qquad k\in\{1,2\},\nonumber
\end{equation}
while the target domain is a globally shifted version of the source:
\begin{equation}
x \mid y=k \sim \mathcal{N}(\mu_k+\Delta,\sigma^2 I_d), \qquad k\in\{1,2\},\nonumber
\end{equation}
where \(\Delta\in\mathbb{R}^d\) denotes the class-invariant translation vector. A key quantity governing the noise behavior is the projection of the domain shift onto the class-separation direction,
\begin{equation}
\alpha=\frac{\Delta^\top(\mu_2-\mu_1)}{\|\mu_2-\mu_1\|^2}.\nonumber
\end{equation}
Under this setup, the pseudo-label noise is shown to be unbounded in a precise conditional sense.

\begin{theorem}[Unbounded Label Noise in SFDA {\citep{xu2025unraveling}}]
\label{thm:unbounded_noise_sfda}
Suppose \(\alpha\neq 0\). Then there exists a non-empty region \(\mathcal{R}\subseteq \mathcal{X}_T\), referred to as a mislabeling area, such that for any \(\delta\in(0,1)\), if
\begin{equation}
\alpha>\frac{1}{d}\log\frac{1-\delta}{\delta},\nonumber
\end{equation}
the conditional error of the source classifier on \(\mathcal{R}\) satisfies
\begin{equation}
\Pr_{(x,y)\sim P_T}\!\big[h_S(x)\neq y \mid x\in\mathcal{R}\big]\geq 1-\delta.\nonumber
\end{equation}
\end{theorem}

Theorem~\ref{thm:unbounded_noise_sfda} shows that, under domain shift, the source classifier can be not merely uncertain but systematically and even confidently wrong on a non-empty target region. In other words, the pseudo-label error rate may exceed the usual learnable regime and can approach one within the mislabeling area. This result provides a rigorous explanation for why naive self-training on the full target domain may fail in source-free adaptation. It also indicates that robust objectives developed under the bounded-noise premise may lose their effectiveness once the pseudo-labels are dominated by systematic target-side bias rather than random label corruption \citep{yi2023sfda_noisy,xu2025unraveling}.

Although Theorem~\ref{thm:unbounded_noise_sfda} reveals a genuine failure mode of source-free adaptation, it is derived under a highly structured global-shift model and analyzes pseudo-label behavior over the entire target domain. These assumptions are stronger than the selective and structure-aware adaptation protocols typically used in SF-GDA. Therefore, before concluding that pseudo-label-based graph adaptation is universally unsafe, it is necessary to examine the scope and limitations of the existing theory.

\subsection{Limitations of Existing Theory}
\label{subsec:limitations_existing_theory}

Although Theorem~\ref{thm:unbounded_noise_sfda} reveals a genuine failure mode of source-free adaptation, its conclusion should be interpreted with caution when transferred to source-free graph domain adaptation. The theorem establishes an impossibility result for a highly structured worst-case regime, whereas practical SF-GDA is conducted under selective adaptation protocols rather than full-domain self-training \citep{xu2025unraveling,mao2024sourcefreegraph,luo2024gala,luo2024rankalign}.

First, the existing analysis assumes a class-invariant global translation model,
\begin{equation}
x \mid y=k \sim \mathcal{N}(\mu_k+\Delta,\sigma^2 I_d), \qquad k\in\{1,2\},\nonumber
\end{equation}
where all classes share the same shift vector \(\Delta\). This particular structure is crucial for constructing the mislabeling area in which the source decision boundary becomes systematically misaligned with the target distribution. In realistic graph adaptation scenarios, domain shift is often heterogeneous across classes and intertwined with variations in graph size, substructure composition, connectivity patterns, and neighborhood statistics \citep{liu2024revisiting,mao2024sourcefreegraph,luo2024gala}.

Second, the theorem requires that the projected shift magnitude
\begin{equation}
\alpha=\frac{\Delta^\top(\mu_2-\mu_1)}{\|\mu_2-\mu_1\|^2}\nonumber
\end{equation}
be sufficiently large to guarantee that the conditional error on the mislabeling area approaches one. This requirement characterizes a severe and highly directional domain discrepancy. When the shift is moderate, or when \(\Delta\) is largely orthogonal to the class-separation direction, the condition for near-certain mislabeling is no longer satisfied \citep{xu2025unraveling}.

Third, and most importantly for SF-GDA, the unbounded-noise theorem analyzes pseudo-label behavior over the entire target domain. In contrast, practical SF-GDA methods rarely update the model using all target samples uniformly. Instead, they typically rely on confidence screening, neighborhood verification, topology-aware refinement, or consistency regularization to restrict training to a filtered subset of target samples that is more likely to provide trustworthy supervision \citep{mao2024sourcefreegraph,luo2024gala,luo2024rankalign}. In graph domains, this distinction is particularly consequential because local topology provides additional relational evidence that can either support or invalidate a pseudo-label. As a result, the effective adaptation distribution is often a restricted subset of \(\mathcal{X}_T\), rather than the full target space analyzed by Theorem~\ref{thm:unbounded_noise_sfda}.

Taken together, these observations suggest that the existing unbounded-noise theory should be viewed as a valid characterization of full-domain source-free self-training under strong global-shift assumptions. It does not rule out the possibility that pseudo-label noise becomes bounded on a carefully filtered target subset whose samples satisfy both high prediction confidence and local structural consistency. Therefore, the more relevant theoretical question for SF-GDA is not whether pseudo-label noise is bounded everywhere on \(\mathcal{X}_T\), but whether one can identify a confidence-consistent safe subspace \(\mathcal{H}_{\tau,\rho}\subseteq\mathcal{X}_T\) on which the bounded-noise condition is restored.

\subsection{Safe Subspace and Bounded Pseudo-Label Noise}
\label{subsec:safe_subspace_bounded_noise}

Motivated by the above discussion, we now analyze a restricted adaptation protocol that operates on target samples supported by both source confidence and local structural agreement. Let
\begin{equation}
s(x)=\max_{c\in\mathcal{Y}}p_S(c\mid x), \; \tilde{y}(x)=\arg\max_{c\in\mathcal{Y}}p_S(c\mid x)\nonumber
\end{equation}
denote the source confidence and the corresponding pseudo-label. Let \(\mathcal{N}(x)\) be the local neighborhood of \(x\) in the target representation space. We define the confidence-qualified consistent neighborhood of \(x\) as
\begin{equation}
\mathcal{N}^{\tau}_{\mathrm{same}}(x)
=
\left\{
x' \in \mathcal{N}(x)
\;\middle|\;
\tilde{y}(x')=\tilde{y}(x),\;
s(x') \ge \tau
\right\}.\nonumber
\end{equation}
The corresponding confidence-consistency ratio is
\begin{equation}
r_{\tau}(x)
=
\frac{|\mathcal{N}^{\tau}_{\mathrm{same}}(x)|}{|\mathcal{N}(x)|},\nonumber
\end{equation}
and the average radius of the confidence-qualified consistent neighborhood is
\begin{equation}
\bar{r}(x)
=
\frac{1}{|\mathcal{N}^{\tau}_{\mathrm{same}}(x)|}
\sum_{x' \in \mathcal{N}^{\tau}_{\mathrm{same}}(x)} d(x,x'),\nonumber
\end{equation}
where \(d(\cdot,\cdot)\) denotes the metric in the target representation space. If \(\mathcal{N}^{\tau}_{\mathrm{same}}(x)=\emptyset\), we set \(\bar r(x)=+\infty\).

\noindent\textbf{Definition 1 (Confidence-Consistent Safe Subspace).}
Given thresholds \(\tau \in (0,1)\) and \(\rho \in (0,1]\), the safe subspace is
\begin{equation}
\mathcal{H}_{\tau,\rho}
=
\left\{
x \in \mathcal{X}_T
\;\middle|\;
s(x)\ge \tau,\;
r_{\tau}(x)\ge \rho
\right\}.
\label{eq:safe_subspace_theory}
\end{equation}
Thus, \(\mathcal{H}_{\tau,\rho}\) contains target samples whose pseudo-labels are supported by both source confidence and local neighborhood agreement. The threshold \(\tau\) controls the population reliability of a selected pseudo-label, while \(\rho\) is an observable target-only certificate that prevents isolated high-confidence samples from dominating adaptation.

To analyze the noise behavior inside \(\mathcal{H}_{\tau,\rho}\), we use the following two assumptions.

\noindent\textbf{Assumption 1 (Restricted Posterior Discrepancy).}
For a confidence threshold \(\tau\), there exists \(\beta_{\tau}(\alpha)\ge 0\) such that every confidence-qualified target point \(z\) satisfies
\begin{equation}
p_T(\tilde{y}(z)\mid z)
\ge
p_S(\tilde{y}(z)\mid z)-\beta_{\tau}(\alpha).\nonumber
\end{equation}
Here, \(\alpha\) denotes the severity of the source-target shift, and \(\beta_{\tau}(\alpha)\) summarizes the residual posterior mismatch on the confidence-qualified region. This assumption is explicitly local to the selected region. Without some target-side posterior regularity of this form, no source-free method can certify pseudo-label correctness from unlabeled target samples alone.

\noindent\textbf{Assumption 2 (Target-Manifold Smoothness).}
The target posterior is \(L\)-Lipschitz continuous on \(\mathcal{X}_T\):
\begin{equation}
|p_T(y\mid x)-p_T(y\mid x')|
\le Ld(x,x'),
\;
\forall x,x'\in\mathcal{X}_T,\; y\in\mathcal{Y}.\nonumber
\end{equation}

The following theorem gives a bounded-noise guarantee on the safe subspace and clarifies the respective roles of confidence, posterior discrepancy, neighborhood consistency, and geometry.

\begin{theorem}[Bounded Pseudo-Label Noise in the Safe Subspace]
\label{thm:bounded_noise_safe_subspace}
Suppose Assumptions 1 and 2 hold. For every \(x\in\mathcal{H}_{\tau,\rho}\), the instance-wise pseudo-label noise satisfies
\begin{equation}
\eta(x)
\le
1-\tau+\beta_{\tau}(\alpha).
\label{eq:center_noise_bound}
\end{equation}
Moreover, the confidence-qualified consistent neighborhood provides the following structural certificates:
\begin{align}
\frac{1}{|\mathcal{N}(x)|}
\sum_{x'\in\mathcal{N}(x)}p_T(\tilde y(x)\mid x')
&\ge
\rho\big(\tau-\beta_{\tau}(\alpha)\big),
\label{eq:aggregate_neighbor_certificate}
\end{align}
\begin{align}
p_T(\tilde y(x)\mid x)
&\ge
\tau-\beta_{\tau}(\alpha)-L\bar r(x).
\label{eq:geometric_neighbor_certificate}
\end{align}
Consequently, if \(\tau-\beta_{\tau}(\alpha)>1/2\), then pseudo-labels on \(\mathcal{H}_{\tau,\rho}\) satisfy the Massart bounded-noise condition.
\end{theorem}

Theorem~\ref{thm:bounded_noise_safe_subspace} should be read as a selective rather than global guarantee. The pointwise part follows from high source confidence and restricted posterior discrepancy; its role is not to claim global correctness of pseudo-labels, but to identify the condition under which hard pseudo-label supervision is justified on a selected target subset. It does not contradict the full-domain unbounded-noise result in Theorem~\ref{thm:unbounded_noise_sfda}. The structural quantities \(\rho\) and \(L\bar r(x)\) should not be interpreted as artificially multiplying the pointwise posterior lower bound. Their role is to provide target-only evidence for identifying reliable samples and to prevent adaptation from relying on isolated high-confidence predictions.

\begin{proposition}[Finite-Sample Reliability of the Structural Gate]
\label{prop:finite_sample_structural_gate}
For a fixed target graph representation \(x\), suppose its \(k_{\mathrm{nn}}\) neighbors are sampled from the local target-neighborhood distribution, and let \(\hat r_{\tau}(x)\) be the empirical confidence-consistency ratio computed from these neighbors. Let \(r_{\tau}(x)=\mathbb{E}[\hat r_{\tau}(x)]\) be the corresponding population ratio. Then, for any \(\epsilon>0\), with probability at least \(1-\delta\),
\begin{equation}
\left|\hat r_{\tau}(x)-r_{\tau}(x)\right|
\le
\sqrt{\frac{\log(2/\delta)}{2k_{\mathrm{nn}}}}.
\label{eq:structural_gate_concentration}
\end{equation}
Consequently, if \(\hat r_{\tau}(x)\ge \rho_{\min}+\epsilon\) and \(k_{\mathrm{nn}}\ge \log(2/\delta)/(2\epsilon^2)\), then \(r_{\tau}(x)\ge \rho_{\min}\) with probability at least \(1-\delta\).
\end{proposition}

\noindent\textit{Proof.}
The empirical ratio \(\hat r_{\tau}(x)\) is an average of \(k_{\mathrm{nn}}\) Bernoulli indicators specifying whether a neighbor has the same pseudo-label and confidence at least \(\tau\). Hoeffding's inequality gives Eq.~\eqref{eq:structural_gate_concentration}. The second statement follows by subtracting the concentration radius from \(\hat r_{\tau}(x)\). 

\subsection{Learnability and Target-Risk Implications}
\label{subsec:learnability_risk}

Theorem~\ref{thm:bounded_noise_safe_subspace} becomes algorithmically meaningful when the pseudo-label noise inside the safe subspace falls into a learnable regime. Let
\begin{equation}
\bar{\eta}_{\tau}
=
\min\{1,\,1-\tau+\beta_{\tau}(\alpha)\}.\nonumber
\end{equation}
By Theorem~\ref{thm:bounded_noise_safe_subspace}, \(\bar{\eta}_{\tau}\) is a uniform upper bound of the instance-wise pseudo-label noise on \(\mathcal{H}_{\tau,\rho}\).

\begin{corollary}[Learnability on the Safe Subspace]
\label{cor:learnability_safe}
If
\begin{equation}
\tau-\beta_{\tau}(\alpha)>\frac{1}{2},
\label{eq:massart_condition_safe}
\end{equation}
then \(\bar{\eta}_{\tau}<1/2\), and the pseudo-labels on \(\mathcal{H}_{\tau,\rho}\) lie in the bounded-noise regime. Consequently, minimizing a classification-calibrated loss on pseudo-labeled samples from \(\mathcal{H}_{\tau,\rho}\) remains statistically learnable under standard noisy-label learning arguments \citep{angluin1988learning,natarajan2013learning}.
\end{corollary}

Corollary~\ref{cor:learnability_safe} shows that reliable source-free adaptation does not require pseudo-label noise to be bounded everywhere on the target domain. It is sufficient that the bounded-noise condition holds on the confidence-consistent subset that actually drives hard pseudo-label supervision. The size of this subset, controlled in practice by \(\rho\), determines the coverage of reliable supervision and creates a quality-coverage trade-off.

The above bound can be improved in practice when pseudo-labels are produced by a diverse committee rather than a single source model. The following result formalizes the idealized case and also states its limitation.

\begin{corollary}[Multi-Expert Voting Bound]
\label{cor:voting_bound}
Suppose that the pseudo-label for each \(x\in\mathcal{H}_{\tau,\rho}\) is produced by majority voting among \(K_e\) conditionally independent experts, each having error at most \(\bar e<1/2\) on \(\mathcal{H}_{\tau,\rho}\). Then the voting error satisfies
\begin{equation}
\eta_{\mathrm{vote}}(x)
\le
\exp\!\left(-2K_e\left(\frac{1}{2}-\bar e\right)^2\right),
\qquad
\forall x\in\mathcal{H}_{\tau,\rho}.
\label{eq:voting_bound}
\end{equation}
\end{corollary}

The conditional-independence assumption in Corollary~\ref{cor:voting_bound} is an idealization. In implementation, source experts trained from the same source domain may be correlated. Therefore, we do not rely on Eq.~\eqref{eq:voting_bound} as an unconditional guarantee. Instead, the algorithm uses ensemble consensus and predictive variance as observable diagnostics: only samples with high agreement and low variance are admitted into the semantic candidate set before structural verification.

Beyond learnability, the safe-subspace view yields a simple target-risk decomposition. Let
\begin{equation}
\pi_{\mathcal{H}}
=
\Pr_{x\sim P_T}\!\left[x\in\mathcal{H}_{\tau,\rho}\right],
\qquad
\pi_{\mathcal{U}}=1-\pi_{\mathcal{H}},\nonumber
\end{equation}
where \(\mathcal{U}=\mathcal{X}_T\setminus\mathcal{H}_{\tau,\rho}\) denotes the uncertain set. Define the pseudo-label risk on the safe subspace as
\begin{equation}
\widetilde{R}_{\mathcal{H}}(h)
=
\Pr_{x\sim P_T}\!\left[h(x)\neq \tilde{y}(x)\mid x\in\mathcal{H}_{\tau,\rho}\right],\nonumber
\end{equation}
and the clean conditional risk on the uncertain set as
\begin{equation}
R_{\mathcal{U}}(h)
=
\Pr_{(x,y)\sim P_T}\!\left[h(x)\neq y \mid x\in\mathcal{U}\right].\nonumber
\end{equation}

\begin{proposition}[Target-Risk Bound Under Safe-Subspace Decomposition]
\label{prop:target_risk_bound}
For any classifier \(h\), the target risk satisfies
\begin{equation}
R_T(h)
\le
\pi_{\mathcal{H}}
\Big(
\widetilde{R}_{\mathcal{H}}(h)+\bar{\eta}_{\tau}
\Big)
+
\pi_{\mathcal{U}}\,R_{\mathcal{U}}(h).
\label{eq:target_risk_bound_basic}
\end{equation}
Furthermore, if the uncertain-set risk is controlled by a regularizer-induced upper bound \(R_{\mathcal{U}}(h)\le \Gamma_{\mathcal{U}}(h)\), then
\begin{equation}
R_T(h)
\le
\pi_{\mathcal{H}}
\Big(
\widetilde{R}_{\mathcal{H}}(h)+\bar{\eta}_{\tau}
\Big)
+
\pi_{\mathcal{U}}\,\Gamma_{\mathcal{U}}(h).
\label{eq:target_risk_bound_reg}
\end{equation}
\end{proposition}

\noindent\textit{Proof.}
Decompose the target risk over \(\mathcal{H}_{\tau,\rho}\) and \(\mathcal{U}\). On \(\mathcal{H}_{\tau,\rho}\), the event \(h(x)\neq y\) is contained in the union of the events \(h(x)\neq\tilde y(x)\) and \(\tilde y(x)\neq y\). Taking conditional probabilities and using the bound \(\Pr[\tilde y(x)\neq y\mid x\in\mathcal{H}_{\tau,\rho}]\le\bar\eta_\tau\) gives Eq.~\eqref{eq:target_risk_bound_basic}. Replacing \(R_{\mathcal U}(h)\) by its regularized upper bound gives Eq.~\eqref{eq:target_risk_bound_reg}. 

Proposition~\ref{prop:target_risk_bound} reveals that the target risk is governed by three factors: the fitting error to reliable pseudo-labels on \(\mathcal{H}_{\tau,\rho}\), the intrinsic pseudo-label noise level \(\bar{\eta}_{\tau}\) of the safe subspace, and the residual risk on the uncertain set \(\mathcal{U}\). This result directly motivates the use of hard pseudo-label supervision only on \(\mathcal{H}_{\tau,\rho}\), while controlling uncertain samples through soft regularization rather than discarding them or assigning hard labels.

\subsection{Roles of the Bounding and Certification Factors}
\label{subsec:minimality_factors}

The preceding analysis identifies four quantities: confidence \(\tau\), restricted posterior discrepancy \(\beta_{\tau}(\alpha)\), neighborhood consistency \(\rho\), and geometric compactness \(L\bar r\). Their roles are different. The population noise bound in Eq.~\eqref{eq:center_noise_bound} is determined by \(\tau\) and \(\beta_{\tau}(\alpha)\): if confidence is low or the source-target posterior discrepancy is large, no bounded-noise guarantee can be obtained. By contrast, \(\rho\) and \(L\bar r\) are target-only certification factors. They do not create correctness by themselves; rather, they improve the reliability of the selected subset by excluding isolated high-confidence predictions and by ensuring that local neighborhood evidence is geometrically meaningful.

This distinction is important for a rigorous interpretation of the theory. We do not claim that all four quantities are mathematically necessary for a pointwise posterior bound once high source confidence and restricted posterior discrepancy already hold at the center point. Instead, we use \(\rho\) and \(L\bar r\) to bridge the gap between population assumptions and a practical label-free algorithm: they determine which target samples are sufficiently supported by the learned target manifold to be used for hard pseudo-label training.

\subsection{Implications for Algorithm Design}
\label{subsec:implications_algorithm_design}

The theoretical analysis prescribes the following design principles. First, the algorithm should increase the reliability of source-induced pseudo-labels by selecting high-confidence, low-variance predictions, corresponding to the terms \(\tau\) and \(\beta_{\tau}(\alpha)\). Second, it should learn a compact and smooth target representation so that neighborhood evidence is meaningful, corresponding to the geometric term \(L\bar r\). Third, it should enforce a local consistency gate \(\rho\) to avoid isolated overconfident errors. Fourth, it should handle the residual uncertain set \(\mathcal{U}\) through soft regularization, as required by Proposition~\ref{prop:target_risk_bound}.

Guided by this correspondence, we develop in the next section a selective-risk-guided SF-GDA framework with four modules. The first module uses a committee of source experts to construct reliable semantic candidates with high consensus and low predictive variance. The second module learns a target-intrinsic structural manifold to reduce geometric distortion and improve local smoothness. The third module identifies the safe subspace \(\mathcal{H}_{\tau,\rho}\) by enforcing consistency between semantic confidence and structural neighborhoods. The fourth module handles the remaining uncertain samples through soft regularization, thereby controlling residual transfer risk without propagating erroneous hard pseudo-labels.
\section{Methodology}
\label{sec:methodology}

\subsection{Overview}

Motivated by the selective-risk analysis in Section~\ref{subsec:implications_algorithm_design}, we propose Safe-Subspace Pseudo-Label Refinement (\method{}) for source-free graph domain adaptation. The framework follows a simple principle: hard pseudo-label supervision is used only on target graphs that pass both semantic and structural reliability checks, while the remaining target graphs are exploited through soft regularization. This design avoids full-domain self-training and directly matches the safe-subspace decomposition in Proposition~\ref{prop:target_risk_bound}. Operationally, \method{} implements this principle through multi-expert uncertainty quantification, target-intrinsic structure learning, dual-view nested pseudo-label refinement, and noise-tolerated regularization.

Let \(\mathcal{D}^t=\{G_j^t\}_{j=1}^{n_t}\) be the unlabeled target set. The source provider releases a committee of \(K_e\) pretrained source hypotheses
\(\mathcal{M}=\{f_S^k\}_{k=1}^{K_e}\), which may be obtained from different random seeds or heterogeneous graph encoders before the source data are discarded. During target adaptation, no source graph or source label is accessed. In the algorithmic implementation, \(\zeta\) is the empirical counterpart of the theoretical confidence threshold \(\tau\), and \(\rho_{\min}\) is the empirical neighborhood-consistency threshold used to construct \(\mathcal{H}_{\zeta,\rho}\). Unless otherwise stated, the adapted target classifier \(f_T\) is initialized from a primary source checkpoint and is the only model used at inference; auxiliary source experts are frozen and used only for reliability estimation. When \(K_e=1\), \method{} reduces to a single-expert variant by disabling the ensemble-variance gate.

\subsection{Module I: Multi-Expert Uncertainty Quantification}
\label{sec:module_I}

Corollary~\ref{cor:voting_bound} motivates the use of a diverse source committee, while also making clear that correlated experts cannot be treated as independent without verification. We therefore use the committee not only to average predictions but also to estimate predictive disagreement. For a target graph \(G_j^t\), the ensemble consensus is
\begin{equation}
\bar{p}_j
=
\frac{1}{K_e}\sum_{k=1}^{K_e}\sigma\big(f_S^k(G_j^t)\big),
\label{eq:consensus}
\end{equation}
where \(\sigma(\cdot)\) denotes the softmax function. The corresponding pseudo-label and confidence are
\begin{equation}
\hat y_j=\arg\max_{c\in\mathcal{Y}}\bar p_{j,c},
\qquad
s_j=\max_{c\in\mathcal{Y}}\bar p_{j,c}.\nonumber
\end{equation}
To reject samples whose apparent confidence is caused by one unstable expert rather than shared semantic evidence, we compute the predictive variance
\begin{equation}
u_j
=
\frac{1}{K_e}\sum_{k=1}^{K_e}
\left\|\sigma\big(f_S^k(G_j^t)\big)-\bar p_j\right\|_2^2.
\label{eq:variance}
\end{equation}
The semantic candidate set is then defined as
\begin{equation}
\mathcal{S}_{\mathrm{sem}}
=
\left\{
G_j^t\in\mathcal{D}^t
\;\middle|\;
 s_j\ge\zeta,
\; u_j\le \nu
\right\},
\label{eq:selection_criteria}
\end{equation}
where \(\zeta\) is the confidence threshold and \(\nu\) is the variance threshold. We set \(\nu\) by a target-side percentile rule, \(\nu=Q_{q_u}(\{u_j\}_{j=1}^{n_t})\), where \(Q_{q_u}\) is the \(q_u\)-th percentile of predictive variance on the unlabeled target set. This rule is label-free and avoids tuning \(\nu\) with target labels. This module operationalizes the confidence requirement in Theorem~\ref{thm:bounded_noise_safe_subspace}; only high-confidence and low-disagreement target graphs are allowed to enter the next verification stage.

\subsection{Module II: Target-Intrinsic Structure Learning}
\label{sec:module_II}

The structural certificates in Theorem~\ref{thm:bounded_noise_safe_subspace} require a target representation in which neighborhoods are meaningful and compact. Directly using the source feature space may be unreliable because the source encoder can preserve source-specific topology. We therefore introduce a target-intrinsic explicit branch \(g_{\psi}\), trained only with unlabeled target graphs, to provide the metric space used for neighborhood verification.

For each mini-batch \(\mathcal{B}\), we generate two stochastic graph views \(\tilde G_i^1,\tilde G_i^2\sim\mathcal{T}(G_i)\) using edge dropping, feature masking, and subgraph sampling. Their normalized embeddings are
\begin{equation}
z_i^v=\frac{g_{\psi}(\tilde G_i^v)}{\|g_{\psi}(\tilde G_i^v)\|_2},
\qquad v\in\{1,2\}.
\end{equation}
We train \(g_{\psi}\) with the graph contrastive objective
\begin{equation}
D_i^v
=
\sum_{m=1}^{|\mathcal{B}|}
\sum_{w=1}^{2}
\mathbb{I}[(m,w)\neq(i,v)] 
\exp\!\left(
\frac{\operatorname{sim}(z_i^v,z_m^w)}{\tau_g}
\right),\nonumber
\end{equation}
\begin{equation}
\mathcal{L}_{\mathrm{GCL}}
=
-\frac{1}{2|\mathcal{B}|}
\sum_{i=1}^{|\mathcal{B}|}
\sum_{v=1}^{2}
\log
\frac{
\exp\!\left(
\operatorname{sim}(z_i^v,z_i^{3-v})/\tau_g
\right)
}{
D_i^v
}.
\label{eq:gcl_loss}
\end{equation}
Here, \(D_i^v\) denotes the contrastive normalization term for anchor \(z_i^v\). By pulling augmented views of the same graph together and separating different graphs, this branch reduces local dispersion and improves the practical reliability of the neighborhood radius \(\bar r\) used in the safe-subspace construction.

\subsection{Module III: Dual-View Nested Pseudo-Label Refinement}
\label{sec:module_III}

Module I provides a semantic view from the source committee, whereas Module II provides a structural view from the target-intrinsic manifold. The nested refinement module intersects these two views to identify the empirical safe subspace. Specifically, we build a \(k_{\mathrm{nn}}\)-nearest-neighbor graph in the embedding space produced by \(g_{\psi}\). For each candidate \(G_j^t\in\mathcal{S}_{\mathrm{sem}}\), let \(\mathcal{N}_j\) be its \(k_{\mathrm{nn}}\) nearest target graphs. We compute the local confidence-consistency score
\begin{equation}
\rho_j
=
\frac{1}{|\mathcal{N}_j|}
\sum_{G_\ell^t\in\mathcal{N}_j}
\mathbb{I}
\left[
\hat y_\ell=\hat y_j,
\; s_\ell\ge\zeta,
\; u_\ell\le\nu
\right].
\label{eq:consistency_score}
\end{equation}
The safe subspace used for hard pseudo-label supervision is
\begin{equation}
\mathcal{H}_{\zeta,\rho}
=
\left\{
G_j^t\in\mathcal{S}_{\mathrm{sem}}
\;\middle|\;
\rho_j\ge\rho_{\min}
\right\},
\label{eq:safe_subspace_final}
\end{equation}
and the uncertain set is \(\mathcal{U}=\mathcal{D}^t\setminus\mathcal{H}_{\zeta,\rho}\). This selection rule implements the safe-subspace definition at the sample level and is supported by the finite-sample concentration result in Proposition~\ref{prop:finite_sample_structural_gate}. It does not assume that neighborhood consistency alone proves correctness; rather, it uses local consistency as a target-only certificate that screens out isolated high-confidence predictions before they influence model updates.

\subsection{Module IV: Noise-Tolerated Regularization}
\label{sec:module_IV}

The uncertain set \(\mathcal{U}\) still contains useful information about the target distribution, but Theorem~\ref{thm:bounded_noise_safe_subspace} does not justify assigning hard pseudo-labels to these samples. We therefore regularize them softly. Let \(f_T\) be the target model initialized from the primary source checkpoint, while auxiliary experts remain frozen for reliability estimation, and let
\begin{equation}
p_j=\sigma\big(f_T(G_j^t)\big)
\end{equation}
be its current target prediction. For uncertain samples, we minimize
\begin{equation}
\mathcal{L}_{\mathrm{reg}}
=
\frac{1}{|\mathcal{U}|}
\sum_{G_j^t\in\mathcal{U}}
\left[
\mathcal{H}(p_j)
+
\lambda_{\mathrm{kl}}
D_{\mathrm{KL}}\big(\operatorname{sg}(\bar p_j)\,\|\,p_j\big)
\right],
\label{eq:reg_loss}
\end{equation}
where \(\mathcal{H}(\cdot)\) denotes entropy, \(D_{\mathrm{KL}}\) is the Kullback-Leibler divergence, and \(\operatorname{sg}(\cdot)\) stops gradients through the source ensemble anchor. The entropy term encourages low-density decision boundaries on the target domain, while the KL term prevents the target model from drifting arbitrarily far from source knowledge. Since no hard label is imposed on \(\mathcal{U}\), this module controls the residual term \(\Gamma_{\mathcal U}(h)\) in Proposition~\ref{prop:target_risk_bound} without propagating unverified pseudo-label errors.

\subsection{Overall Optimization and Training Protocol}
\label{sec:learning_framework}

For a mini-batch \(\mathcal{B}\!\subset\!\mathcal{D}^t\), define \(\mathcal{B}_{\mathcal H}=\mathcal{B}\cap\mathcal{H}_{\zeta,\rho}\) and \(\mathcal{B}_{\mathcal U}=\mathcal{B}\cap\mathcal{U}\). The hard-supervision loss on the safe subspace is
\begin{equation}
\mathcal{L}_{\mathcal H}
=
\frac{1}{|\mathcal{B}_{\mathcal H}|}
\sum_{G_j^t\in\mathcal{B}_{\mathcal H}}
\mathcal{L}_{\mathrm{ce}}(p_j,\hat y_j),
\label{eq:safe_ce_loss}
\end{equation}
with \(\mathcal{L}_{\mathcal H}=0\) when \(\mathcal{B}_{\mathcal H}\) is empty. The total objective is
\begin{equation}
\mathcal{L}_{\mathrm{total}}
=
\mathcal{L}_{\mathcal H}
+
\lambda_{\mathrm{reg}}\mathcal{L}_{\mathrm{reg}}(\mathcal{B}_{\mathcal U})
+
\lambda_{\mathrm{struct}}\mathcal{L}_{\mathrm{GCL}}(\mathcal{B}).
\label{eq:total_loss}
\end{equation}
We optimize Eq.~\eqref{eq:total_loss} in an alternating manner. At the beginning of each epoch, \method{} updates ensemble predictions, target structural embeddings, nearest-neighbor relations, and the safe subspace. It then updates \(f_T\) and \(g_{\psi}\) using mini-batch stochastic optimization. The details of the algorithm are introduced in Algorithm~\ref{alg:s2plr}.

\begin{table*}[t]
\small
\centering
\caption{Graph classification results (in \%) under node and edge density domain shifts on Mutagenicity, and feature domain shifts on DD, PROTEINS, BZR, BZR\_MD, COX2, and COX2\_MD. For convenience, PROTEINS, DD, COX2, COX2\_MD, BZR, and BZR\_MD are abbreviated as P, D, C, CM, B, and BM, respectively. \textbf{Bold} results indicate the best performance.}

\resizebox{\textwidth}{!}{
\setlength{\tabcolsep}{4pt}
\begin{tabular}{
>{\centering\arraybackslash}m{1.8cm}|
*{3}{>{\centering\arraybackslash}m{1.25cm}}|
*{3}{>{\centering\arraybackslash}m{1.25cm}}|
*{6}{>{\centering\arraybackslash}m{1.25cm}}
}
\toprule
\multirow{2}{*}[-0.25em]{\textbf{Methods}}
& \multicolumn{3}{c|}{\textbf{Node Shift}}
& \multicolumn{3}{c|}{\textbf{Edge Shift}}
& \multicolumn{6}{c}{\textbf{Feature Shift}} \\
\cmidrule(lr){2-4} \cmidrule(lr){5-7} \cmidrule(lr){8-13}
& M0$\rightarrow$M1 & M0$\rightarrow$M2 & M0$\rightarrow$M3
& M0$\rightarrow$M1 & M0$\rightarrow$M2 & M0$\rightarrow$M3
& P$\rightarrow$D & D$\rightarrow$P & C$\rightarrow$CM & CM$\rightarrow$C & B$\rightarrow$BM & BM$\rightarrow$B \\
\midrule

WL subtree
& 34.3 & 40.4 & 52.7
& 34.4 & 47.6 & 52.7
& 43.0 & 42.2 & 53.1 & 58.2 & 51.3 & 44.0 \\

GCN
& 64.1\scriptsize{$\pm$1.4} & 65.5\scriptsize{$\pm$2.0} & 56.9\scriptsize{$\pm$2.1}
& 66.3\scriptsize{$\pm$1.7} & 63.6\scriptsize{$\pm$1.4} & 56.0\scriptsize{$\pm$1.4}
& 48.9\scriptsize{$\pm$2.0} & 60.9\scriptsize{$\pm$2.3}
& 51.0\scriptsize{$\pm$1.8} & 66.9\scriptsize{$\pm$1.8}
& 48.7\scriptsize{$\pm$2.0} & 78.8\scriptsize{$\pm$1.7} \\

GIN
& 66.5\scriptsize{$\pm$2.1} & 52.0\scriptsize{$\pm$1.7} & 53.7\scriptsize{$\pm$1.7}
& 67.1\scriptsize{$\pm$1.7} & 54.2\scriptsize{$\pm$2.6} & 55.4\scriptsize{$\pm$1.9}
& 57.3\scriptsize{$\pm$2.2} & 61.9\scriptsize{$\pm$1.9}
& 53.8\scriptsize{$\pm$2.5} & 55.6\scriptsize{$\pm$2.0}
& 49.9\scriptsize{$\pm$2.4} & 79.2\scriptsize{$\pm$2.8} \\

GMT
& 65.7\scriptsize{$\pm$1.8} & 62.1\scriptsize{$\pm$2.1} & 59.0\scriptsize{$\pm$2.0}
& 67.9\scriptsize{$\pm$1.3} & 61.5\scriptsize{$\pm$1.8} & 58.2\scriptsize{$\pm$2.4}
& 59.5\scriptsize{$\pm$2.5} & 50.7\scriptsize{$\pm$2.2}
& 49.3\scriptsize{$\pm$1.8} & 58.2\scriptsize{$\pm$2.0}
& 50.2\scriptsize{$\pm$2.3} & 74.4\scriptsize{$\pm$1.8} \\

CIN
& 65.1\scriptsize{$\pm$1.7} & 66.0\scriptsize{$\pm$1.7} & 55.2\scriptsize{$\pm$1.5}
& 66.3\scriptsize{$\pm$1.8} & 60.8\scriptsize{$\pm$1.7} & 55.8\scriptsize{$\pm$2.4}
& 59.1\scriptsize{$\pm$2.6} & 58.0\scriptsize{$\pm$2.7}
& 51.2\scriptsize{$\pm$2.0} & 55.6\scriptsize{$\pm$1.5}
& 49.2\scriptsize{$\pm$1.4} & 74.2\scriptsize{$\pm$1.9} \\

PathNN
& 70.2\scriptsize{$\pm$1.5} & 67.1\scriptsize{$\pm$2.0} & 58.0\scriptsize{$\pm$1.9}
& 68.9\scriptsize{$\pm$1.9} & 62.9\scriptsize{$\pm$1.7} & 58.1\scriptsize{$\pm$1.6}
& 57.9\scriptsize{$\pm$1.8} & 53.8\scriptsize{$\pm$3.3}
& 49.8\scriptsize{$\pm$1.7} & 66.9\scriptsize{$\pm$2.5}
& 50.3\scriptsize{$\pm$2.3} & 75.3\scriptsize{$\pm$2.2} \\

\midrule

SFDA\_LLN
& 70.7\scriptsize{$\pm$1.6} & 68.1\scriptsize{$\pm$1.2} & 60.4\scriptsize{$\pm$1.3}
& 67.7\scriptsize{$\pm$1.5} & 66.5\scriptsize{$\pm$2.6} & 58.8\scriptsize{$\pm$2.1}
& 59.2\scriptsize{$\pm$1.1} & 59.9\scriptsize{$\pm$1.5}
& 56.9\scriptsize{$\pm$1.9} & 74.2\scriptsize{$\pm$2.0}
& 53.5\scriptsize{$\pm$2.2} & 77.3\scriptsize{$\pm$1.8} \\

SF(DA)$^2$
& 72.0\scriptsize{$\pm$1.5} & 69.7\scriptsize{$\pm$1.9} & 62.0\scriptsize{$\pm$1.8}
& 71.7\scriptsize{$\pm$1.9} & 68.8\scriptsize{$\pm$2.3} & 61.9\scriptsize{$\pm$2.0}
& 62.2\scriptsize{$\pm$1.7} & 63.2\scriptsize{$\pm$1.8}
& 51.7\scriptsize{$\pm$2.3} & 74.6\scriptsize{$\pm$2.6}
& 48.9\scriptsize{$\pm$1.4} & 78.8\scriptsize{$\pm$2.0} \\

NVC\_LLN
& 71.4\scriptsize{$\pm$1.3} & 67.8\scriptsize{$\pm$1.1} & 62.1\scriptsize{$\pm$1.5}
& 70.4\scriptsize{$\pm$2.2} & 67.0\scriptsize{$\pm$2.2} & 62.2\scriptsize{$\pm$1.7}
& 59.2\scriptsize{$\pm$2.1} & 60.0\scriptsize{$\pm$1.8}
& 56.6\scriptsize{$\pm$3.1} & 76.7\scriptsize{$\pm$2.0}
& 53.8\scriptsize{$\pm$1.8} & 78.2\scriptsize{$\pm$2.0} \\

Ucon\_SFDA & 73.2\scriptsize{$\pm$0.8}
& 69.6\scriptsize{$\pm$2.1}
& 62.9\scriptsize{$\pm$3.2} & 74.1\scriptsize{$\pm$1.1}
& 68.2\scriptsize{$\pm$1.2}
& 63.5\scriptsize{$\pm$3.7} & 64.3\scriptsize{$\pm$0.7} & \textbf{66.0\scriptsize{$\pm$2.1}} & 51.2\scriptsize{$\pm$1.9} & 78.3\scriptsize{$\pm$1.7} & 53.1\scriptsize{$\pm$2.5} & 77.9\scriptsize{$\pm$2.0} \\

\midrule

SOGA
& 73.5\scriptsize{$\pm$1.3} & 69.4\scriptsize{$\pm$2.5} & 63.1\scriptsize{$\pm$2.1}
& 72.3\scriptsize{$\pm$2.2} & 69.8\scriptsize{$\pm$1.9} & 63.1\scriptsize{$\pm$1.8}
& 64.4\scriptsize{$\pm$1.3} & 65.4\scriptsize{$\pm$1.8}
& 51.5\scriptsize{$\pm$2.0} & 76.9\scriptsize{$\pm$2.5}
& 53.3\scriptsize{$\pm$3.2} & 78.9\scriptsize{$\pm$2.3} \\

GraphCTA
& 73.7\scriptsize{$\pm$1.8} & 70.7\scriptsize{$\pm$1.5} & 65.2\scriptsize{$\pm$1.7}
& 72.0\scriptsize{$\pm$2.2} & \textbf{70.5\scriptsize{$\pm$1.6}} & 62.8\scriptsize{$\pm$1.8}
& 60.6\scriptsize{$\pm$1.9} & 61.7\scriptsize{$\pm$1.8}
& 56.8\scriptsize{$\pm$3.5} & 77.3\scriptsize{$\pm$2.0}
& 54.2\scriptsize{$\pm$2.7} & 79.0\scriptsize{$\pm$2.6} \\

GALA
& 75.1\scriptsize{$\pm$2.1} & 70.8\scriptsize{$\pm$1.6} & 64.0\scriptsize{$\pm$2.1}
& 72.2\scriptsize{$\pm$1.2} & 70.3\scriptsize{$\pm$1.7} & 63.4\scriptsize{$\pm$1.8}
& 65.6\scriptsize{$\pm$1.3} & 66.0\scriptsize{$\pm$2.5}
& \textbf{61.8\scriptsize{$\pm$2.7}} & 77.0\scriptsize{$\pm$2.6}
& 57.4\scriptsize{$\pm$2.3} & 78.8\scriptsize{$\pm$1.8} \\

GraphATA
& 75.2\scriptsize{$\pm$2.1}
& 69.9\scriptsize{$\pm$1.3}
& 62.1\scriptsize{$\pm$1.0} & 74.4\scriptsize{$\pm$1.3}
& 69.0\scriptsize{$\pm$0.9}
& 61.9\scriptsize{$\pm$2.2} & 64.8\scriptsize{$\pm$1.5} & 63.7\scriptsize{$\pm$2.1} & 59.2\scriptsize{$\pm$2.0} & 77.9\scriptsize{$\pm$2.1} & 55.1\scriptsize{$\pm$1.3} & 78.0\scriptsize{$\pm$3.6} \\

\midrule

\method{}
& \textbf{77.1\scriptsize{$\pm$0.5}} & \textbf{71.6\scriptsize{$\pm$2.0}} & \textbf{65.8\scriptsize{$\pm$1.1}}
& \textbf{75.8\scriptsize{$\pm$2.1}} & 69.9\scriptsize{$\pm$1.6} & \textbf{66.2\scriptsize{$\pm$1.4}}
& \textbf{66.3\scriptsize{$\pm$1.9}} & 61.8\scriptsize{$\pm$2.2}
& 59.1\scriptsize{$\pm$1.5} & \textbf{78.6\scriptsize{$\pm$0.9}}
& \textbf{58.3\scriptsize{$\pm$1.8}} & \textbf{79.3\scriptsize{$\pm$0.8}} \\

\bottomrule
\end{tabular}}
\label{tab:mutag_feature_combined}

\end{table*}

\begin{table*}[t]
\small
\centering
\caption{Image classification results (in \%) on MNIST and CIFAR-10 under edge density domain shifts. \textbf{Bold} results indicate the best performance.}
\resizebox{\textwidth}{!}{
\begin{tabular}{
>{\centering\arraybackslash}m{1.8cm}|
*{6}{>{\centering\arraybackslash}m{1.25cm}}|
*{6}{>{\centering\arraybackslash}m{1.25cm}}
}
\toprule
\multirow{2}{*}[-0.25em]{\textbf{Methods}}
& \multicolumn{6}{c|}{\textbf{MNIST}}
& \multicolumn{6}{c}{\textbf{CIFAR-10}} \\
\cmidrule(lr){2-7} \cmidrule(lr){8-13}
& S0$\rightarrow$S1 & S1$\rightarrow$S0 & S0$\rightarrow$S2
& S2$\rightarrow$S0 & S1$\rightarrow$S2 & S2$\rightarrow$S1
& C0$\rightarrow$C1 & C1$\rightarrow$C0 & C0$\rightarrow$C2
& C2$\rightarrow$C0 & C1$\rightarrow$C2 & C2$\rightarrow$C1 \\
\midrule

WL subtree
& 29.5 & 15.8 & 9.5 & 8.7 & 18.0 & 13.9
& 10.1 & 10.2 & 9.9 & 10.0 & 9.9 & 10.2 \\

GCN
& 82.4\scriptsize{$\pm$0.3} & 84.7\scriptsize{$\pm$0.3}
& 60.9\scriptsize{$\pm$1.9} & 70.5\scriptsize{$\pm$4.0}
& 81.5\scriptsize{$\pm$1.0} & 83.2\scriptsize{$\pm$0.9}
& 44.9\scriptsize{$\pm$0.4} & 46.5\scriptsize{$\pm$0.2}
& 42.0\scriptsize{$\pm$1.2} & 43.0\scriptsize{$\pm$1.8}
& 44.1\scriptsize{$\pm$0.3} & 45.2\scriptsize{$\pm$0.3} \\

GIN
& 86.2\scriptsize{$\pm$1.5} & 79.8\scriptsize{$\pm$2.3}
& 81.4\scriptsize{$\pm$2.6} & 83.9\scriptsize{$\pm$2.1}
& 83.2\scriptsize{$\pm$1.6} & 84.8\scriptsize{$\pm$2.2}
& 44.4\scriptsize{$\pm$2.2} & 43.1\scriptsize{$\pm$1.1}
& 43.3\scriptsize{$\pm$2.1} & 44.7\scriptsize{$\pm$1.7}
& 42.4\scriptsize{$\pm$1.1} & 42.0\scriptsize{$\pm$2.8} \\

GMT
& 83.3\scriptsize{$\pm$1.3} & 83.5\scriptsize{$\pm$1.4}
& 58.2\scriptsize{$\pm$0.8} & 68.3\scriptsize{$\pm$2.4}
& 80.9\scriptsize{$\pm$1.4} & 84.6\scriptsize{$\pm$1.1}
& 52.2\scriptsize{$\pm$0.4} & 53.6\scriptsize{$\pm$0.1}
& 44.1\scriptsize{$\pm$2.7} & 50.0\scriptsize{$\pm$7.6}
& 49.6\scriptsize{$\pm$0.7} & 50.0\scriptsize{$\pm$0.9} \\

CIN
& 85.8\scriptsize{$\pm$2.0} & 86.4\scriptsize{$\pm$1.8}
& 82.2\scriptsize{$\pm$1.5} & 82.8\scriptsize{$\pm$1.5}
& 86.5\scriptsize{$\pm$1.7} & 86.3\scriptsize{$\pm$2.3}
& 51.0\scriptsize{$\pm$0.5} & 55.4\scriptsize{$\pm$1.9}
& 48.1\scriptsize{$\pm$3.0} & 51.9\scriptsize{$\pm$0.8}
& 47.1\scriptsize{$\pm$0.4} & 47.2\scriptsize{$\pm$0.6} \\

PathNN
& 80.2\scriptsize{$\pm$2.8} & 87.9\scriptsize{$\pm$1.4}
& 39.3\scriptsize{$\pm$2.0} & 53.7\scriptsize{$\pm$1.6}
& 79.9\scriptsize{$\pm$1.7} & 89.6\scriptsize{$\pm$2.5}
& 43.7\scriptsize{$\pm$2.0} & 46.5\scriptsize{$\pm$1.5}
& 47.6\scriptsize{$\pm$2.3} & 52.1\scriptsize{$\pm$1.6}
& 44.5\scriptsize{$\pm$2.1} & 47.7\scriptsize{$\pm$2.0} \\

\midrule

SFDA\_LLN
& 93.3\scriptsize{$\pm$0.3} & 92.6\scriptsize{$\pm$1.5}
& 89.7\scriptsize{$\pm$4.1} & 88.3\scriptsize{$\pm$2.0}
& 92.7\scriptsize{$\pm$1.5} & 92.1\scriptsize{$\pm$1.3}
& 56.5\scriptsize{$\pm$3.4} & 55.4\scriptsize{$\pm$1.2}
& 55.2\scriptsize{$\pm$2.0} & 53.8\scriptsize{$\pm$1.5}
& 58.1\scriptsize{$\pm$0.9} & 56.7\scriptsize{$\pm$1.5} \\

SF(DA)$^2$
& 93.7\scriptsize{$\pm$1.0} & 93.8\scriptsize{$\pm$1.3}
& 90.6\scriptsize{$\pm$0.7} & 90.1\scriptsize{$\pm$1.3}
& 93.1\scriptsize{$\pm$1.2} & 93.3\scriptsize{$\pm$1.1}
& 56.7\scriptsize{$\pm$1.4} & 56.4\scriptsize{$\pm$1.3}
& 55.1\scriptsize{$\pm$0.8} & 53.3\scriptsize{$\pm$1.4}
& 57.7\scriptsize{$\pm$1.4} & 57.4\scriptsize{$\pm$1.8} \\

NVC\_LLN
& 93.4\scriptsize{$\pm$1.3} & 93.2\scriptsize{$\pm$1.0}
& 88.9\scriptsize{$\pm$4.8} & 83.0\scriptsize{$\pm$1.9}
& 93.2\scriptsize{$\pm$1.2} & 92.4\scriptsize{$\pm$1.5}
& 56.9\scriptsize{$\pm$1.5} & 55.9\scriptsize{$\pm$1.2}
& 55.5\scriptsize{$\pm$1.4} & 54.1\scriptsize{$\pm$2.8}
& 57.8\scriptsize{$\pm$1.3} & 57.8\scriptsize{$\pm$1.0} \\

Ucon\_SFDA & 93.8\scriptsize{$\pm$1.7} & 93.8\scriptsize{$\pm$0.8} & 88.5\scriptsize{$\pm$0.9} & 90.9\scriptsize{$\pm$2.2} & 92.9\scriptsize{$\pm$0.7} & 93.8\scriptsize{$\pm$1.4} &
57.3\scriptsize{$\pm$1.7} & 56.1\scriptsize{$\pm$2.2} & 56.6\scriptsize{$\pm$1.6} & 54.4\scriptsize{$\pm$1.5} & 58.3\scriptsize{$\pm$0.8} & 57.5\scriptsize{$\pm$1.3} \\

\midrule

SOGA
& 93.9\scriptsize{$\pm$1.1} & 94.3\scriptsize{$\pm$1.7}
& 91.5\scriptsize{$\pm$1.3} & 90.1\scriptsize{$\pm$2.0}
& 93.2\scriptsize{$\pm$1.1} & 93.8\scriptsize{$\pm$1.0}
& 57.6\scriptsize{$\pm$1.1} & 57.0\scriptsize{$\pm$1.3}
& 56.3\scriptsize{$\pm$1.7} & 56.8\scriptsize{$\pm$1.4}
& 57.9\scriptsize{$\pm$0.9} & 58.1\scriptsize{$\pm$2.2} \\

GraphCTA
& 93.8\scriptsize{$\pm$1.0} & 94.5\scriptsize{$\pm$2.1}
& 91.1\scriptsize{$\pm$0.9} & 88.2\scriptsize{$\pm$1.7}
& 94.1\scriptsize{$\pm$2.0} & 93.5\scriptsize{$\pm$1.4}
& 57.9\scriptsize{$\pm$0.9} & 57.4\scriptsize{$\pm$0.8}
& 57.0\scriptsize{$\pm$1.5} & 56.7\scriptsize{$\pm$2.2}
& 57.7\scriptsize{$\pm$1.1} & 58.4\scriptsize{$\pm$1.3} \\

GALA
& 94.1\scriptsize{$\pm$1.7} & 94.9\scriptsize{$\pm$0.9}
& 91.8\scriptsize{$\pm$1.1} & 89.2\scriptsize{$\pm$0.6}
& 93.9\scriptsize{$\pm$0.8} & 94.0\scriptsize{$\pm$0.7}
& 58.1\scriptsize{$\pm$1.2} & 57.6\scriptsize{$\pm$0.8}
& 57.4\scriptsize{$\pm$0.7} & 56.9\scriptsize{$\pm$2.2}
& 58.3\scriptsize{$\pm$1.3} & 58.8\scriptsize{$\pm$1.5} \\

GraphATA & 93.2\scriptsize{$\pm$2.1} & 94.5\scriptsize{$\pm$1.0} & 91.7\scriptsize{$\pm$1.5} & 89.5\scriptsize{$\pm$1.7} & 93.3\scriptsize{$\pm$2.2} & 94.1\scriptsize{$\pm$0.9} & 57.2\scriptsize{$\pm$1.3} & 56.7\scriptsize{$\pm$2.8} & \textbf{58.8\scriptsize{$\pm$1.9}} & 54.4\scriptsize{$\pm$1.6} & \textbf{59.3\scriptsize{$\pm$1.7}} & 57.6\scriptsize{$\pm$2.5} \\

\midrule

\method{}
& \textbf{94.7\scriptsize{$\pm$1.4}} & \textbf{95.8\scriptsize{$\pm$0.3}}
& \textbf{92.5\scriptsize{$\pm$1.3}} & \textbf{91.6\scriptsize{$\pm$1.0}}
& \textbf{94.3\scriptsize{$\pm$0.6}} & \textbf{94.5\scriptsize{$\pm$0.8}}
& \textbf{58.7\scriptsize{$\pm$1.1}} & \textbf{57.8\scriptsize{$\pm$1.4}}
& 56.4\scriptsize{$\pm$0.7} & \textbf{57.9\scriptsize{$\pm$2.2}}
& 58.0\scriptsize{$\pm$1.2} & \textbf{59.0\scriptsize{$\pm$2.0}} \\

\bottomrule
\end{tabular}}
\vspace{-0.5cm}
\label{tab:mnist_cifar10_edge}

\end{table*}

\subsection{Complexity Analysis}
\label{sec:complexity}

The additional cost of \method{} comes from ensemble inference, target contrastive learning, and nearest-neighbor construction. Ensemble inference for reliability estimation costs \(\mathcal{O}(K_e n_t C)\) after the forward representations are computed, where \(C\) is the number of classes. Unless an ensemble variant is explicitly reported, test-time prediction uses only the adapted target model \(f_T\), so the default inference cost is that of a single graph classifier. Exact nearest-neighbor construction in a \(d\)-dimensional embedding space costs \(\mathcal{O}(n_t^2 d)\), and can be replaced by approximate search for large target sets. The contrastive branch has the same order of graph encoder cost as a standard GNN forward pass on augmented views. Since source graphs are never loaded during adaptation, the memory cost is dominated by target mini-batches, cached target embeddings, and the released source checkpoints.

\section{Experiments}

\subsection{Experimental Settings}

\noindent\textbf{Datasets.} We evaluate \method{} on image-derived graph benchmarks, i.e., MNIST~\citep{lecun2002gradient} and CIFAR10~\citep{krizhevsky2009learning}, and real-world graph benchmarks from TUDataset\footnote{\url{https://chrsmrrs.github.io/datasets/}} and OGB\footnote{\url{https://ogb.stanford.edu/}}. Following~\citep{yin2023coco,luo2024gala,yin2025dream}, we consider two types of domain shifts. 
For structural domain shifts, MNIST, CIFAR10, DD, Mutagenicity, NCI1, FRANKENSTEIN, and ogbg-molhiv are partitioned into quantile-based subdomains according to the number of nodes $|V|$ or edges $|E|$, and each ordered pair of distinct subdomains is treated as a source-free transfer task. 
For feature domain shifts, we evaluate PROTEINS$\leftrightarrow$DD, COX2$\leftrightarrow$COX2\_MD, and BZR$\leftrightarrow$BZR\_MD, where source and target domains share the same label space but differ in node-feature distributions. 
More details of these datasets are provided in Appendix~\ref{sec:dataset}.

\noindent\textbf{Baselines.}
We compare \method{} with three groups of baselines. The first group consists of source-only graph learning methods, including the graph kernel WL~\citep{shervashidze2011weisfeiler} and representative graph neural networks, i.e., GCN~\citep{kipf2022semi}, GIN~\citep{xu2018powerful}, CIN~\citep{bodnar2021weisfeiler}, GMT~\citep{baek2021accurate}, and PathNN~\citep{michel2023path}. These methods are trained on the labeled source domain and directly evaluated on the target domain without source-free adaptation. The second group includes general source-free domain adaptation methods, namely SFDA\_LLN~\citep{yi2023sfda_noisy}, SF(DA)$^2$~\citep{hwang2024sf}, NVC\_LLN~\citep{xu2025unraveling}, and Ucon\_SFDA~\citep{xu2025revisiting}. The third group contains graph-specific source-free domain adaptation methods, including SOGA~\citep{mao2024sourcefreegraph}, GraphCTA~\citep{zhang2024collaborate}, GALA~\citep{luo2024gala}, and GraphATA~\citep{zhang2025aggregate}. More details of baselines are provided in Appendix~\ref{app:baseline_protocol}.

\begin{figure*}[t]
    \centering

    \begin{subfigure}[t]{0.19\textwidth}
        \centering
        \includegraphics[width=\linewidth]{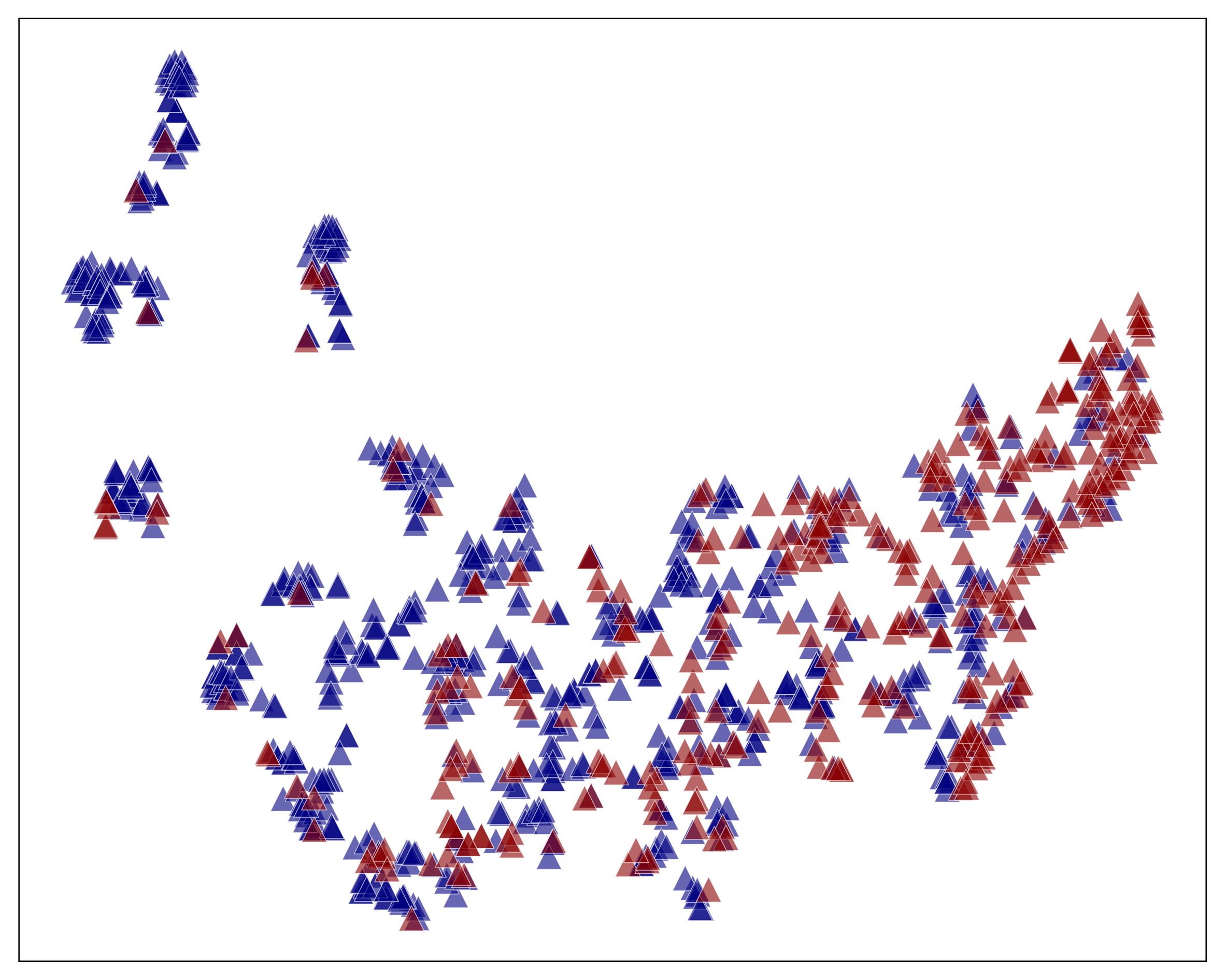}
        \caption{SF(DA)$^2$}
    \end{subfigure}
    \hfill
    \begin{subfigure}[t]{0.19\textwidth}
        \centering
        \includegraphics[width=\linewidth]{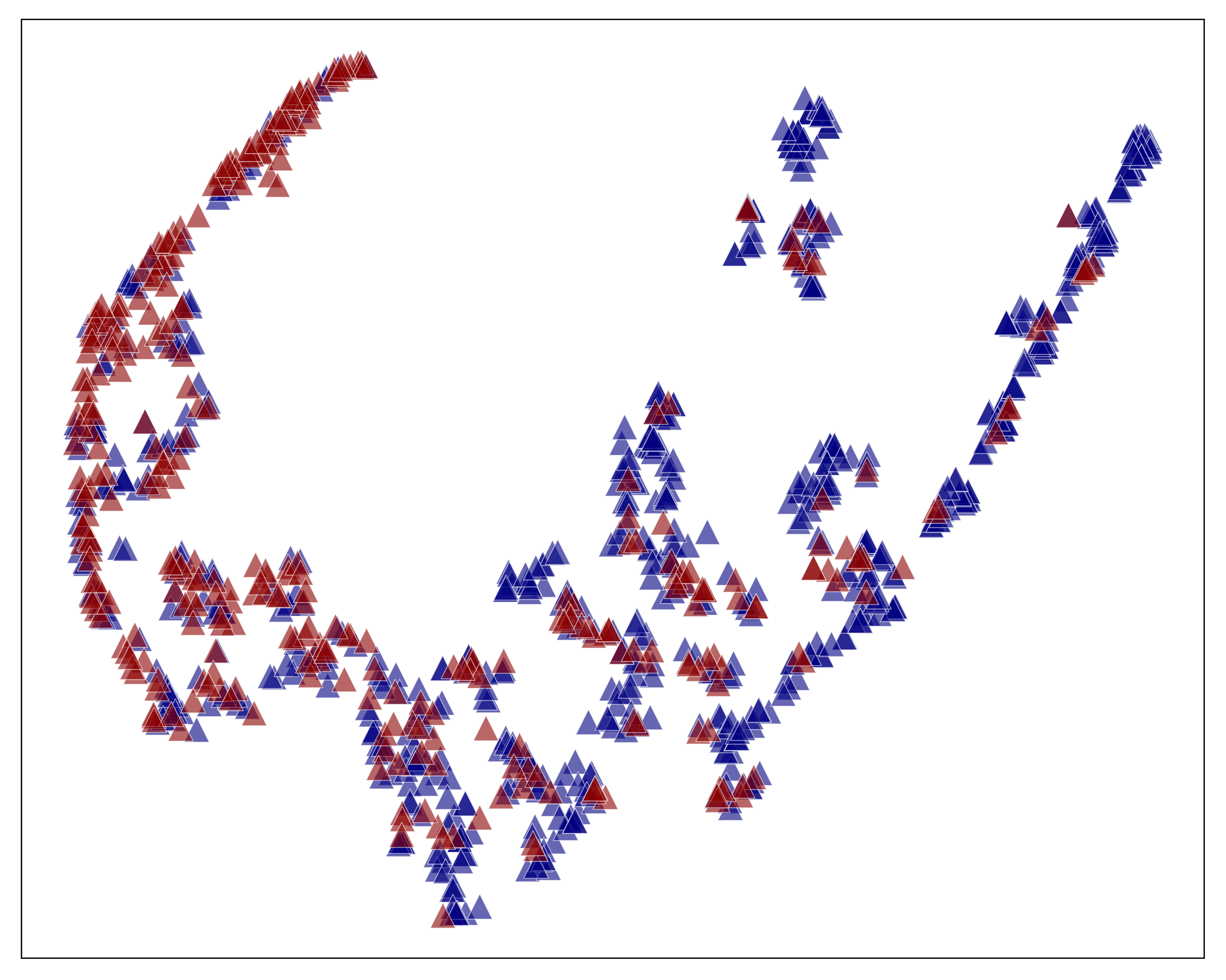}
        \caption{NVC\_LLN}
    \end{subfigure}
    \hfill
    \begin{subfigure}[t]{0.19\textwidth}
        \centering
        \includegraphics[width=\linewidth]{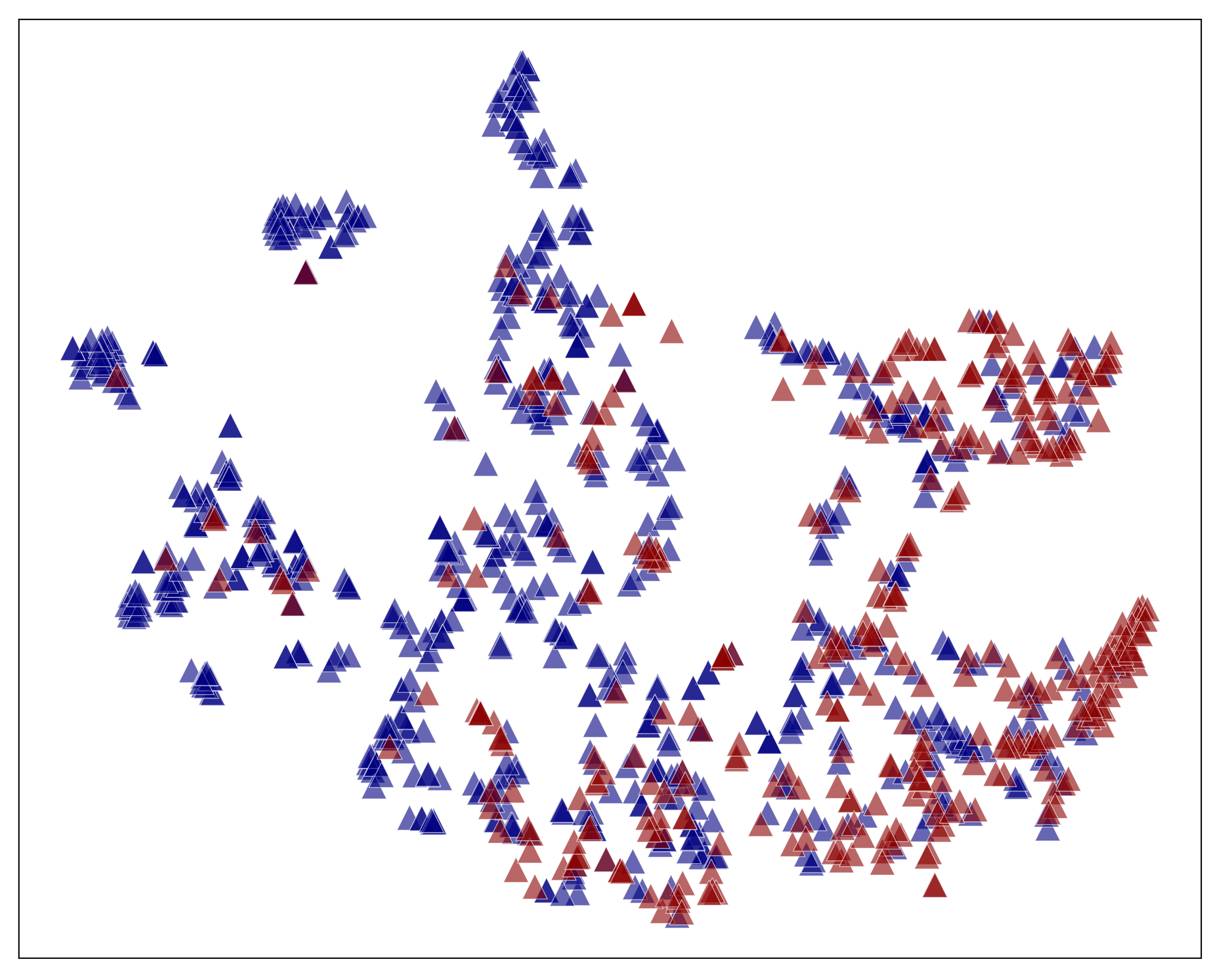}
        \caption{GraphCTA}
    \end{subfigure}
    \hfill
    \begin{subfigure}[t]{0.19\textwidth}
        \centering
        \includegraphics[width=\linewidth]{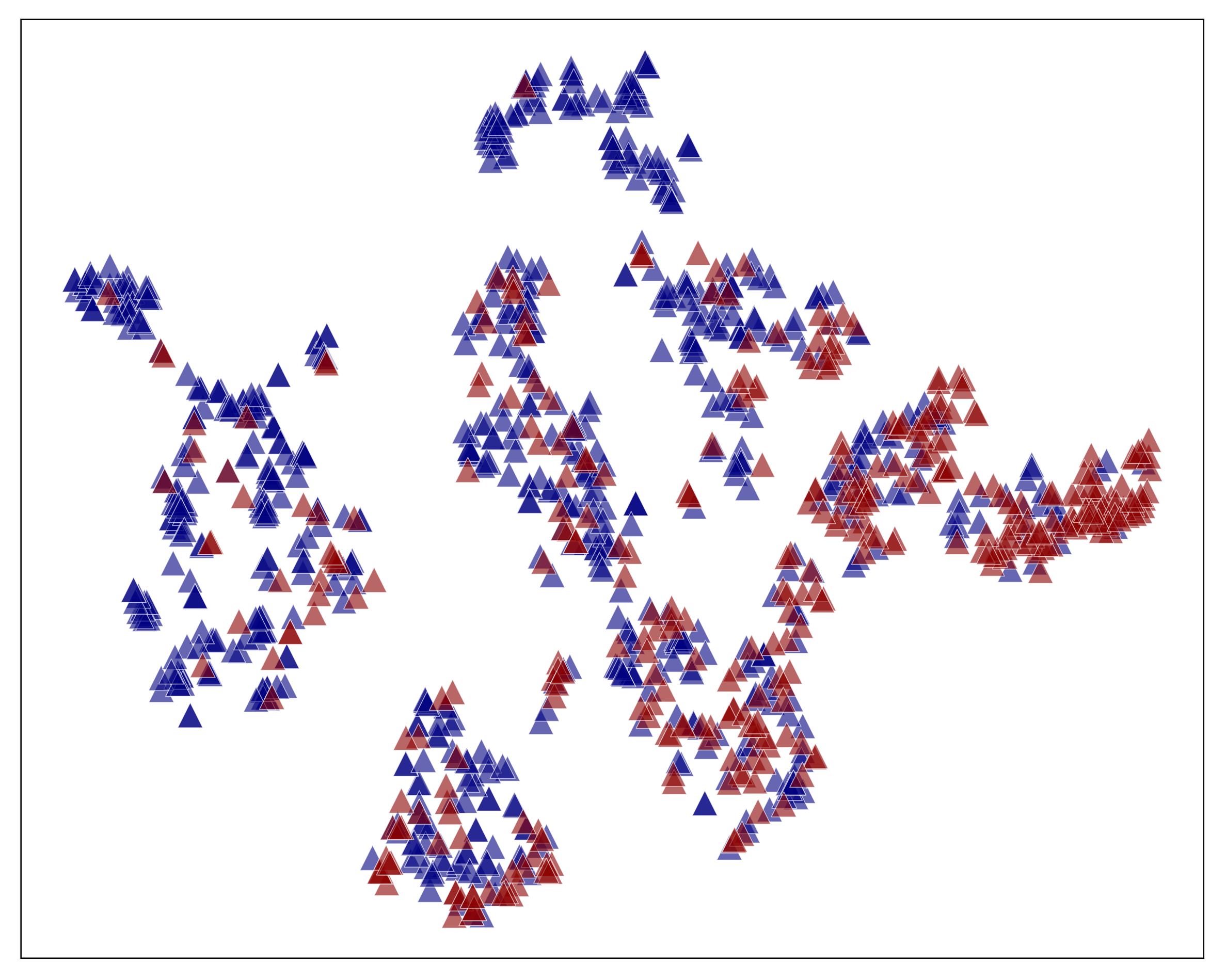}
        \caption{GALA}
    \end{subfigure}
    \hfill
    \begin{subfigure}[t]{0.19\textwidth}
        \centering
        \includegraphics[width=\linewidth]{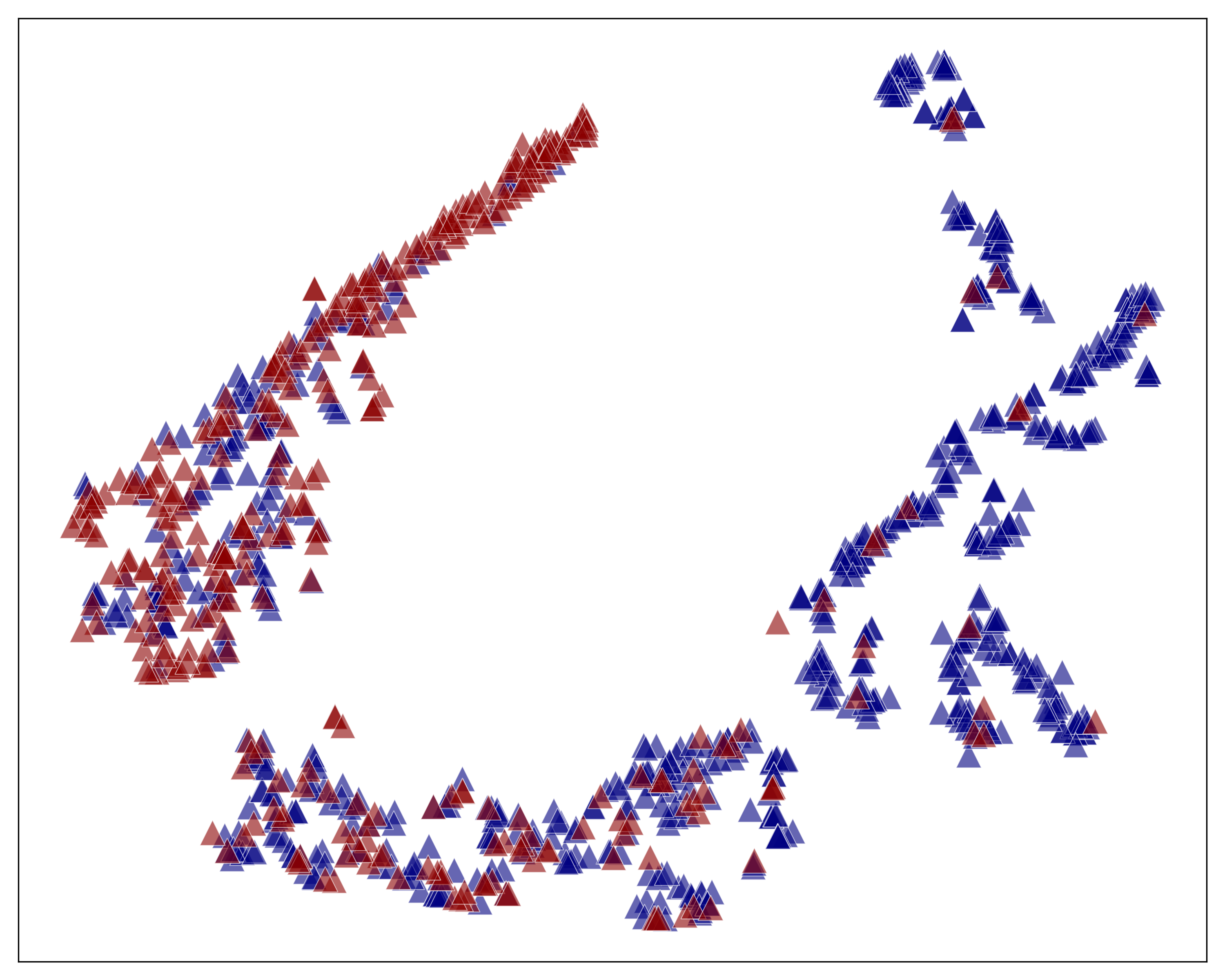}
        \caption{\method{}}
    \end{subfigure}

    \caption{T-SNE visualizations of \method{} and baselines on the Mutagenicity dataset.}
    \label{fig:generalization_tsne_mutag_more}
    \vspace{-0.5cm}
\end{figure*}

\noindent\textbf{Implementation Details.}
We implement \method{} and all baselines in PyTorch\footnote{\url{https://pytorch.org/}} and run all experiments on NVIDIA A100 GPUs. 
Source models are trained only with labeled source graphs, and source graphs and labels are inaccessible during target adaptation. Unless otherwise specified, the target encoder is a 3-layer GNN with hidden dimension 256, optimized by Adam with learning rate \(10^{-4}\) and weight decay \(10^{-12}\). For \method{}, we fix \(K_e=3\) source experts for all datasets. All hyperparameters are determined without using target labels, which are used only for final evaluation. We report classification accuracy on image and TUDataset benchmarks, and ROC-AUC on OGB benchmarks. All results are averaged over five independent runs with different random seeds and reported with standard deviation.

\begin{table}[t]
\centering
\caption{Aggregate comparison across all reported transfer tasks. \textbf{Bold} results indicate the best performance.}
\small
\resizebox{0.5\textwidth}{!}{
\setlength{\tabcolsep}{3pt}
\begin{tabular}{
>{\centering\arraybackslash}m{1.7cm}
ccccc
}
\toprule
Method & Node Avg. & Edge Avg. & Overall Avg. & Avg. Rank & \method{} W/T/L \\
\midrule
WL & 49.29 & 43.23 & 46.10 & 13.60 & 138/0/0 \\
GCN & 51.36 & 53.39 & 52.76 & 12.55 & 134/4/0 \\
GIN & 51.79 & 54.25 & 53.42 & 12.61 & 134/4/0 \\
GMT & 53.62 & 55.12 & 54.55 & 11.97 & 137/1/0 \\
CIN & 51.65 & 54.83 & 53.58 & 12.40 & 134/4/0 \\
PathNN & 56.15 & 57.32 & 56.89 & 11.11 & 135/3/0 \\
\midrule
SFDA\_LLN & 64.67 & 66.42 & 65.53 & 8.01 & 94/44/0 \\
SF(DA)$^2$ & 65.68 & 67.25 & 66.39 & 6.85 & 89/48/1 \\
NVC\_LLN & 65.27 & 66.81 & 66.02 & 7.46 & 90/48/0 \\
Ucon\_SFDA & 66.11 & 67.94 & 67.02 & 5.75 & 80/56/2 \\
\midrule
SOGA & 67.08 & 68.45 & 67.71 & 4.91 & 68/69/1 \\
GraphCTA & 67.60 & 68.69 & 68.06 & 4.26 & 66/71/1 \\
GALA & 67.92 & 69.22 & 68.59 & 3.21 & 63/72/3 \\
GraphATA & 68.66 & 69.88 & 69.20 & 3.53 & 54/80/4 \\
\midrule
{\method{}} & \textbf{71.81} & \textbf{72.77} & \textbf{72.11} & \textbf{1.80} & -- \\
\bottomrule
\end{tabular}
}
\label{tab:overall_summary}
\vspace{-0.5cm}
\end{table}

\subsection{Performance Comparison}
\label{sec:performance_comparison}

\begin{table*}[t]
\centering
\small
\caption{Post-hoc coverage--precision diagnostics for semantic candidates and the final safe subspace across datasets and shift types.}
\label{tab:pseudo_quality}
\resizebox{1.0\textwidth}{!}{
\setlength{\tabcolsep}{3pt}
\hspace*{-0.3cm}
\begin{tabular}{
>{\centering\arraybackslash}m{2.5cm}
>{\centering\arraybackslash}m{2.2cm}
>{\centering\arraybackslash}m{2.2cm}
>{\centering\arraybackslash}m{2.2cm}
>{\centering\arraybackslash}m{1.4cm}
>{\centering\arraybackslash}m{1.4cm}
>{\centering\arraybackslash}m{2.2cm}
>{\centering\arraybackslash}m{2.2cm}
}
\toprule
Dataset & Shift Type & Candidate Cov. & Candidate Prec. & Safe Cov. & Safe Prec. & Precision Gain & Coverage Drop \\
\midrule
Mutagenicity & Node & 35.06\% & 86.32\% & 23.25\% & 92.46\% & \textbf{+6.14\%} & 11.81\% \\
Mutagenicity & Edge & 29.70\% & 80.43\% & 15.96\% & 89.60\% & \textbf{+9.16\%} & 13.75\% \\
NCI1 & Node & 24.34\% & 77.60\% & 13.05\% & 88.06\% & \textbf{+10.46\%} & 11.30\% \\
NCI1 & Edge & 20.25\% & 82.21\% & 11.49\% & 93.22\% & \textbf{+11.01\%} & 8.76\% \\
FRANKENSTEIN & Node & 50.00\% & 68.82\% & 12.73\% & 84.06\% & \textbf{+15.24\%} & 37.27\% \\
FRANKENSTEIN & Edge & 50.00\% & 71.96\% & 12.82\% & 84.89\% & \textbf{+12.94\%} & 37.18\% \\
DD & Node & 50.17\% & 58.11\% & 14.92\% & 70.45\% & \textbf{+12.35\%} & 35.25\% \\
DD & Edge & 45.08\% & 57.89\% & 10.17\% & 73.33\% & \textbf{+15.44\%} & 34.92\% \\
ogbg-molhiv & Node & 49.97\% & 97.88\% & 12.61\% & 99.15\% & \textbf{+1.27\%} & 37.36\% \\
ogbg-molhiv & Edge & 50.00\% & 97.22\% & 8.10\% & 98.92\% & \textbf{+1.70\%} & 41.90\% \\
MNIST & Edge & 49.74\% & 98.04\% & 20.37\% & 99.28\% & \textbf{+1.25\%} & 29.37\% \\
CIFAR10 & Edge & 35.40\% & 63.71\% & 10.56\% & 69.95\% & \textbf{+6.24\%} & 24.84\% \\
\bottomrule
\end{tabular}
}

\end{table*}

We report the performance of \method{} and all baselines under the source-free graph domain adaptation setting in Tables~\ref{tab:mutag_feature_combined}, \ref{tab:mnist_cifar10_edge}, and \ref{tab:dd_node}--\ref{tab:hiv_idx}. 
Overall, the results show that source-only graph models, including WL, GCN, GIN, CIN, GMT, and PathNN, often degrade under structural domain shifts because they cannot adapt to the target distribution after source graphs become unavailable. 
General source-free domain adaptation methods, such as SF(DA)$^2$ and Ucon\_SFDA, improve over source-only models but remain less stable on graph-structured shifts, as they do not explicitly model graph topology or local structural consistency. 
Graph-specific adaptation methods, including GALA and GraphATA, provide stronger baselines and can be competitive on some transfer pairs, especially when the shift is mild or pseudo-labels are relatively reliable. 
In contrast, \method{} achieves the best or empirically comparable performance in most settings. 
Its advantage comes from jointly using multi-expert semantic reliability estimation, target-intrinsic structural learning, and nested safe-subspace refinement, which restricts hard supervision to reliable target samples while regularizing uncertain ones. 
We also provide t-SNE visualizations on Mutagenicity in Fig.~\ref{fig:generalization_tsne_mutag_more}. 
Compared with baselines, \method{} produces more compact and better separated target representations, whereas the baselines, e.g., NVC\_LLN and GALA, show stronger class mixing or scattered local structure. This qualitative evidence is consistent with the quantitative results and suggests that \method{} improves target-domain discriminability under structural shifts. More results on other datasets can be found in Appendix~\ref{sec:more_experiments}.

We further summarize the aggregate results in Table~\ref{tab:overall_summary}. 
Node Avg. and Edge Avg. are computed over the corresponding node-density and edge-density transfer tasks, while Overall Avg. and Avg. Rank are computed over all unique transfer tasks, including feature-shift settings. 
For W/T/L, \method{} is compared with each baseline on every task using the reported mean $\pm$ standard-deviation intervals, where overlapping intervals are treated as ties. The aggregate results show that \method{} achieves the best performance across Node Avg., Edge Avg., Overall Avg., and Avg. Rank, indicating that its improvements are consistent across diverse domain shifts rather than concentrated on a few favorable tasks. 
The newly added baselines, especially GraphATA and Ucon\_SFDA, further strengthen the comparison and improve over several previous source-free adaptation methods. 
Nevertheless, \method{} still maintains the strongest overall robustness. 
The W/T/L statistics also show that \method{} rarely underperforms existing methods, while many comparisons with strong graph adaptation baselines are counted as ties due to overlapping uncertainty intervals. 
These results further confirm the effectiveness of reliability-guided safe-subspace selection under node-density, edge-density, and feature-domain shifts.

\begin{table*}[t]
\centering
\caption{Controlled comparison of source-hypothesis budgets among \method{} and baselines on the Mutagenicity dataset.}
\label{tab:expert_budget}

\setlength{\tabcolsep}{3pt}

\resizebox{\textwidth}{!}{
\begin{tabular}{
>{\centering\arraybackslash}m{2.3cm}
>{\centering\arraybackslash}m{1.0cm}
>{\centering\arraybackslash}m{1.5cm}
>{\centering\arraybackslash}m{1.5cm}
>{\centering\arraybackslash}m{1.2cm}
*{6}{>{\centering\arraybackslash}m{1.2cm}}
>{\centering\arraybackslash}m{1.2cm}
}
\toprule

\multirow{2}{*}[-0.1em]{\centering Method}
& \multirow{2}{*}[-0.1em]{\centering \shortstack{Source\\Hyp.}}
& \multirow{2}{*}[-0.1em]{\centering \shortstack{Committee\\Type}}
& \multirow{2}{*}[-0.1em]{\centering \shortstack{Trainable\\Target Models}}
& \multirow{2}{*}[-0.1em]{\centering \shortstack{Test\\Models}}
& \multicolumn{3}{c}{Node Shift}
& \multicolumn{3}{c}{Edge Shift}
& \multirow{2}{*}[-0.1em]{\centering \shortstack{Avg.\\Rank}} \\

\cmidrule(lr){6-8}
\cmidrule(lr){9-11}

& & & & 
& M0$\rightarrow$M1 
& M0$\rightarrow$M2 
& M0$\rightarrow$M3 
& M0$\rightarrow$M1 
& M0$\rightarrow$M2 
& M0$\rightarrow$M3 
& \\

\midrule

GraphATA-Single
& 1 & None & 1 & 1
& 74.1\scriptsize{$\pm$1.4} & 69.3\scriptsize{$\pm$1.7} & 59.7\scriptsize{$\pm$0.9} 
& 73.8\scriptsize{$\pm$1.3} & 70.0\scriptsize{$\pm$0.7} & 60.6\scriptsize{$\pm$2.5}
& 10.25 \\

GraphATA-Homo3
& 3 & Homogeneous & 3 & 3
& 74.4\scriptsize{$\pm$1.8} & 68.7\scriptsize{$\pm$2.0} & 60.1\scriptsize{$\pm$0.7} 
& 74.1\scriptsize{$\pm$2.3} & 68.9\scriptsize{$\pm$0.5} & 61.6\scriptsize{$\pm$1.3}
& 10.50 \\

GraphATA-Hetero3
& 3 & Heterogeneous & 3 & 3
& 75.2\scriptsize{$\pm$2.1} & 69.9\scriptsize{$\pm$1.3} & 62.1\scriptsize{$\pm$1.0} 
& 74.4\scriptsize{$\pm$1.3} & 69.0\scriptsize{$\pm$0.9} & 61.9\scriptsize{$\pm$2.2}
& 8.50 \\

\midrule

GALA-Single
& 1 & None & 1 & 1
& 75.1\scriptsize{$\pm$2.1}
& 70.8\scriptsize{$\pm$1.6}
& 64.0\scriptsize{$\pm$2.1}
& 72.2\scriptsize{$\pm$1.2}
& 70.3\scriptsize{$\pm$1.7}
& 63.4\scriptsize{$\pm$1.8}
& 6.00 \\

GALA-Homo3
& 3 & Homogeneous & 3 & 3
& 75.7\scriptsize{$\pm$0.1}
& 68.4\scriptsize{$\pm$0.4}
& 63.1\scriptsize{$\pm$0.5}
& 73.7\scriptsize{$\pm$0.2}
& 67.7\scriptsize{$\pm$2.7}
& 60.3\scriptsize{$\pm$1.3}
& 10.75 \\

GALA-Hetero3
& 3 & Heterogeneous & 3 & 3
& 75.7\scriptsize{$\pm$1.0}
& 70.0\scriptsize{$\pm$0.6}
& 63.7\scriptsize{$\pm$2.2}
& 74.7\scriptsize{$\pm$0.8}
& 69.5\scriptsize{$\pm$1.5}
& 61.4\scriptsize{$\pm$1.9}
& 6.67 \\

\midrule

GraphCTA-Single
& 1 & None & 1 & 1
& 73.7\scriptsize{$\pm$1.8}
& 70.7\scriptsize{$\pm$1.5}
& 65.2\scriptsize{$\pm$1.7}
& 72.0\scriptsize{$\pm$2.2}
& \textbf{70.5\scriptsize{$\pm$1.6}}
& 62.8\scriptsize{$\pm$1.8}
& 6.83 \\

GraphCTA-Homo3
& 3 & Homogeneous & 3 & 3
& 74.5\scriptsize{$\pm$2.1}
& 67.8\scriptsize{$\pm$1.3}
& 64.2\scriptsize{$\pm$1.6}
& 73.9\scriptsize{$\pm$1.1}
& 67.0\scriptsize{$\pm$0.8}
& 61.7\scriptsize{$\pm$1.3}
& 10.42 \\

GraphCTA-Hetero3
& 3 & Heterogeneous & 3 & 3
& 75.5\scriptsize{$\pm$1.8}
& 70.1\scriptsize{$\pm$2.2}
& 63.8\scriptsize{$\pm$1.4}
& 74.4\scriptsize{$\pm$1.3}
& 69.0\scriptsize{$\pm$2.1}
& 62.5\scriptsize{$\pm$1.8}
& 6.50 \\

\midrule

SOGA-Single
& 1 & None & 1 & 1
& 73.5\scriptsize{$\pm$1.3}
& 69.4\scriptsize{$\pm$2.5}
& 63.1\scriptsize{$\pm$2.1}
& 72.3\scriptsize{$\pm$2.2}
& 69.8\scriptsize{$\pm$1.9}
& 63.1\scriptsize{$\pm$1.8}
& 9.25 \\

SOGA-Homo3
& 3 & Homogeneous & 3 & 3
& 72.8\scriptsize{$\pm$2.5}
& 67.4\scriptsize{$\pm$1.1}
& 59.7\scriptsize{$\pm$0.6}
& 73.3\scriptsize{$\pm$0.4}
& 68.7\scriptsize{$\pm$0.5}
& 62.2\scriptsize{$\pm$0.5}
& 12.58 \\

SOGA-Hetero3
& 3 & Heterogeneous & 3 & 3
& 75.7\scriptsize{$\pm$1.1}
& 70.3\scriptsize{$\pm$0.5}
& 65.5\scriptsize{$\pm$1.6}
& 74.7\scriptsize{$\pm$1.8}
& 69.5\scriptsize{$\pm$1.2}
& 62.0\scriptsize{$\pm$1.9}
& 4.67 \\

\midrule

\method{}-Single
& 1 & None & 1 & 1
& 75.3\scriptsize{$\pm$1.6}
& 67.9\scriptsize{$\pm$2.3}
& 63.3\scriptsize{$\pm$1.8}
& 73.7\scriptsize{$\pm$2.5}
& 67.7\scriptsize{$\pm$1.3}
& 63.2\scriptsize{$\pm$1.9}
& 9.33 \\

\method{}-Homo3
& 3 & Homogeneous & 1 & 1
& 76.3\scriptsize{$\pm$0.1}
& 67.8\scriptsize{$\pm$3.2}
& 64.7\scriptsize{$\pm$0.7}
& 74.6\scriptsize{$\pm$0.4}
& 68.6\scriptsize{$\pm$1.0}
& 64.8\scriptsize{$\pm$2.3}
& 6.25 \\

\method{}-Hetero3
& 3 & Heterogeneous & 1 & 1
& \textbf{77.1\scriptsize{$\pm$0.5}}
& \textbf{71.6\scriptsize{$\pm$2.0}}
& \textbf{65.8\scriptsize{$\pm$1.1}}
& \textbf{75.8\scriptsize{$\pm$2.1}}
& 69.9\scriptsize{$\pm$1.6}
& \textbf{66.2\scriptsize{$\pm$1.4}}
& \textbf{1.50} \\

\bottomrule
\end{tabular}
}
\vspace{-0.4cm}
\end{table*}

\subsection{Pseudo-Label Quality Diagnostics}
\label{sec:pseudo_quality}

Since the core design of \method{} lies in reliable pseudo-label selection, final classification performance alone cannot fully reveal whether the selected target samples are trustworthy. 
We therefore report post-hoc pseudo-label quality diagnostics in Table~\ref{tab:pseudo_quality} to directly evaluate the selective-learning mechanism. 
Candidate Cov. and Candidate Prec. denote the coverage $|\mathcal{S}_{\mathrm{sem}}|/n_t$ and pseudo-label precision of the semantic candidate set, respectively. 
Safe Cov. and Safe Prec. denote the coverage $|\mathcal{H}_{\zeta,\rho}|/n_t$ and pseudo-label precision of the final safe subspace. 
We further report Precision Gain and Coverage Drop to quantify how much neighborhood-consistency refinement improves pseudo-label reliability at the cost of reducing selected-sample coverage. 
Target labels are used only for this post-hoc diagnostic analysis.

The results reveal a consistent precision--coverage trade-off across the reported datasets and shift types. 
Compared with the initial semantic candidate set, the final safe subspace consistently achieves higher pseudo-label precision, confirming that neighborhood-consistency verification effectively filters out unreliable samples that pass confidence-based semantic selection alone. 
This improvement is especially evident on structurally shifted graph datasets such as DD, NCI1, and FRANKENSTEIN, where confidence-only pseudo-labels are more vulnerable to topology-induced bias; Mutagenicity also shows clear gains under both node- and edge-density shifts. 
For ogbg-molhiv and MNIST, the candidate pseudo-labels are already highly reliable, so the additional gain from safe-subspace refinement is relatively smaller but remains positive. 
CIFAR10 is more challenging, with lower candidate reliability and limited safe coverage, suggesting that conservative pseudo-label selection is harder on complex image-derived graphs. 
Overall, these diagnostics show that \method{} prioritizes pseudo-label reliability over pseudo-label quantity: it selects a smaller but more trustworthy safe subspace for hard supervision, while leaving uncertain samples to be handled by noise-tolerated regularization.

\begin{table*}[t]
\centering
\caption{Per-class safe-subspace diagnostics on the Mutagenicity dataset.}
\label{tab:per_class_coverage}

\resizebox{\textwidth}{!}{
\begin{tabular}{
>{\centering\arraybackslash}m{1.8cm}
>{\centering\arraybackslash}m{1.0cm}
>{\centering\arraybackslash}m{1.5cm}
>{\centering\arraybackslash}m{1.2cm}
*{6}{>{\centering\arraybackslash}m{1.25cm}}
}
\toprule
\multirow{2}{*}[-0.6em]{\centering Dataset}
& \multirow{2}{*}[-0.6em]{\centering Domain}
& \multirow{2}{*}[-0.6em]{\centering Shift Type}
& \multirow{2}{*}[-0.6em]{\centering \shortstack{Safe\\Coverage}}
& \multicolumn{2}{c}{Target Distribution}
& \multicolumn{2}{c}{Safe-Subspace Pseudo-labels}
& \multicolumn{2}{c}{Per-Class Safe Prec.} \\
\cmidrule(lr){5-6} \cmidrule(lr){7-8} \cmidrule(lr){9-10}
& & &
& Class 0 & Class 1
& Class 0 & Class 1
& Class 0 & Class 1 \\
\midrule

\multirow{6}{*}{Mutagenicity} 
& \multirow{3}{*}{Node} 
& M0$\rightarrow$M1 & 23.25\% & 67.34\% & 32.66\% & 76.19\% & 23.81\% & 94.79\% & 85.00\% \\
& & M0$\rightarrow$M2 & 21.86\% & 59.87\% & 40.13\% & 71.31\% & 28.69\% & 89.35\% & 80.88\% \\
& & M0$\rightarrow$M3 & 30.14\% & 46.73\% & 53.27\% & 37.31\% & 62.69\% & 76.23\% & 71.71\% \\
\cmidrule(lr){2-10}
& \multirow{3}{*}{Edge} 
& M0$\rightarrow$M1 & 29.89\% & 66.70\% & 33.30\% & 75.31\% & 24.69\% & 94.67\% & 86.25\% \\
& & M0$\rightarrow$M2 & 15.96\% & 57.01\% & 42.99\% & 56.65\% & 43.35\% & 86.73\% & 93.33\% \\
& & M0$\rightarrow$M3 & 18.34\% & 46.36\% & 53.64\% & 46.73\% & 53.27\% & 73.12\% & 76.42\% \\

\bottomrule
\end{tabular}
}
\end{table*}

\begin{table*}[t]
\small
\centering
\caption{\textcolor{black}{The results of ablation studies on the Mutagenicity dataset (source $\rightarrow$ target).}}
\vspace{-4pt}
\resizebox{1.0\textwidth}{!}{
\begin{tabular}{
>{\centering\arraybackslash}m{2.3cm}|
*{12}{>{\centering\arraybackslash}m{1.2cm}}
}
\toprule
Methods &M0$\rightarrow$M1 &M1$\rightarrow$M0 &M0$\rightarrow$M2 &M2$\rightarrow$M0 &M0$\rightarrow$M3 &M3$\rightarrow$M0 &M1$\rightarrow$M2 &M2$\rightarrow$M1 &M1$\rightarrow$M3 &M3$\rightarrow$M1 &M2$\rightarrow$M3 &M3$\rightarrow$M2 \\
\midrule

\method{} w/o ME & 76.4 & 75.1 & 68.5 & 73.0 & 61.1 & 68.0 & 77.1 & 84.3 & 67.1 & 79.5 & 70.4 & 76.6 \\

\method{} w/o CF & 76.2 & 74.0 & 68.9 & 74.8 & 64.7 & 71.2 & 79.0 & 86.2 & 70.3 & 80.1 & 71.3 & 80.8 \\

\method{} w/o TS & 76.0 & 74.5 & 68.2 & 73.6 & 64.6 & 70.2 & 78.5 & 86.0 & 69.0 & 79.8 & 70.7 & 80.2 \\

\method{} w/o NC & 76.8 & 75.4 & 69.1 & 74.0 & 64.8 & 70.3 & 78.6 & 85.7 & 68.8 & 80.1 & 70.5 & 80.7 \\

\method{} w/o SR & 74.5 & 72.5 & 69.5 & 73.5 & 58.8 & 70.2 & 76.7 & 81.8 & 66.9 & 75.9 & 67.5 & 79.9 \\

\midrule 
\method{} & \textbf{77.1} & \textbf{75.3} & \textbf{71.6} & \textbf{75.4} & \textbf{65.8} & \textbf{72.0} & \textbf{79.9} & \textbf{87.2} & \textbf{70.8} & \textbf{82.4} & \textbf{71.9} & \textbf{81.6} \\
\bottomrule
\end{tabular}
}
\vspace{-0.4cm}
\label{tab:ablation_study_mutag}
\end{table*}

\subsection{Source-Hypothesis Budget and Committee Design}
\label{sec:encoder_combination}

The source committee in \method{} is introduced to provide reliable uncertainty signals rather than simply to increase model capacity. 
To examine the effects of source-hypothesis budget, encoder diversity, and reliability-guided refinement, we report a controlled comparison in Table~\ref{tab:expert_budget}. 
For \method{}, \method{}-Single uses one GMT source hypothesis, \method{}-Homo3 uses three GMT hypotheses trained with different random seeds, and \method{}-Hetero3 uses a heterogeneous committee composed of GMT, GIN, and PathNN. 
All three variants use only one trainable target model and one test model, and the released source hypotheses are used only for semantic reliability estimation. 
For comparison, we apply analogous Single, Homo3, and Hetero3 settings to baselines, including GraphATA, GALA, GraphCTA, and SOGA, where the heterogeneous baseline setting uses GIN, SAGE, and GMT source models.

\begin{figure}[t]
    \centering

    \begin{subfigure}[t]{0.32\linewidth}
        \centering
        \includegraphics[width=\linewidth]{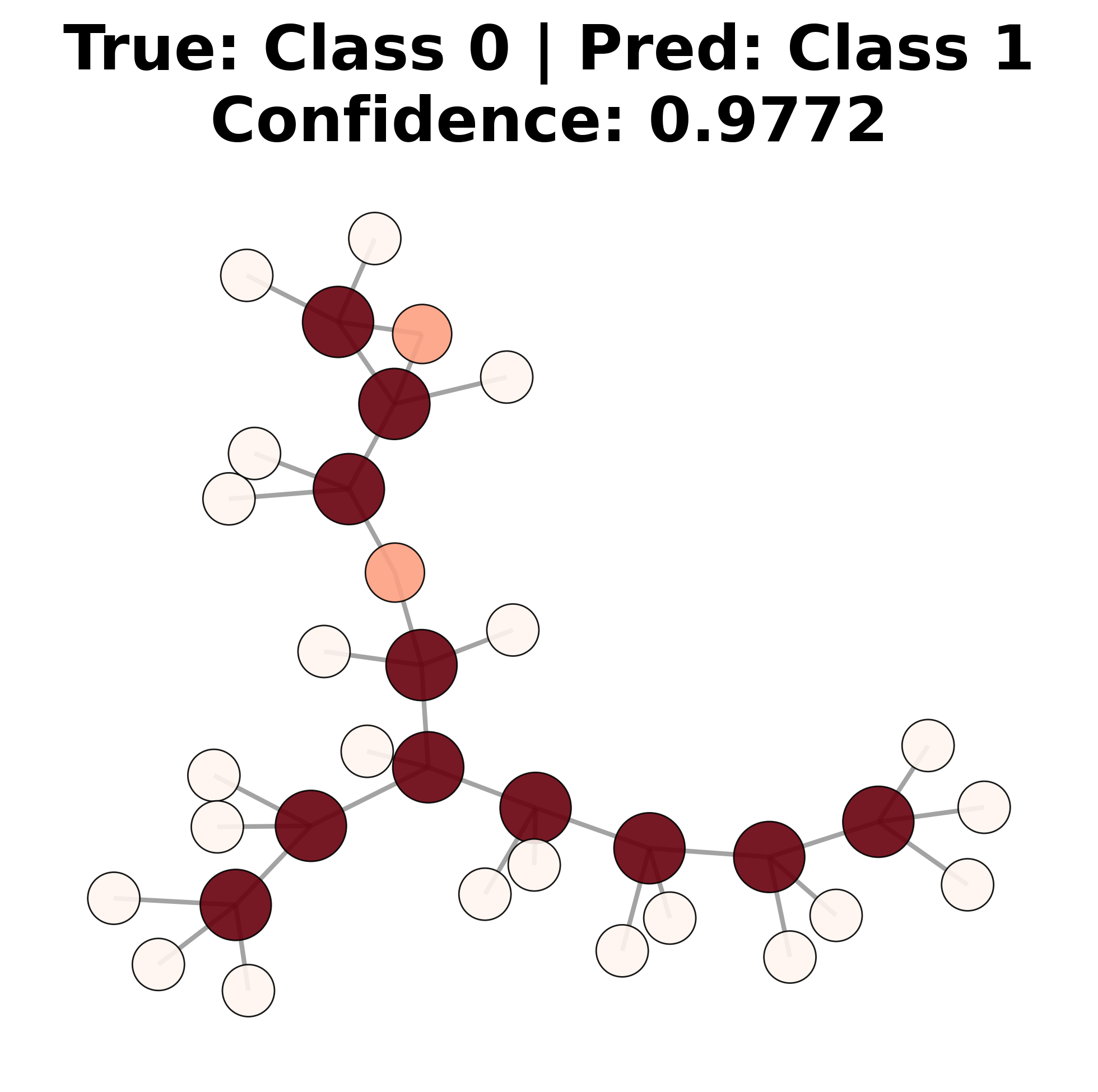}
        \caption{M0$\rightarrow$M1}
    \end{subfigure}
    \hfill
    \begin{subfigure}[t]{0.32\linewidth}
        \centering
        \includegraphics[width=\linewidth]{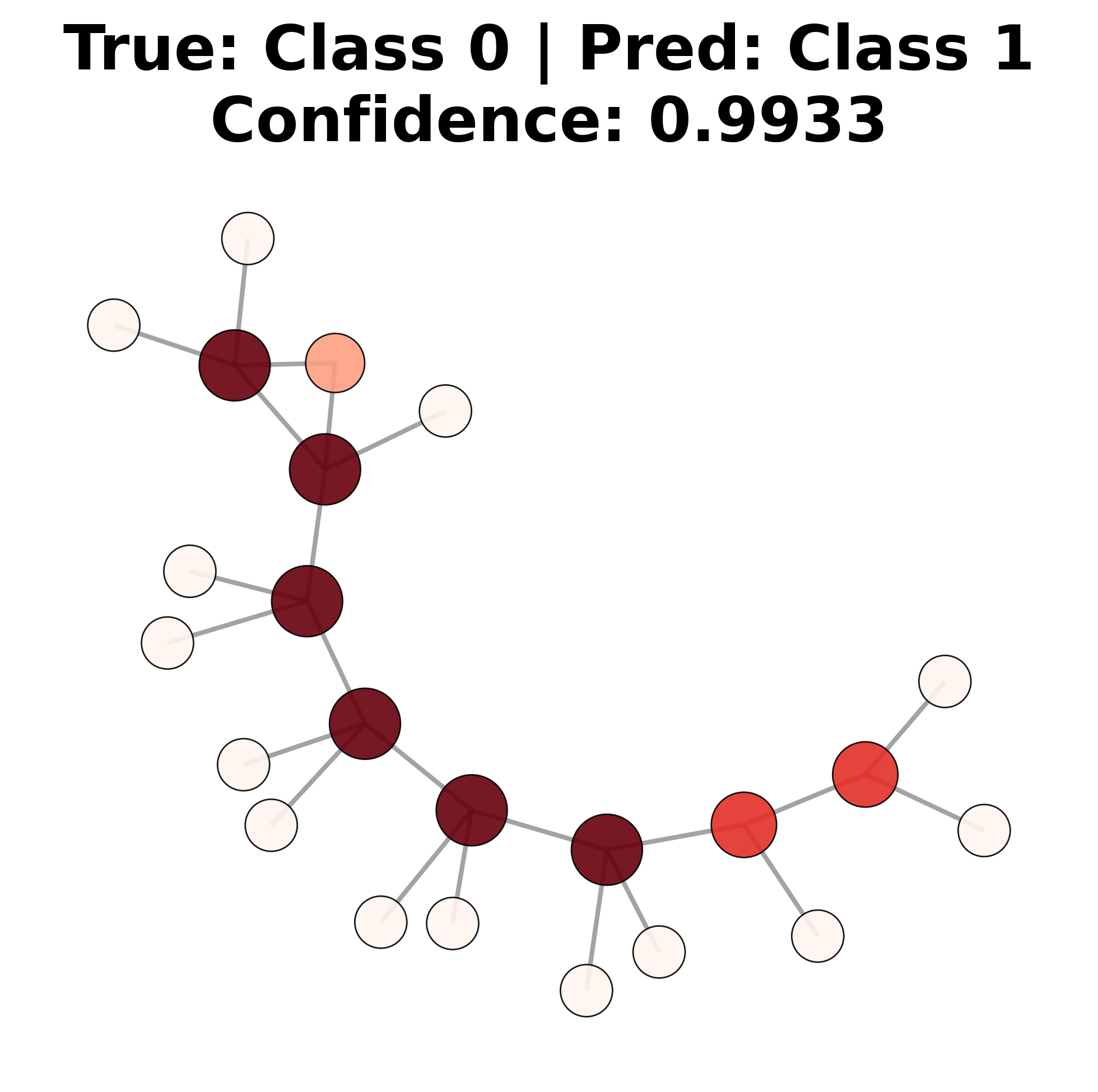}
        \caption{M0$\rightarrow$M2}
    \end{subfigure}
    \hfill
    \begin{subfigure}[t]{0.32\linewidth}
        \centering
        \includegraphics[width=\linewidth]{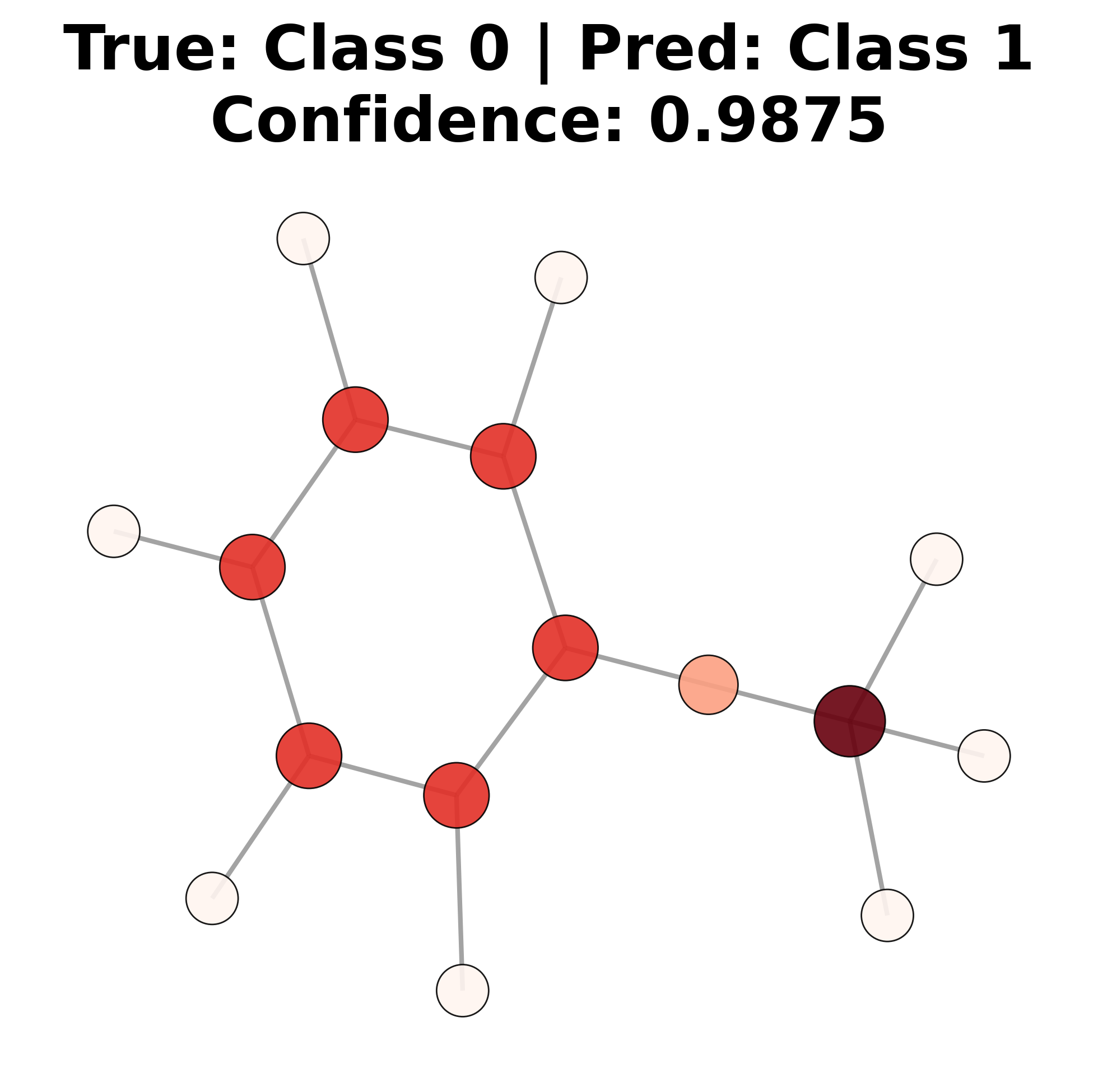}
        \caption{M0$\rightarrow$M3}
    \end{subfigure}

    \caption{High-confidence false pseudo-labels on Mutagenicity.}
    \vspace{-0.5cm}
    \label{fig:failure_cases}
\end{figure}

The results show that increasing the number of source hypotheses alone does not guarantee better adaptation. 
For \method{}, the homogeneous three-GMT committee improves over the single-GMT variant in most transfer tasks, while the heterogeneous GMT/GIN/PathNN committee achieves the best average rank and the strongest overall performance. 
This indicates that complementary graph encoders provide more informative disagreement signals for identifying reliable pseudo-label candidates. 
In contrast, the multi-hypothesis variants of GraphATA, GALA, GraphCTA, and SOGA show less consistent improvements, and some Homo3 or Hetero3 variants even underperform their single-model counterparts despite using more trainable target models and test-time models. 
These observations suggest that the advantage of \method{} is not merely due to a larger source-hypothesis budget or increased test-time capacity. 
Rather, the key benefit comes from integrating multi-expert disagreement with target-structural verification and safe-subspace supervision.

\subsection{Failure Analysis of Safe-Subspace Selection}
\label{sec:failure_analysis}

We further conduct a failure analysis on the Mutagenicity dataset to examine the limitations of safe-subspace selection under severe structural shifts. 
Table~\ref{tab:per_class_coverage} reports the target class distribution, the pseudo-label distribution within the selected safe subspace, and the class-wise pseudo-label precision after safe selection. 
Here, Per-Class Safe Prec. measures the precision of selected safe samples conditioned on each predicted pseudo-label class, and target labels are used only for post-hoc analysis. 
The results show that the safe subspace is generally reliable under mild or moderate shifts, but its class composition and class-wise reliability may become biased when the structural shift becomes severe. 
This issue can be further amplified by class-prior variation induced by density-based partitioning, since the selection process may favor the class whose structural patterns are more confidently recognized by the source hypothesis. 
For example, under the node-density shift M0$\rightarrow$M3, the safe pseudo-label distribution becomes more skewed toward Class 1 than the true target distribution, while the class-wise precision also drops. 
This suggests that some Class-0 target graphs can be absorbed into the Class-1 prediction region, even after confidence-based filtering and neighborhood-consistency verification. 
For edge-density shifts, the selected class distribution can remain relatively balanced, but class-wise precision still decreases under harder transfers, indicating that class balance alone is insufficient to guarantee pseudo-label correctness.

Additionally, Fig.~\ref{fig:failure_cases} visualizes representative high-confidence failure cases where Class-0 target graphs are predicted as Class 1. 
These graphs contain local dense or branched structural patterns that may resemble source-domain Class-1 cues, suggesting that the source hypothesis can rely on shortcut structural signals under topology shift. 
Such cases explain why aggressive pseudo-label expansion is risky: even highly confident predictions may be systematically wrong in certain target regions. 
This analysis supports the design of \method{}, which restricts hard supervision to a verified safe subspace and handles the remaining uncertain samples through noise-tolerated regularization.

\subsection{Ablation Study}
\label{sec:ablation} 

We conduct ablation studies to evaluate the contribution of each component in \method{}, with variants corresponding to the four modules in Section~\ref{sec:methodology}. 
\method{} w/o ME removes multi-expert uncertainty quantification and uses a single source hypothesis, disabling the predictive-variance-based disagreement criterion. 
\method{} w/o CF removes confidence-based semantic filtering and admits pseudo-labels without the threshold $\zeta$. 
\method{} w/o TS removes the target-intrinsic structure learning branch and graph contrastive objective, so neighborhoods are not constructed from a target-specific structural manifold. 
\method{} w/o NC removes neighborhood-consistency filtering and disables the structural gate $\rho_{\min}$. 
\method{} w/o SR removes soft regularization on the uncertain set $\mathcal{U}$, leaving only hard pseudo-label supervision on the safe subspace.

Experimental results are shown in Table~\ref{tab:ablation_study_mutag}. 
The results show that each component contributes to robust adaptation. 
First, removing ME or CF weakens semantic reliability estimation: the former loses ensemble-disagreement filtering, while the latter admits low-confidence samples, both increasing pseudo-label noise. 
Second, removing TS degrades performance because the model can no longer learn target-specific structural neighborhoods, making refinement more dependent on source-biased representations. 
Third, removing NC allows semantic candidates to be used for hard supervision without local structural verification, confirming that confidence alone is insufficient for reliable graph pseudo-labeling. 
Finally, removing SR underutilizes uncertain target samples, indicating that noise-tolerated regularization helps exploit target distribution information without assigning unreliable hard labels. More results on other datasets are provided in Appendix~\ref{sec:more_ablation}.

\begin{figure}[t]
    \centering

    \begin{subfigure}[t]{0.48\linewidth}
        \centering
        \includegraphics[width=\linewidth]{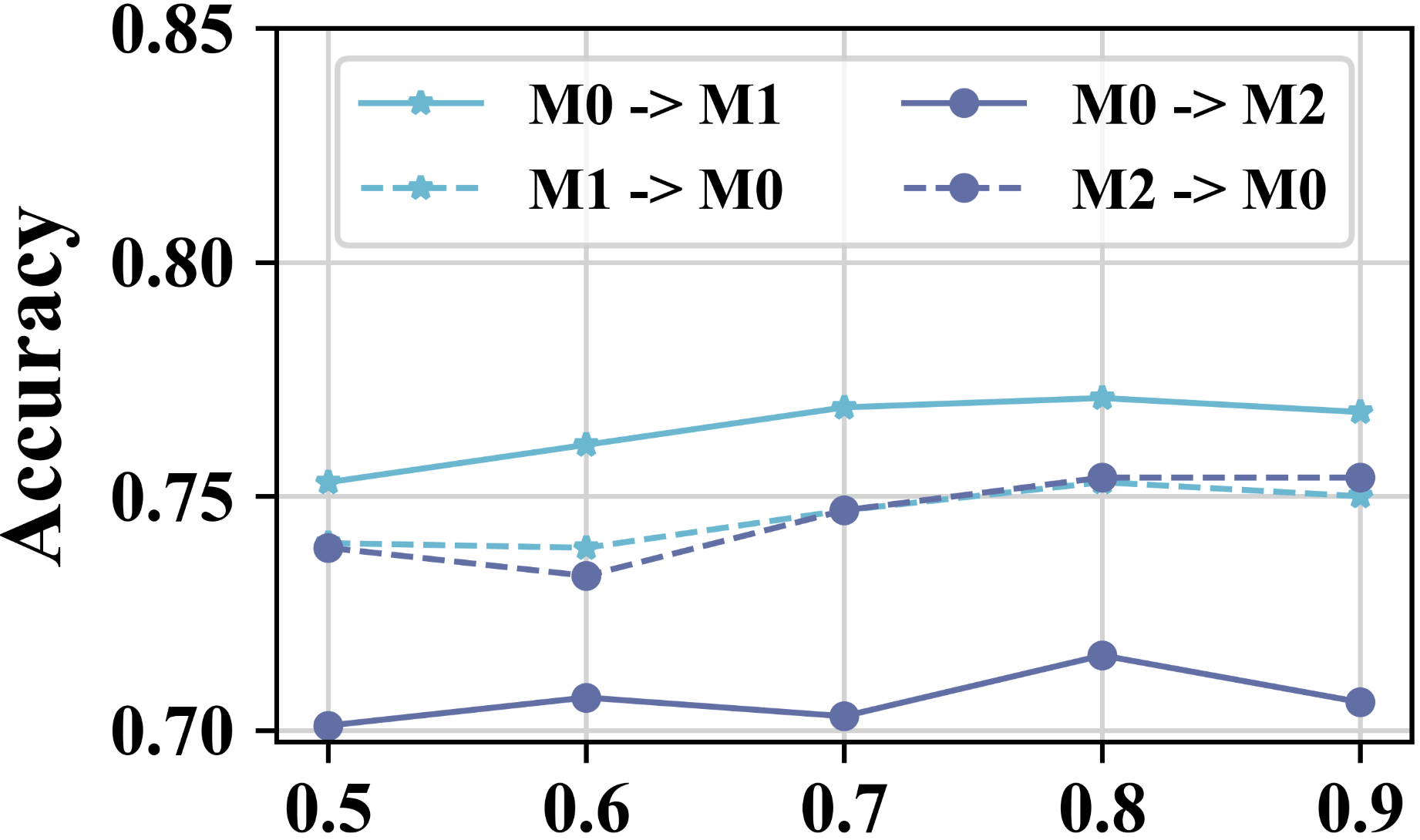}
        \caption{$\zeta$}
    \end{subfigure}
    \hfill
    \begin{subfigure}[t]{0.48\linewidth}
        \centering
        \includegraphics[width=\linewidth]{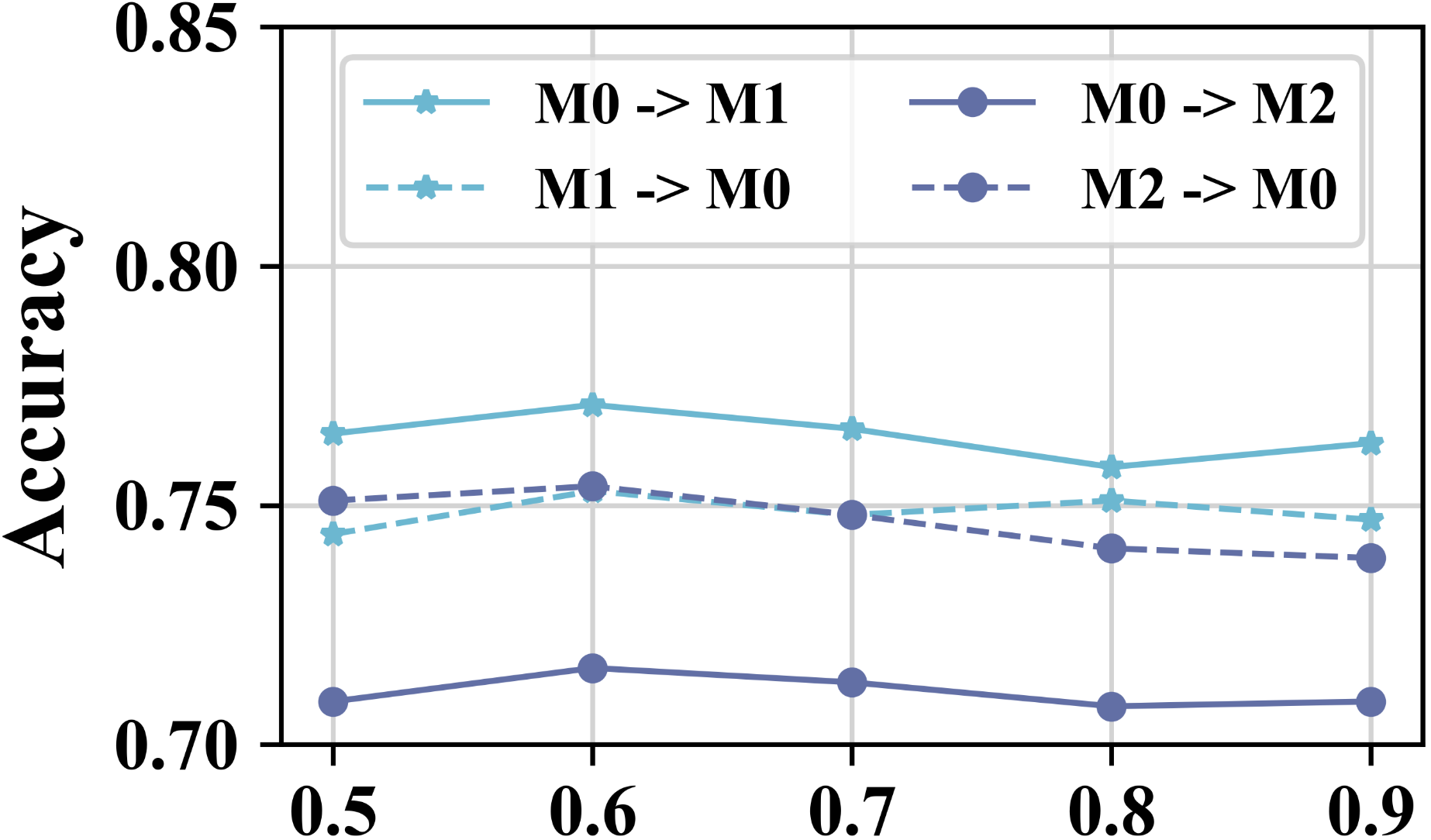}
        \caption{$\rho_{\min}$}
    \end{subfigure}

    \caption{Sensitivity analysis of the confidence threshold $\zeta$ and neighborhood-consistency threshold $\rho_{\min}$ in \method{} on the Mutagenicity dataset.}

    \vspace{-0.5cm}
    \label{fig:sensitivity}
\end{figure}

\subsection{Sensitivity Analysis}
\label{sec:sensitivity}

We conduct a sensitivity analysis on two key thresholds in \method{}, namely the confidence threshold $\zeta$ and the neighborhood-consistency threshold $\rho_{\min}$, as illustrated in Fig.~\ref{fig:sensitivity}. 
The confidence threshold $\zeta$ controls the semantic confidence required for selecting pseudo-label candidates, while $\rho_{\min}$ determines the minimum local neighborhood agreement needed for a sample to be admitted into the safe subspace. 
To evaluate the robustness of \method{} with respect to these thresholds, we vary one parameter at a time while fixing the other hyperparameters to their default settings. 
Specifically, both $\zeta$ and $\rho_{\min}$ are varied over $\{0.5, 0.6, 0.7, 0.8, 0.9\}$.

As shown in Fig.~\ref{fig:sensitivity}, both thresholds exhibit a clear quality--coverage trade-off. 
For $\zeta$, increasing the threshold generally improves pseudo-label reliability by filtering out low-confidence predictions, but an overly strict threshold may reduce the safe-subspace coverage and limit the amount of hard supervision. 
On the Mutagenicity dataset, $\zeta=0.8$ provides a favorable balance between pseudo-label quality and selected-sample coverage. 
For $\rho_{\min}$, smaller values may admit locally inconsistent samples into the safe subspace, whereas larger values may discard useful target samples whose neighborhoods are only partially consistent under domain shift. 
The setting $\rho_{\min}=0.6$ achieves a stable compromise between structural reliability and effective target coverage. 
Overall, the results indicate that \method{} is reasonably robust to threshold choices when the semantic and structural filters are set within a moderate range. 
Additional sensitivity results on other datasets are provided in Appendix~\ref{sec:more_sensitivity}.

\section{Conclusion}

This paper revisited source-free graph domain adaptation from the perspective of selective adaptation under systematic pseudo-label noise. Instead of assuming that source-induced pseudo-labels are reliable over the entire target domain, we characterized a confidence-consistent safe subspace where hard pseudo-label supervision can be more cautiously justified. Based on this view, we proposed \method{}, which integrates source-committee uncertainty estimation, target-intrinsic structural representation learning, semantic-structural pseudo-label verification, and noise-tolerant regularization for uncertain target samples. Extensive experiments on image and graph benchmarks under node-density, edge-density, and feature shifts demonstrate that \method{} improves pseudo-label reliability and achieves robust performance across diverse source-free transfer settings. Further diagnostic, ablation, sensitivity, and source-hypothesis budget analyses show that the gains mainly come from reliability-guided safe-subspace supervision rather than simply from additional source hypotheses. A remaining limitation is that safe-subspace selection may become overly conservative or class-biased under severe structural shifts, suggesting future work on adaptive coverage control and class-aware reliability estimation for broader SF-GDA scenarios.


\bibliographystyle{IEEEtranN}
\bibliography{reference}

\clearpage
\clearpage
\onecolumn
\appendix

\section{Detailed Proofs and Supplementary Material}
\label{app:details}

This appendix provides detailed proofs, implementation details, and additional
experimental protocols for \method{}. The notation follows
Sections~\ref{subsec:safe_subspace_bounded_noise}--\ref{sec:methodology}. The
purpose of the appendix is to make explicit the distinction between two notions
that are easy to conflate: (i) a population noise bound on selected pseudo labels
and (ii) target-only structural certificates used to identify reliable samples in
practice.

\subsection{Detailed Proofs}
\label{app:proofs}

\subsubsection{Proof of Theorem~\ref{thm:bounded_noise_safe_subspace}}

Let \(x\in\mathcal{H}_{\tau,\rho}\) and denote its pseudo-label by
\(c=\tilde y(x)\). By the definition of the safe subspace in
Eq.~\eqref{eq:safe_subspace_theory}, we have
\begin{equation}
    p_S(c\mid x)=s(x)\ge \tau .
\end{equation}
Assumption~1 then gives
\begin{equation}
    p_T(c\mid x)
    \ge p_S(c\mid x)-\beta_{\tau}(\alpha)
    \ge \tau-\beta_{\tau}(\alpha).
\end{equation}
Since \(\eta(x)=\Pr[\tilde y(x)\neq y\mid x]=1-p_T(c\mid x)\), we obtain
\begin{equation}
    \eta(x)\le 1-\tau+\beta_{\tau}(\alpha),
\end{equation}
which proves Eq.~\eqref{eq:center_noise_bound}.

We next prove the aggregate neighborhood certificate. For any
\(x'\in\mathcal{N}_{\mathrm{same}}^{\tau}(x)\), the construction of
\(\mathcal{N}_{\mathrm{same}}^{\tau}(x)\) implies
\(\tilde y(x')=\tilde y(x)=c\) and \(s(x')\ge\tau\). Applying Assumption~1 at
\(x'\) yields
\begin{equation}
    p_T(c\mid x')
    \ge p_S(c\mid x')-\beta_{\tau}(\alpha)
    \ge \tau-\beta_{\tau}(\alpha).
\end{equation}
Moreover, because \(x\in\mathcal{H}_{\tau,\rho}\), at least a \(\rho\)-fraction
of the target neighborhood \(\mathcal{N}(x)\) belongs to
\(\mathcal{N}_{\mathrm{same}}^{\tau}(x)\). The posterior probabilities of the
remaining neighbors are nonnegative. Therefore,
\begin{align}
    \frac{1}{|\mathcal{N}(x)|}
    \sum_{x'\in\mathcal{N}(x)}p_T(c\mid x')
    &\ge
    \frac{|\mathcal{N}_{\mathrm{same}}^{\tau}(x)|}{|\mathcal{N}(x)|}
    \big(\tau-\beta_{\tau}(\alpha)\big) \\
    &\ge
    \rho\big(\tau-\beta_{\tau}(\alpha)\big),
\end{align}
which proves Eq.~\eqref{eq:aggregate_neighbor_certificate}.

Finally, we prove the geometric certificate. For any
\(x'\in\mathcal{N}_{\mathrm{same}}^{\tau}(x)\), Assumption~2 gives
\begin{equation}
    p_T(c\mid x)
    \ge p_T(c\mid x') - Ld(x,x').
\end{equation}
Combining this inequality with the lower bound on \(p_T(c\mid x')\) gives
\begin{equation}
    p_T(c\mid x)
    \ge \tau-\beta_{\tau}(\alpha)-Ld(x,x').
\end{equation}
Averaging over \(x'\in\mathcal{N}_{\mathrm{same}}^{\tau}(x)\) yields
\begin{equation}
    p_T(c\mid x)
    \ge
    \tau-\beta_{\tau}(\alpha)
    -
    L\frac{1}{|\mathcal{N}_{\mathrm{same}}^{\tau}(x)|}
    \sum_{x'\in\mathcal{N}_{\mathrm{same}}^{\tau}(x)}d(x,x'),
\end{equation}
which is Eq.~\eqref{eq:geometric_neighbor_certificate}. If
\(\tau-\beta_{\tau}(\alpha)>1/2\), then Eq.~\eqref{eq:center_noise_bound}
implies \(\eta(x)<1/2\) for every \(x\in\mathcal{H}_{\tau,\rho}\), which is the
Massart bounded-noise condition on the selected safe subspace. 

\paragraph{Remark on the role of \(\rho\).}
The proof intentionally separates the pointwise noise bound from the structural
certificates. Under Assumption~1, the pointwise pseudo-label noise at the center
sample is already controlled by \(\tau\) and \(\beta_{\tau}(\alpha)\). The
neighborhood ratio \(\rho\) should not be interpreted as a factor that directly
multiplies the center-point posterior. Instead, \(\rho\) provides observable
target-domain evidence that the selected pseudo-label is locally supported rather
than an isolated high-confidence error. This is why \(\rho\) enters the empirical
safe-subspace selection rule in Eq.~\eqref{eq:safe_subspace_final}, while the
Massart noise condition follows from Eq.~\eqref{eq:center_noise_bound}.

\subsubsection{Proof of Proposition~\ref{prop:finite_sample_structural_gate}}

For a fixed target representation \(x\), define the Bernoulli variable
\begin{equation}
B_i=\mathbb{I}[\tilde y(x_i)=\tilde y(x),\ s(x_i)\ge \tau],
\end{equation}
where \(x_i\) is the \(i\)-th sampled neighbor of \(x\). The empirical structural gate is
\begin{equation}
\hat r_{\tau}(x)=\frac{1}{k_{\mathrm{nn}}}\sum_{i=1}^{k_{\mathrm{nn}}}B_i,
\end{equation}
and its population counterpart is \(r_{\tau}(x)=\mathbb{E}[B_i]\). Hoeffding's inequality gives
\begin{equation}
\Pr\left[\left|\hat r_{\tau}(x)-r_{\tau}(x)\right|\ge \epsilon\right]
\le 2\exp(-2k_{\mathrm{nn}}\epsilon^2).
\end{equation}
Setting the right-hand side to \(\delta\) proves Eq.~\eqref{eq:structural_gate_concentration}. If \(\hat r_{\tau}(x)\ge\rho_{\min}+\epsilon\), then with the same probability,
\begin{equation}
r_{\tau}(x)\ge \hat r_{\tau}(x)-\epsilon\ge \rho_{\min}.
\end{equation}
This proves the finite-sample reliability statement. 

\subsubsection{Proof of Corollary~\ref{cor:voting_bound}}

Fix \(x\in\mathcal{H}_{\tau,\rho}\). Let
\(Z_k=\mathbb{I}[f_S^k(x)\neq y]\) be the error indicator of the \(k\)-th source
expert. Under the idealized condition in Corollary~\ref{cor:voting_bound}, the
variables \(Z_1,\ldots,Z_{K_e}\) are conditionally independent given \(x\), and
\(\mathbb{E}[Z_k\mid x]\le \bar e<1/2\). Majority voting is wrong only when
\begin{equation}
    \frac{1}{K_e}\sum_{k=1}^{K_e}Z_k \ge \frac{1}{2}.
\end{equation}
Hoeffding's inequality gives
\begin{align}
    \Pr\!\left[\frac{1}{K_e}\sum_{k=1}^{K_e}Z_k\ge\frac{1}{2}\mid x\right]
    &\le
    \Pr\!\left[
    \frac{1}{K_e}\sum_{k=1}^{K_e}Z_k-\bar e
    \ge \frac{1}{2}-\bar e
    \mid x
    \right] \\
    &\le
    \exp\!\left(-2K_e\left(\frac{1}{2}-\bar e\right)^2\right).
\end{align}
This proves Eq.~\eqref{eq:voting_bound}. 

The conditional-independence assumption is only a sufficient condition for the
closed-form exponential bound. In practice, source checkpoints trained on the
same source domain can have correlated errors. For this reason, \method{} does
not use Eq.~\eqref{eq:voting_bound} as an unconditional performance guarantee.
Instead, it uses the committee to measure empirical consensus and predictive
variance through Eqs.~\eqref{eq:consensus}--\eqref{eq:variance}. A target graph is
admitted into \(\mathcal{S}_{\mathrm{sem}}\) only when the committee is both
confident and low-disagreement, and it must further pass the structural gate in
Eq.~\eqref{eq:safe_subspace_final} before hard pseudo-label supervision is used.

\subsubsection{Proof of Proposition~\ref{prop:target_risk_bound}}

Let \(\mathcal{U}=\mathcal{X}_T\setminus\mathcal{H}_{\tau,\rho}\). The target risk
can be decomposed as
\begin{align}
    R_T(h)
    &=
    \Pr[h(x)\neq y,\,x\in\mathcal{H}_{\tau,\rho}]
    +
    \Pr[h(x)\neq y,\,x\in\mathcal{U}] \\
    &=
    \pi_{\mathcal H}
    \Pr[h(x)\neq y\mid x\in\mathcal{H}_{\tau,\rho}]
    +
    \pi_{\mathcal U}R_{\mathcal U}(h).
\end{align}
On the safe subspace, the event \(h(x)\neq y\) is contained in the union of the
two events \(h(x)\neq\tilde y(x)\) and \(\tilde y(x)\neq y\). Hence
\begin{align}
    \Pr[h(x)\neq y\mid x\in\mathcal{H}_{\tau,\rho}]
    &\le
    \Pr[h(x)\neq\tilde y(x)\mid x\in\mathcal{H}_{\tau,\rho}]
    +
    \Pr[\tilde y(x)\neq y\mid x\in\mathcal{H}_{\tau,\rho}] \\
    &\le
    \widetilde R_{\mathcal H}(h)+\bar\eta_{\tau}.
\end{align}
Substituting this inequality into the risk decomposition gives
Eq.~\eqref{eq:target_risk_bound_basic}. If the uncertain-set risk is upper
bounded by a regularizer-induced quantity
\(R_{\mathcal U}(h)\le\Gamma_{\mathcal U}(h)\), then replacing
\(R_{\mathcal U}(h)\) with \(\Gamma_{\mathcal U}(h)\) gives
Eq.~\eqref{eq:target_risk_bound_reg}. 

\subsubsection{A finite-sample variant}

The previous proposition is a population decomposition. For completeness, we
state a standard finite-sample form that explains the empirical objective in
Eq.~\eqref{eq:total_loss}. Suppose \(m\) selected target graphs are sampled from
\(\mathcal{H}_{\tau,\rho}\), and let \(\widehat R_{\mathcal H}(h)\) be the
empirical pseudo-label risk on them. For a bounded classification loss and a
hypothesis class with Rademacher complexity \(\mathfrak{R}_m(\mathcal{H})\), with
probability at least \(1-\delta\), uniformly over \(h\),
\begin{equation}
    \widetilde R_{\mathcal H}(h)
    \le
    \widehat R_{\mathcal H}(h)
    +2\mathfrak{R}_m(\mathcal{H})
    +\sqrt{\frac{\log(2/\delta)}{2m}}.
\end{equation}
Combining this inequality with Proposition~\ref{prop:target_risk_bound} gives
\begin{equation}
    R_T(h)
    \le
    \pi_{\mathcal H}
    \left(
    \widehat R_{\mathcal H}(h)
    +\bar\eta_{\tau}
    +2\mathfrak{R}_m(\mathcal{H})
    +\sqrt{\frac{\log(2/\delta)}{2m}}
    \right)
    +
    \pi_{\mathcal U}\Gamma_{\mathcal U}(h).
\end{equation}
This finite-sample variant should be read as a conventional generalization
statement for the selected pseudo-labeled subset. It is not used to claim that
pseudo-labels are reliable on the full target domain.

\subsection{Connection to Full-Domain Unbounded-Noise Analyses}
\label{app:connection_unbounded}

The safe-subspace guarantee does not contradict the full-domain unbounded-noise
result in Theorem~\ref{thm:unbounded_noise_sfda}. The two statements analyze
different adaptation protocols. The unbounded-noise result shows that if a fixed
source classifier is applied to all shifted target samples, then there may exist
a target region on which pseudo-labels are almost surely wrong. Our analysis
instead studies a selective protocol in which hard pseudo-label supervision is
restricted to
\begin{equation}
    \mathcal{H}_{\tau,\rho}
    =
    \{x:s(x)\ge\tau,\ r_{\tau}(x)\ge\rho\}.
\end{equation}
Samples outside this set are not discarded; they are used through soft target
regularization. Therefore, the theory should be interpreted as a sufficient
condition for reliable selective adaptation, not as a global claim that every
source-free pseudo-label is bounded-noise.

This distinction also explains the design of \method{}. Module I estimates
semantic reliability through source confidence and ensemble disagreement. Module
II learns a target-intrinsic representation in which neighborhood evidence is
more meaningful. Module III intersects semantic and structural evidence to form
\(\mathcal{H}_{\zeta,\rho}\). Module IV avoids assigning hard labels to the
remaining uncertain samples and instead controls them through entropy and KL
regularization.

\begin{figure*}[t]
\includegraphics[width=1.0\linewidth]{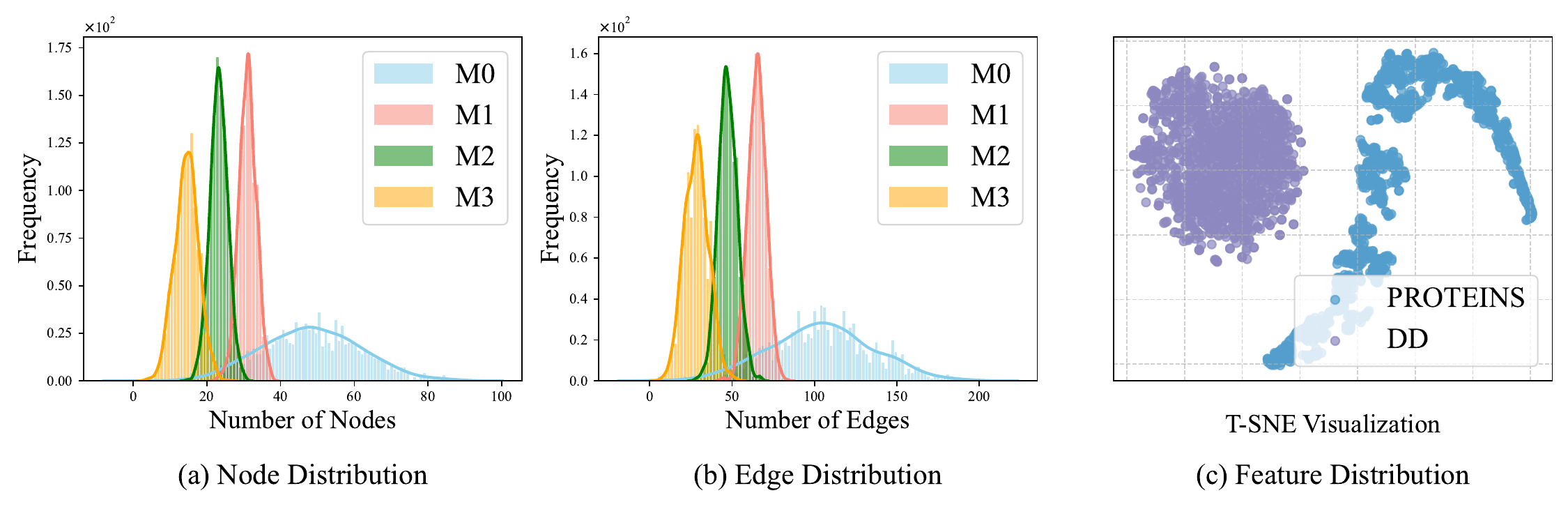}
\vspace{-0.8cm}
    \caption{Visualization of domain shifts across different types. (a,b) Node and edge distribution shifts between sub-datasets of Mutagenicity. (c) Feature distribution shift between PROTEINS and DD datasets.}
    \label{fig:shift}
\end{figure*}

\begin{table*}[t]
    \centering
    \small
    \caption{Statistics of the experimental datasets.}
    \label{tab:dataset}
    \resizebox{0.55\columnwidth}{!}{
    \begin{tabular}{lcccc}
        \toprule
        Datasets & Graphs & Avg. Nodes & Avg. Edges & Classes \\
        \midrule
        DD & 1,178 & 284.32 & 715.66 & 2 \\
        NCI1 & 4,110 & 29.87 & 32.30 & 2 \\
        Mutagenicity & 4,337 & 30.32 & 30.77 & 2 \\
        FRANKENSTEIN & 4,337 & 16.9 & 17.88 & 2 \\
        ogbg-molhiv & 41,127 & 25.5 & 27.5 & 2 \\
        CIFAR10 & 60,000 & 117.6 & 941.2 & 10 \\
        MNIST & 70,000 & 70.6 & 564.5 & 10 \\
        
        \midrule
        
        PROTEINS & 1,113 & 39.1 & 72.8 & 2 \\
        COX2 & 467 & 41.22 & 43.45 & 2 \\
        COX2\_MD & 303 & 26.28 & 335.12 & 2 \\
        BZR & 405 & 35.75 & 38.36 & 2 \\
        BZR\_MD & 306 & 21.30 & 225.06 & 2 \\
        
        \bottomrule
    \end{tabular}
    }
\end{table*}

\subsection{Dataset}\label{sec:dataset}

\subsubsection{Dataset Description}

We provide additional details of the datasets used in our experiments. 
For structure-based domain shifts, we use real-world graph benchmarks and image-derived graph benchmarks. 
For graph datasets, molecules are represented as graphs with atoms as nodes and chemical bonds as edges, while proteins are represented with amino acids as nodes and spatial or chemical proximity relations as edges. 
Following~\citep{yin2022deal,yin2023coco,yin2025dream}, we construct structural domains by sorting graphs according to graph-level structural statistics. 
Specifically, node-density shifts are generated by sorting graphs according to the number of nodes $|V|$, while edge-density shifts are generated by sorting graphs according to the number of edges $|E|$. 
The sorted graphs are then divided into quantile-based subdomains with a shared label space.

\noindent\textbf{(1) Structure-based domain shifts.}
\begin{itemize}[leftmargin=*]
\item \textbf{DD.} 
DD~\citep{dobson2003distinguishing} contains 1,178 protein graphs for binary classification, where each graph is labeled according to whether the protein is an enzyme. 
Compared with PROTEINS, DD graphs are typically larger and denser, making the dataset suitable for evaluating structural shifts with higher graph complexity. 
We partition DD into four subdomains, denoted as D0, D1, D2, and D3.
\begin{table*}[t]
\centering
\caption{Subdomain statistics for density-based partitions. \textit{Avg. Struct.} and \textit{Std. Struct.} denote the mean and standard deviation of the corresponding structural metrics. Class 0 and Class 1 are reported only for binary datasets; class counts for MNIST and CIFAR10 are omitted because they contain 10 classes.}
\label{tab:subdomain_stats}
\begin{tabular}{llcccccc}
\toprule
Dataset & Split & Domain & \#Graphs & Class 0 & Class 1 & Avg. Struct. & Std. Struct. \\
\midrule
Mutagenicity & Node & M0 & 1084 & 515 & 569 & 52.71 & 28.12 \\
Mutagenicity & Node & M1 & 1084 & 730 & 354 & 31.00 & 2.37 \\
Mutagenicity & Node & M2 & 1084 & 649 & 435 & 23.16 & 2.42 \\
Mutagenicity & Node & M3 & 1085 & 507 & 578 & 14.41 & 3.32 \\
Mutagenicity & Edge & M0 & 1084 & 557 & 527 & 52.80 & 16.91 \\
Mutagenicity & Edge & M1 & 1084 & 723 & 361 & 32.55 & 2.70 \\
Mutagenicity & Edge & M2 & 1084 & 618 & 466 & 23.66 & 2.81 \\
Mutagenicity & Edge & M3 & 1085 & 503 & 582 & 14.08 & 3.64 \\
NCI1 & Node & N0 & 1027 & 277 & 750 & 47.97 & 13.78 \\
NCI1 & Node & N1 & 1027 & 504 & 523 & 30.59 & 2.43 \\
NCI1 & Node & N2 & 1027 & 588 & 439 & 23.83 & 1.72 \\
NCI1 & Node & N3 & 1029 & 684 & 345 & 17.09 & 3.03 \\
NCI1 & Edge & N0 & 1027 & 281 & 746 & 52.25 & 14.98 \\
NCI1 & Edge & N1 & 1027 & 524 & 503 & 33.16 & 2.67 \\
NCI1 & Edge & N2 & 1027 & 546 & 481 & 25.79 & 1.86 \\
NCI1 & Edge & N3 & 1029 & 702 & 327 & 18.04 & 3.54 \\
FRANKENSTEIN & Node & F0 & 1084 & 468 & 616 & 28.46 & 14.17 \\
FRANKENSTEIN & Node & F1 & 1084 & 355 & 729 & 18.13 & 1.42 \\
FRANKENSTEIN & Node & F2 & 1084 & 457 & 627 & 13.08 & 1.47 \\
FRANKENSTEIN & Node & F3 & 1085 & 656 & 429 & 7.93 & 2.06 \\
FRANKENSTEIN & Edge & F0 & 1084 & 434 & 650 & 31.13 & 14.81 \\
FRANKENSTEIN & Edge & F1 & 1084 & 363 & 721 & 19.54 & 1.70 \\
FRANKENSTEIN & Edge & F2 & 1084 & 484 & 600 & 13.43 & 1.74 \\
FRANKENSTEIN & Edge & F3 & 1085 & 655 & 430 & 7.42 & 2.35 \\
DD & Node & D0 & 294 & 54 & 240 & 96.69 & 28.71 \\
DD & Node & D1 & 295 & 153 & 142 & 191.05 & 29.01 \\
DD & Node & D2 & 294 & 233 & 61 & 298.36 & 32.55 \\
DD & Node & D3 & 295 & 251 & 44 & 550.58 & 421.71 \\
DD & Edge & D0 & 294 & 57 & 237 & 229.79 & 70.62 \\
DD & Edge & D1 & 295 & 153 & 142 & 470.15 & 74.44 \\
DD & Edge & D2 & 294 & 227 & 67 & 746.97 & 82.80 \\
DD & Edge & D3 & 295 & 254 & 41 & 1414.19 & 1058.74 \\
ogbg-molhiv & Node & H0 & 10281 & 10080 & 201 & 14.90 & 2.65 \\
ogbg-molhiv & Node & H1 & 10281 & 9972 & 309 & 20.74 & 1.35 \\
ogbg-molhiv & Node & H2 & 10281 & 9987 & 294 & 25.92 & 1.85 \\
ogbg-molhiv & Node & H3 & 10284 & 9645 & 639 & 40.48 & 14.66 \\
ogbg-molhiv & Edge & H0 & 10281 & 10079 & 202 & 15.54 & 3.07 \\
ogbg-molhiv & Edge & H1 & 10281 & 9960 & 321 & 22.31 & 1.55 \\
ogbg-molhiv & Edge & H2 & 10281 & 9994 & 287 & 28.15 & 2.08 \\
ogbg-molhiv & Edge & H3 & 10284 & 9651 & 633 & 43.87 & 15.61 \\
MNIST & Edge & S0 & 23334 & - & - & 204.40 & 20.20 \\
MNIST & Edge & S1 & 23333 & - & - & 273.80 & 26.26 \\
MNIST & Edge & S2 & 23333 & - & - & 360.41 & 33.79 \\
CIFAR10 & Edge & C0 & 20000 & - & - & 400.52 & 14.96 \\
CIFAR10 & Edge & C1 & 20000 & - & - & 593.02 & 21.92 \\
CIFAR10 & Edge & C2 & 20000 & - & - & 795.90 & 29.14 \\
\bottomrule
\end{tabular}
\end{table*}
\item \textbf{NCI1.} 
NCI1~\citep{wale2008comparison} consists of 4,110 molecular graphs. 
Each graph corresponds to a chemical compound, and the prediction task is to determine whether the compound is active against cancer cell growth. 
We divide NCI1 into four structure-based subdomains, denoted as N0, N1, N2, and N3.

\item \textbf{Mutagenicity.} 
Mutagenicity~\citep{kazius2005derivation} contains 4,337 molecular graphs labeled by mutagenic effect. 
The dataset contains diverse molecular sizes and bonding patterns, which naturally induce structural heterogeneity. 
We split it into four subdomains, denoted as M0, M1, M2, and M3.

\item \textbf{FRANKENSTEIN.} 
FRANKENSTEIN~\citep{orsini2015graph} includes 4,337 molecular graphs for biological activity prediction. 
The graphs are relatively small compared with several other molecular benchmarks, providing a complementary setting for evaluating adaptation across different graph-size regimes. 
We partition the dataset into F0, F1, F2, and F3.

\item \textbf{ogbg-molhiv.} 
ogbg-molhiv~\citep{hu2021ogblsc} is an OGB molecular benchmark with 41,127 graphs, where the task is to predict whether a molecule inhibits HIV replication. 
Compared with the TUDataset benchmarks, it is larger and more class-imbalanced, offering a more realistic molecular adaptation scenario. 
We construct four structural subdomains, denoted as H0, H1, H2, and H3.

\item \textbf{MNIST.} 
MNIST~\citep{lecun2002gradient} contains 70,000 grayscale digit images from 10 classes. 
We convert each image into a graph representation, where nodes correspond to pixels or superpixels and edges encode spatial adjacency. 
The resulting graphs are partitioned into three edge-based subdomains, denoted as S0, S1, and S2.

\item \textbf{CIFAR-10.} 
CIFAR-10~\citep{krizhevsky2009learning} contains 60,000 color images from 10 object categories. 
Similar to MNIST, each image is transformed into a graph with local image regions as nodes and spatial neighborhood relations as edges. 
We construct three edge-based subdomains, denoted as C0, C1, and C2.
\end{itemize}

\noindent\textbf{(2) Feature-based domain shifts.}
For feature-based domain shifts, we evaluate paired domains that share the same prediction semantics but differ in node-feature distributions. 
These tasks test whether a method can adapt when the global graph-level semantics are comparable but the local feature statistics change across domains.

\begin{itemize}[leftmargin=*]
\item \textbf{PROTEINS and DD.} 
PROTEINS~\citep{dobson2003distinguishing} contains 1,113 protein graphs for binary graph classification, with labels indicating enzyme or non-enzyme classes. 
Together with DD, it forms the PROTEINS$\leftrightarrow$DD transfer pair. 
This pair evaluates adaptation across protein graph domains with related label semantics but different graph and feature characteristics.

\item \textbf{COX2 and COX2\_MD.} 
COX2 contains 467 molecular graphs, while COX2\_MD contains 303 molecular graphs from a modified domain with the same label space. 
Both datasets describe molecular structures and are used as the COX2$\leftrightarrow$COX2\_MD transfer pair. 
This setting evaluates feature-based distribution shifts between related molecular domains.

\item \textbf{BZR and BZR\_MD.} 
BZR contains 405 molecular graphs, and BZR\_MD contains 306 molecular graphs from the corresponding modified domain. 
The two datasets share the same binary prediction semantics and form the BZR$\leftrightarrow$BZR\_MD transfer pair. 
This pair provides another molecular feature-shift benchmark with consistent label space but shifted node-feature statistics.
\end{itemize}

Table~\ref{tab:subdomain_stats} provides detailed statistics of the constructed subdomains for all structure-based benchmarks. 
For each dataset and split type, it reports the subdomain identifier, number of graphs, class distribution, and the mean and standard deviation of the corresponding structural statistic. 
For node-density splits, Avg. Struct. and Std. Struct. refer to the number of nodes; for edge-density splits, they refer to the number of edges used to characterize graph connectivity. 
These statistics show that the constructed subdomains exhibit clear differences in graph scale or connectivity while maintaining the same label space, thereby forming meaningful structural domain shifts. 
They also reveal that density-based partitioning may introduce class-prior variations across subdomains, which makes the source-free adaptation setting more realistic and challenging.

\subsubsection{Data Processing}

For real-world graph datasets from TUDataset\footnote{\url{https://chrsmrrs.github.io/datasets/}}, including DD, PROTEINS, Mutagenicity, NCI1, FRANKENSTEIN, BZR, BZR\_MD, COX2, and COX2\_MD, we follow the standard preprocessing pipeline provided by PyTorch Geometric\footnote{\url{https://pyg.org/}}. 
Each graph is converted into a PyG data object, with the original node attributes or node labels used as input features when available. For image benchmarks, i.e., MNIST and CIFAR10, we transform each image into a graph representation. Nodes correspond to pixels or superpixels, and edges are constructed according to spatial proximity. Specifically, we build KNN graphs based on node spatial coordinates, where the neighborhood size is adjusted to control graph density and induce edge-density domain shifts~\citep{dwivedi2023benchmarking, wang2026dsbd}. The original image category is used as the graph-level label. For datasets from the Open Graph Benchmark (OGB)\footnote{\url{https://ogb.stanford.edu/}}, such as ogbg-molhiv, we follow the official OGB preprocessing and evaluation protocol. The molecular graphs are loaded with their provided atom and bond features, and the ROC-AUC metric is used following the standard OGB setting.

\subsection{Baselines}
\label{app:baseline_protocol}

We compare the proposed \method{} with a comprehensive set of competitive baselines on the datasets above. These baselines include six graph kernels and general graph neural networks, including WL~\citep{shervashidze2011weisfeiler}, PathNN~\citep{michel2023path}, GCN~\citep{kipf2022semi}, GIN~\citep{xu2018powerful}, CIN~\citep{bodnar2021weisfeiler}, and GMT~\citep{baek2021accurate}; four source-free domain adaptation methods: SFDA\_LLN~\citep{yi2023sfda_noisy}, SF(DA)$^2$~\citep{hwang2024sf},  NVC\_LLN~\citep{xu2025unraveling}, and Ucon\_SFDA~\cite{xu2025revisiting}; four source-free graph domain adaptation methods: SOGA~\citep{mao2024sourcefreegraph}, GraphCTA~\citep{zhang2024collaborate}, GALA~\citep{luo2024gala}, and GraphATA~\cite{zhang2025aggregate}.

\begin{itemize}
    \item \textbf{WL}: The Weisfeiler–Lehman (WL) graph kernel~\citep{shervashidze2011weisfeiler} captures hierarchical graph structures by iteratively relabeling nodes based on neighborhood information. It computes graph similarities through subtree pattern matching, enabling efficient and expressive structural comparison.
    \item \textbf{PathNN}: Path Neural Networks (PathNN)~\citep{michel2023path} enhance the expressive power of graph neural networks by modeling relationships along paths rather than only adjacent nodes. They aggregate path-based representations through attention mechanisms, effectively capturing higher-order structural dependencies and improving graph-level prediction accuracy.

    \item  \textbf{GCN}: Graph Convolutional Networks (GCN)~\citep{kipf2022semi} perform graph learning by propagating and aggregating feature information from neighboring nodes through spectral convolutions. This approach effectively captures local graph structure while enabling efficient end-to-end training.
    \item \textbf{GIN}: Graph Isomorphism Networks (GIN)~\citep{xu2018powerful} achieve maximal discriminative power among message-passing GNNs by using sum aggregation and multi-layer perceptrons to capture injective neighborhood functions, enabling them to distinguish complex graph structures equivalent to the Weisfeiler–Lehman test. 
    \item \textbf{CIN}: Cellular Weisfeiler–Lehman Networks (CIN)~\citep{bodnar2021weisfeiler} extend the Weisfeiler–Lehman framework to higher-dimensional cellular complexes, enabling message passing beyond edges to higher-order cells. This design captures multi-scale structural interactions and richer topological dependencies in complex graphs. 
    \item \textbf{GMT}: Graph Multiset Transformer (GMT) \citep{baek2021accurate} introduces attention-based multiset pooling to learn expressive graph-level representations. By employing learnable queries to aggregate node embeddings, GMT flexibly captures diverse structural patterns and global dependencies for accurate graph representation learning. 
    \item \textbf{SFDA\_LLN:} SFDA\_LLN~\citep{yi2023sfda_noisy} introduces a robust source-free domain adaptation framework that models label noise as unbounded and exploits the early-time training phenomenon (ETP) to prevent noise memorization. By incorporating early-learning regularization into the SFDA objective, it effectively mitigates noisy pseudo-labels and significantly improves cross-domain adaptation performance.
    \item \textbf{SF(DA)$^2$}: SF(DA)$^2$~\citep{hwang2024sf} reinterprets source-free domain adaptation from a data augmentation perspective. It generates diverse target-domain variants to simulate source data, enabling model adaptation through consistency regularization and self-training without requiring access to source samples.
    \item \textbf{NVC\_LLN:} NVC\_LLN~\citep{xu2025unraveling} provides a theoretical and empirical analysis of label noise in SFDA, distinguishing between bounded and unbounded noise regimes. It introduces principled strategies to mitigate noise-induced degradation, offering both theoretical insights and practical algorithms that enhance adaptation robustness across domains.
    \item \textbf{Ucon\_SFDA:} Ucon\_SFDA~\citep{xu2025revisiting} revisits source-free domain adaptation from the perspective of uncertainty control. It explicitly models and regulates target-side uncertainty to reduce unreliable pseudo-labeling, thereby improving adaptation robustness without accessing source data.

    \item \textbf{SOGA:} SOGA~\citep{mao2024sourcefreegraph}  introduces a graph-based framework for source-free domain adaptation. It transfers knowledge from a pre-trained source model to the target domain by aligning graph structural features and node representations without accessing source data, leveraging graph topology to enhance domain-invariant learning and cross-domain generalization.
    \item \textbf{GraphCTA:} GraphCTA~\citep{zhang2024collaborate} proposes source-free graph domain adaptation via bi-directional collaboration between source hypothesis and target structure, iteratively refining pseudo-labels and representations through consistency, topology-aware alignment, and cooperative training to improve cross-domain generalization.
    \item \textbf{GALA:} GALA~\citep{luo2024gala} introduces a diffusion-driven graph alignment framework that preserves structural consistency between domains. It employs a jigsaw-style self-supervised task to enhance feature discrimination and jointly optimizes diffusion alignment and representation learning, enabling robust adaptation without access to source data.
    \item \textbf{GraphATA:} GraphATA~\citep{zhang2025aggregate} addresses multi-source-free graph domain adaptation through node-centric aggregation. It dynamically aggregates knowledge from multiple source hypotheses according to local graph context, enabling fine-grained target adaptation without using source graphs.
\end{itemize}

All methods are evaluated on the same source-target splits. Results
are reported as mean performance with standard deviation over five random seeds. Small numerical differences whose standard-deviation ranges overlap are treated
as empirically comparable rather than as strict wins. This convention is
important because several transfer pairs have margins that are smaller than the
run-to-run variance.

\subsection{Algorithm}

The complete training procedure is summarized in Algorithm~\ref{alg:s2plr}.
\begin{algorithm}[t]
\caption{Training Procedure of \method{}}
\label{alg:s2plr}
\begin{algorithmic}[1]
\State \textbf{Input:} Source experts \(\mathcal{M}\), target data \(\mathcal{D}^t\), thresholds \(\zeta,\rho_{\min}\), variance percentile \(q_u\), neighborhood size \(k_{\mathrm{nn}}\), maximum epochs \(E_{\max}\).
\State \textbf{Output:} Adapted target model \(f_T\).
\State Initialize \(f_T\) from the primary source checkpoint and initialize the target structure encoder \(g_{\psi}\).
\State Warm up \(g_{\psi}\) on target graphs by minimizing \(\mathcal{L}_{\mathrm{GCL}}\) in Eq.~\eqref{eq:gcl_loss}.
\For{epoch \(=1\) to \(E_{\max}\)}
    \State Compute ensemble consensus \(\bar p_j\), pseudo-label \(\hat y_j\), confidence \(s_j\), and variance \(u_j\) for all \(G_j^t\in\mathcal{D}^t\).
    \State Set \(\nu=Q_{q_u}(\{u_j\}_{j=1}^{n_t})\).
    \State Select semantic candidates \(\mathcal{S}_{\mathrm{sem}}\) using Eq.~\eqref{eq:selection_criteria}.
    \State Compute target structural embeddings with \(g_{\psi}\) and construct the \(k_{\mathrm{nn}}\)-nearest-neighbor graph.
    \State Compute \(\rho_j\) using Eq.~\eqref{eq:consistency_score} and identify \(\mathcal{H}_{\zeta,\rho}\) using Eq.~\eqref{eq:safe_subspace_final}.
    \State Set \(\mathcal{U}=\mathcal{D}^t\setminus\mathcal{H}_{\zeta,\rho}\).
    \For{each mini-batch \(\mathcal{B}\subset\mathcal{D}^t\)}
        \State Compute \(\mathcal{L}_{\mathcal H}\) on \(\mathcal{B}\cap\mathcal{H}_{\zeta,\rho}\) using Eq.~\eqref{eq:safe_ce_loss}.
        \State Compute \(\mathcal{L}_{\mathrm{reg}}\) on \(\mathcal{B}\cap\mathcal{U}\) using Eq.~\eqref{eq:reg_loss}.
        \State Compute \(\mathcal{L}_{\mathrm{GCL}}\) on \(\mathcal{B}\) using Eq.~\eqref{eq:gcl_loss}.
        \State Update \(f_T\) and \(g_{\psi}\) by minimizing Eq.~\eqref{eq:total_loss}.
    \EndFor
\EndFor
\State \textbf{return} \(f_T\).
\end{algorithmic}
\end{algorithm}

\begin{table*}[t]
\centering
\small
\caption{Paired Wilcoxon signed-rank tests on reported task-level mean performance over all node- and edge-shift transfer tasks, with Holm correction for multiple comparisons.}
\label{tab:significance_tests}
\begin{tabular}{lcccc}
\toprule
Comparison & Test Scope & Test Statistic $W$ & Holm-adjusted $p$-value & Significant at $0.05$ \\
\midrule
\method{} vs. GraphATA 
& 132 node/edge tasks & 849.0 & $8.94\times 10^{-16}$ & Yes \\

\method{} vs. GALA 
& 132 node/edge tasks & 810.0 & $8.65\times 10^{-16}$ & Yes \\

\method{} vs. GraphCTA 
& 132 node/edge tasks & 521.5 & $7.43\times 10^{-18}$ & Yes \\

\method{} vs. SOGA 
& 132 node/edge tasks & 378.0 & $3.29\times 10^{-19}$ & Yes \\

\method{} vs. UCon\_SFDA 
& 132 node/edge tasks & 272.5 & $4.39\times 10^{-20}$ & Yes \\

\method{} vs. NVC\_LLN 
& 132 node/edge tasks & 73.5 & $8.84\times 10^{-22}$ & Yes \\

\method{} vs. SF(DA)$^2$ 
& 132 node/edge tasks & 137.5 & $4.14\times 10^{-21}$ & Yes \\

\method{} vs. SFDA\_LLN 
& 132 node/edge tasks & 103.0 & $1.50\times 10^{-21}$ & Yes \\
\bottomrule
\end{tabular}
\end{table*}

\subsection{More experimental results}

\subsubsection{More performance comparison}\label{sec:more_experiments}

In this section, we present additional performance comparisons between \method{} and all baseline methods on more graph benchmarks, as summarized in Tables~\ref{tab:dd_node}--\ref{tab:hiv_idx}. 
The results cover node-density and edge-density domain shifts on DD, FRANKENSTEIN, Mutagenicity, NCI1, and ogbg-molhiv. Overall, \method{} achieves the best or consistently competitive performance in most transfer settings, further demonstrating its robustness across molecular and protein graph domains.

Compared with conventional graph classifiers and kernel-based methods, \method{} shows clear advantages because it explicitly addresses source-target structural mismatch during adaptation. 
Although existing source-free and graph adaptation baselines improve over non-adaptive models, their performance still varies across transfer directions, especially under large graph-size or connectivity shifts. 
The paired Wilcoxon signed-rank tests in Table~\ref{tab:significance_tests} further support this observation. 
Based on the reported task-level mean performance over 132 node- and edge-shift transfer tasks, \method{} achieves statistically significant improvements over representative source-free and graph adaptation baselines after Holm correction, including GraphATA, GALA, GraphCTA, SOGA, Ucon\_SFDA, NVC\_LLN, SF(DA)$^2$, and SFDA\_LLN. 
This indicates that the gains of \method{} are not concentrated on a few favorable transfer pairs, but are consistent across a broad range of structural domain shifts. 
In particular, the significant improvements over strong graph-specific baselines such as GraphATA, GALA, and GraphCTA suggest that topology-aware adaptation alone is not sufficient; reliable source-free graph adaptation also benefits from multi-expert semantic reliability estimation, target-intrinsic structure learning, neighborhood-consistency verification, and soft regularization on uncertain samples. 
These results reinforce the conclusion that effective source-free graph adaptation requires both conservative safe-subspace supervision and principled exploitation of the remaining uncertain target data.

\subsubsection{More Ablation study}\label{sec:more_ablation}

To further validate the effectiveness of each component in \method{}, we conduct additional ablation studies on the DD, NCI1, FRANKENSTEIN, and ogbg-molhiv datasets. 
Specifically, we evaluate five ablated variants of \method{}, including \method{} w/o ME, \method{} w/o CF, \method{} w/o TS, \method{} w/o NC, and \method{} w/o SR. 
The corresponding experimental results are reported in Tables~\ref{tab:ablation_study_dd}, \ref{tab:ablation_study_nci1}, \ref{tab:ablation_study_frank}, and \ref{tab:ablation_study_molhiv}. 
Overall, the observed trends are consistent with those discussed in Section~\ref{sec:ablation}, further confirming that each component contributes to the robustness of \method{} under source-free graph domain adaptation.

These ablations also correspond to the theoretical risk decomposition in Proposition~\ref{prop:target_risk_bound}. 
Removing multi-expert uncertainty quantification weakens the identification of consensus pseudo-labels and makes the safe set more vulnerable to source-biased predictions. 
Removing the confidence filter primarily weakens the control of the pseudo-label noise term $\bar\eta_{\tau}$, while removing the neighborhood-consistency gate weakens the empirical certification of the safe subspace $\mathcal{H}_{\tau,\rho}$. 
Removing the target-intrinsic structure branch makes neighborhood construction less reliable and may enlarge the effective radius term $L\bar r$. 
Removing the soft regularization on the uncertain set $\mathcal{U}$ weakens the control of the uncertain-set risk term $\Gamma_{\mathcal U}(h)$. 
Beyond accuracy, safe-subspace coverage and selected pseudo-label precision are direct post-hoc diagnostics for the selective-learning mechanism. 
When target labels are used for such analysis, these diagnostics are reported separately from model selection and do not affect the source-free adaptation protocol.

\begin{figure*}[t]
    \centering

    \begin{subfigure}[t]{0.24\linewidth}
        \centering
        \includegraphics[width=\linewidth]{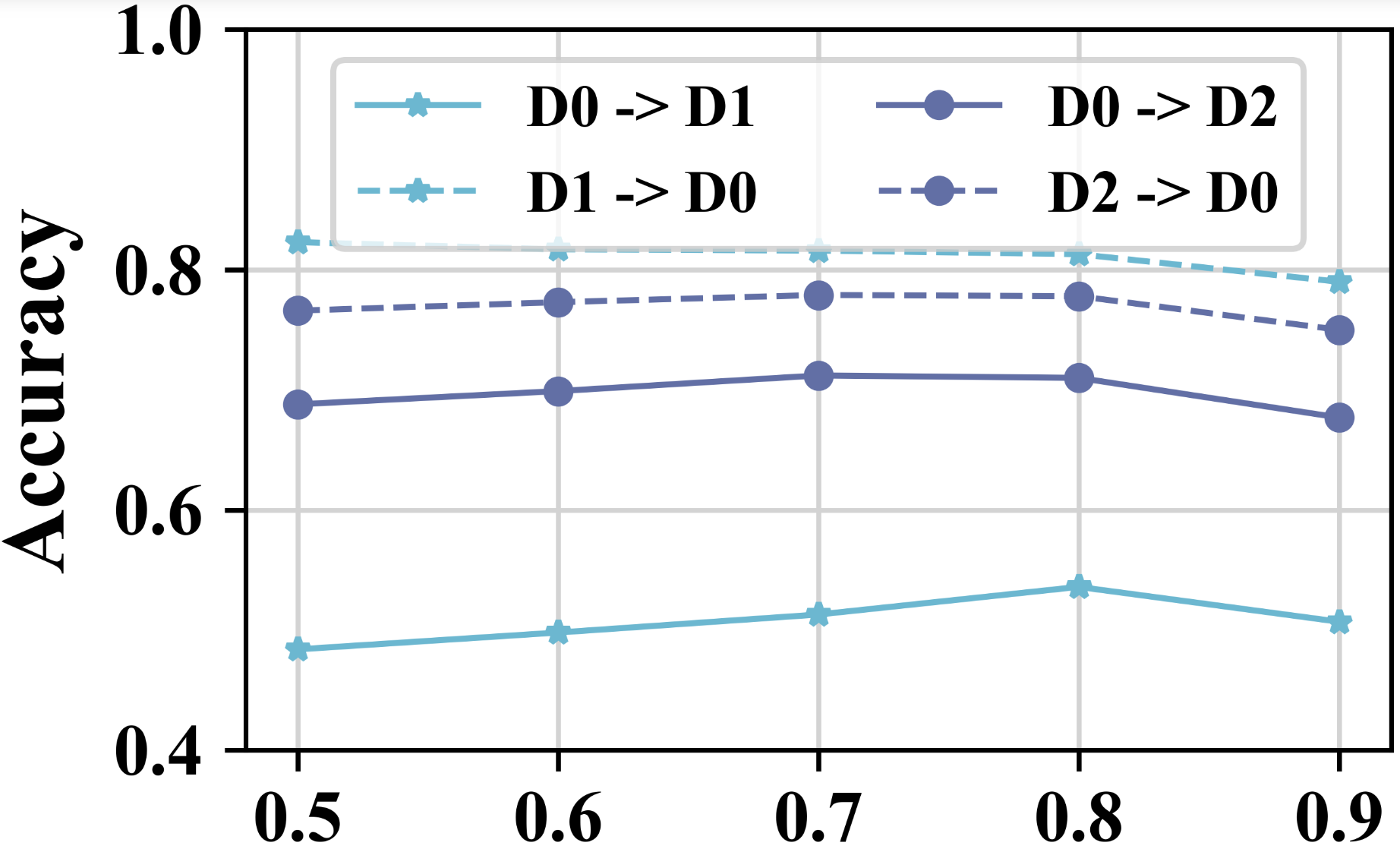}
        \caption{DD}
    \end{subfigure}
    \hfill
    \begin{subfigure}[t]{0.24\linewidth}
        \centering
        \includegraphics[width=\linewidth]{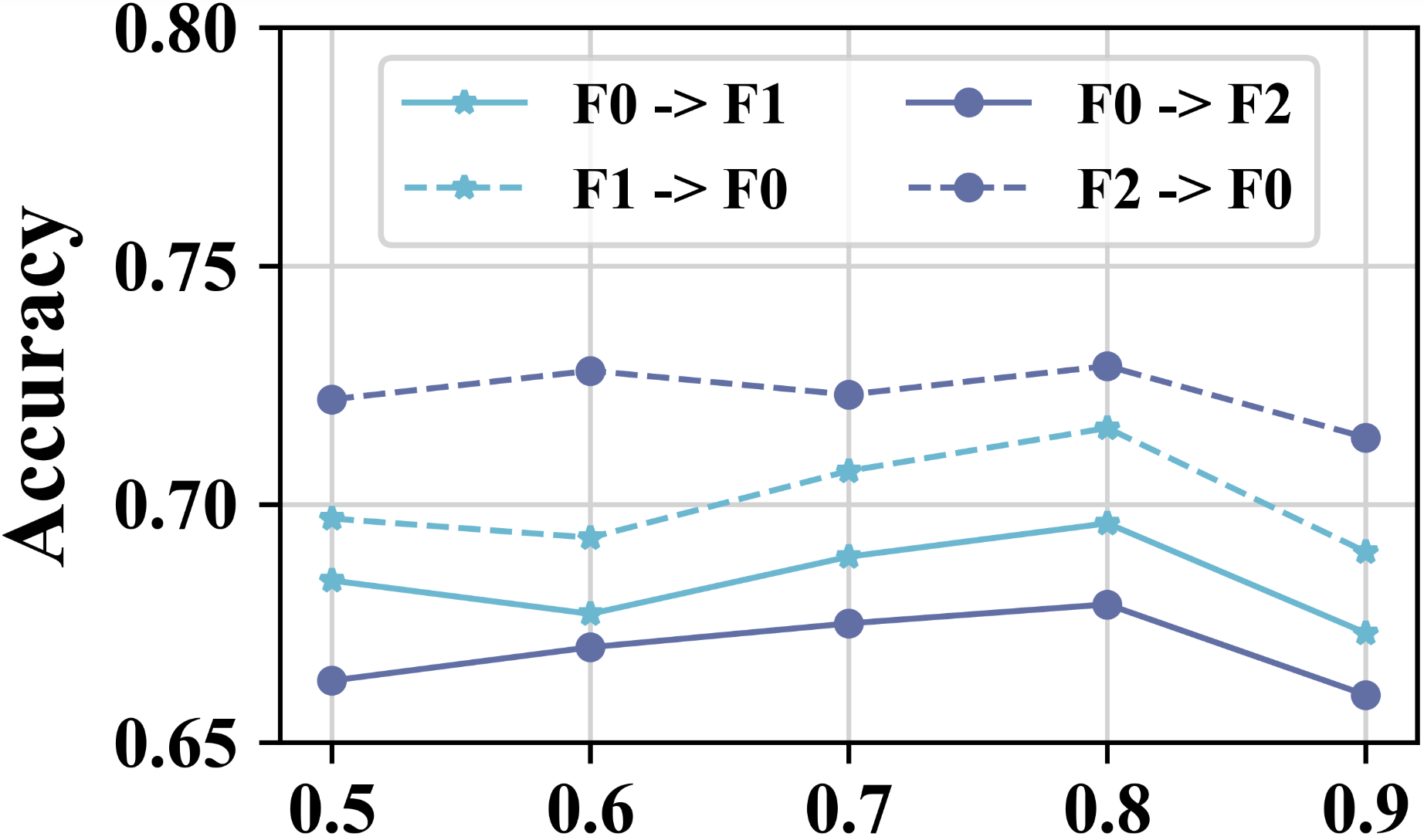}
        \caption{FRANKENSTEIN}
    \end{subfigure}
    \hfill
    \begin{subfigure}[t]{0.24\linewidth}
        \centering
        \includegraphics[width=\linewidth]{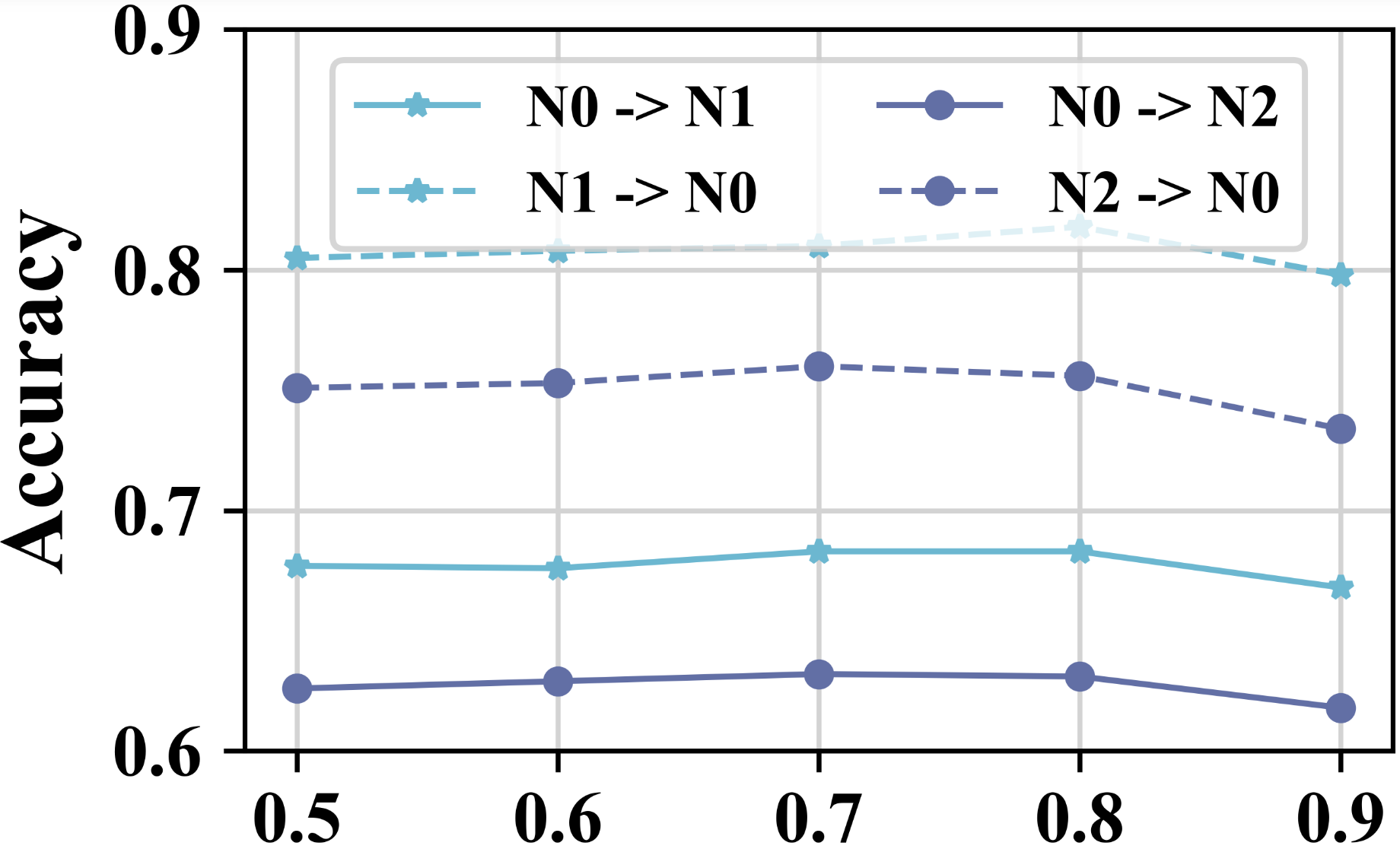}
        \caption{NCI1}
    \end{subfigure}
    \hfill
    \begin{subfigure}[t]{0.24\linewidth}
        \centering
        \includegraphics[width=\linewidth]{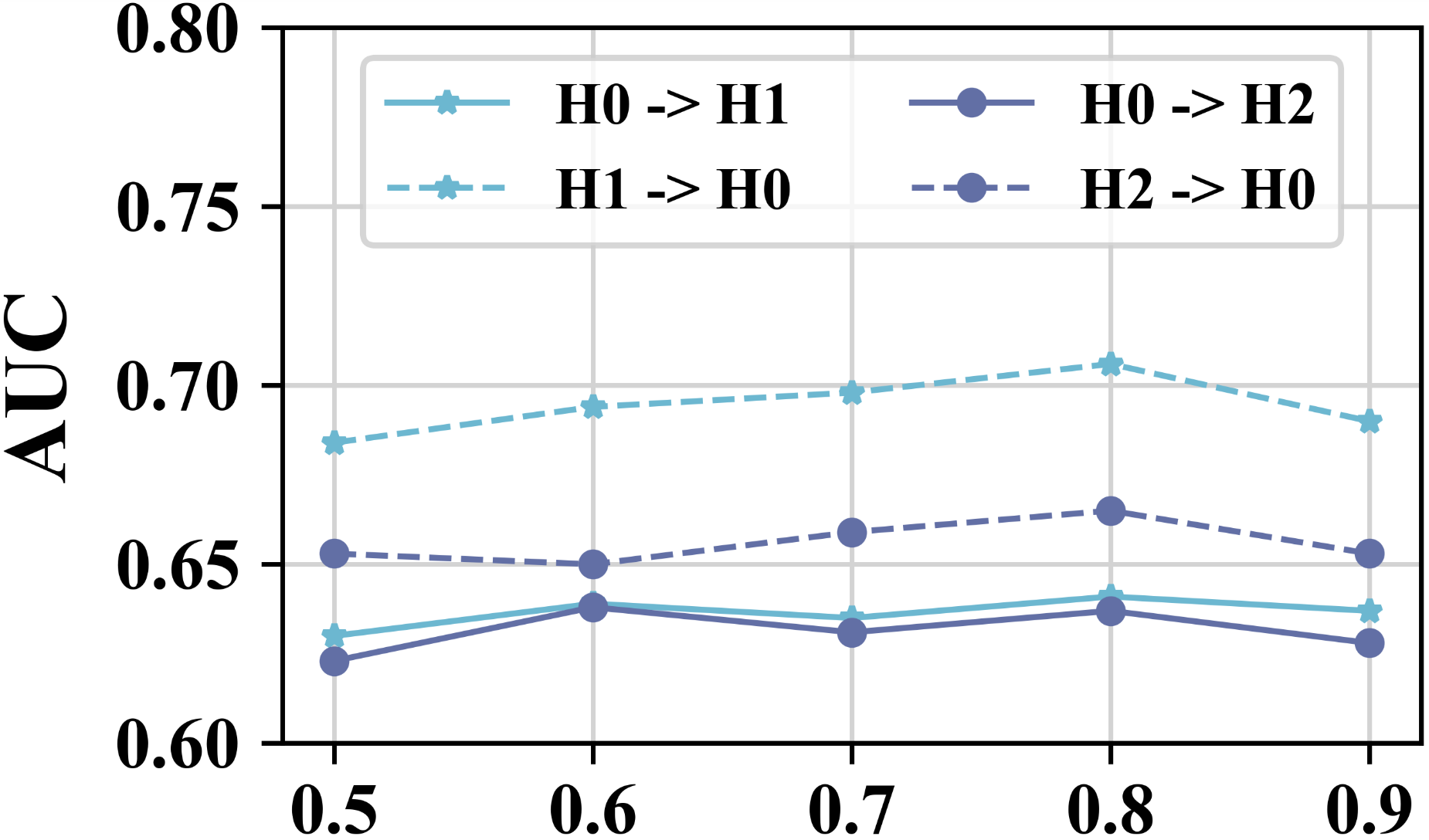}
        \caption{ogbg-molhiv}
    \end{subfigure}

    \vspace{-3pt}

    \caption{Sensitivity analysis of the confidence threshold $\zeta$ on DD, FRANKENSTEIN, NCI1, and ogbg-molhiv.}

    \vspace{-0.1cm}
    \label{fig:zeta_sensitivity}
\end{figure*}

\begin{figure*}[t]
    \centering

    \begin{subfigure}[t]{0.24\linewidth}
        \centering
        \includegraphics[width=\linewidth]{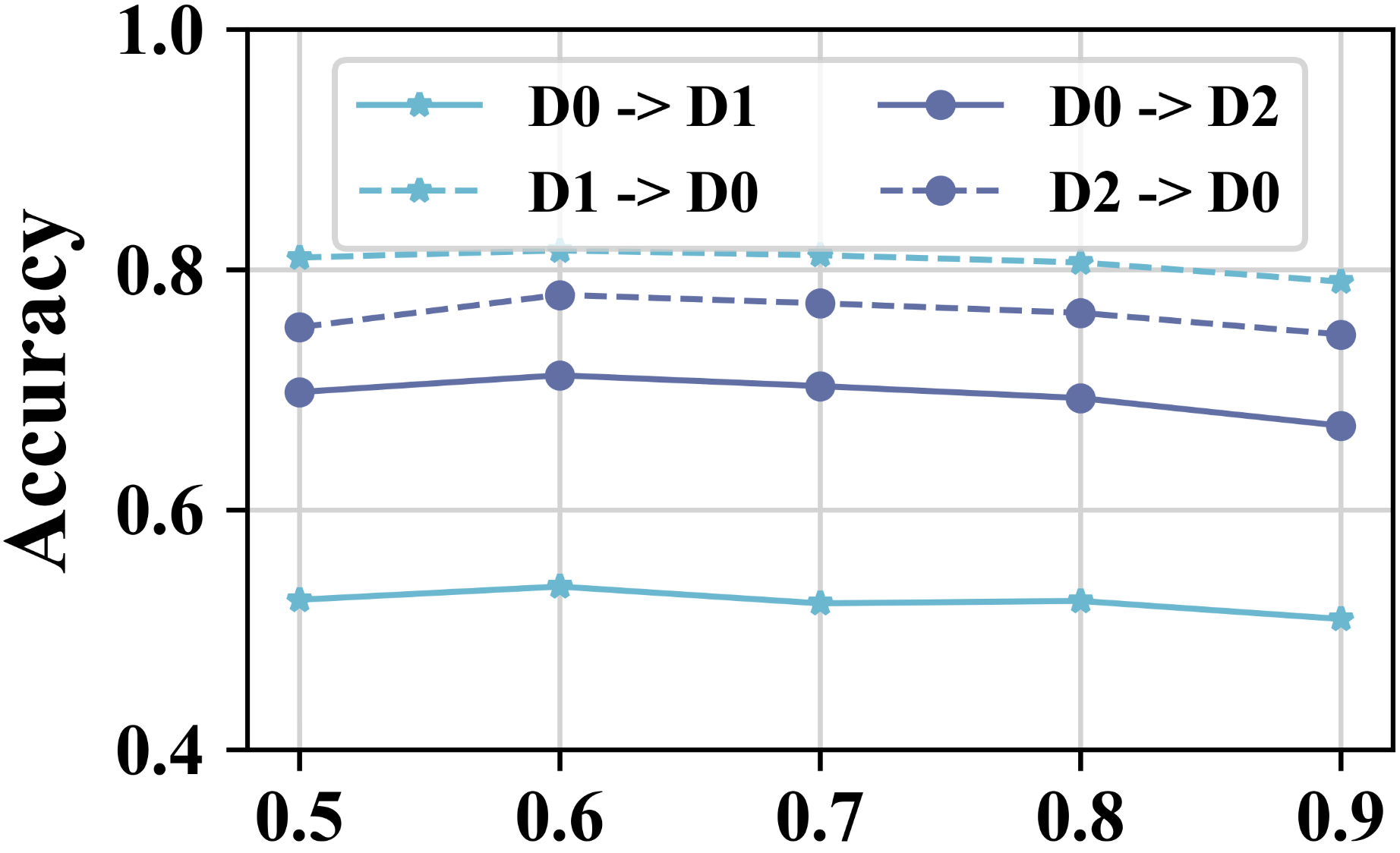}
        \caption{DD}
    \end{subfigure}
    \hfill
    \begin{subfigure}[t]{0.24\linewidth}
        \centering
        \includegraphics[width=\linewidth]{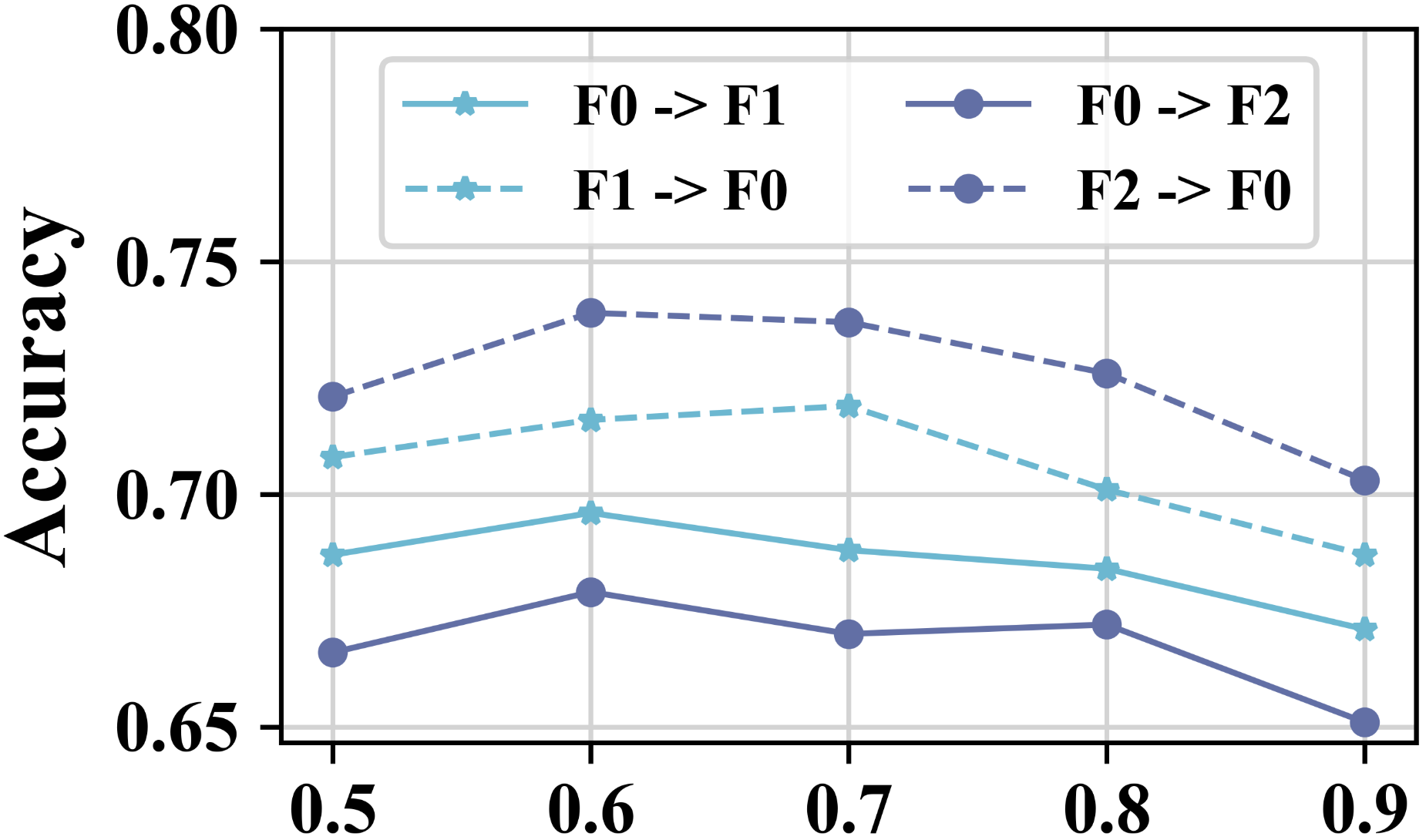}
        \caption{FRANKENSTEIN}
    \end{subfigure}
    \hfill
    \begin{subfigure}[t]{0.24\linewidth}
        \centering
        \includegraphics[width=\linewidth]{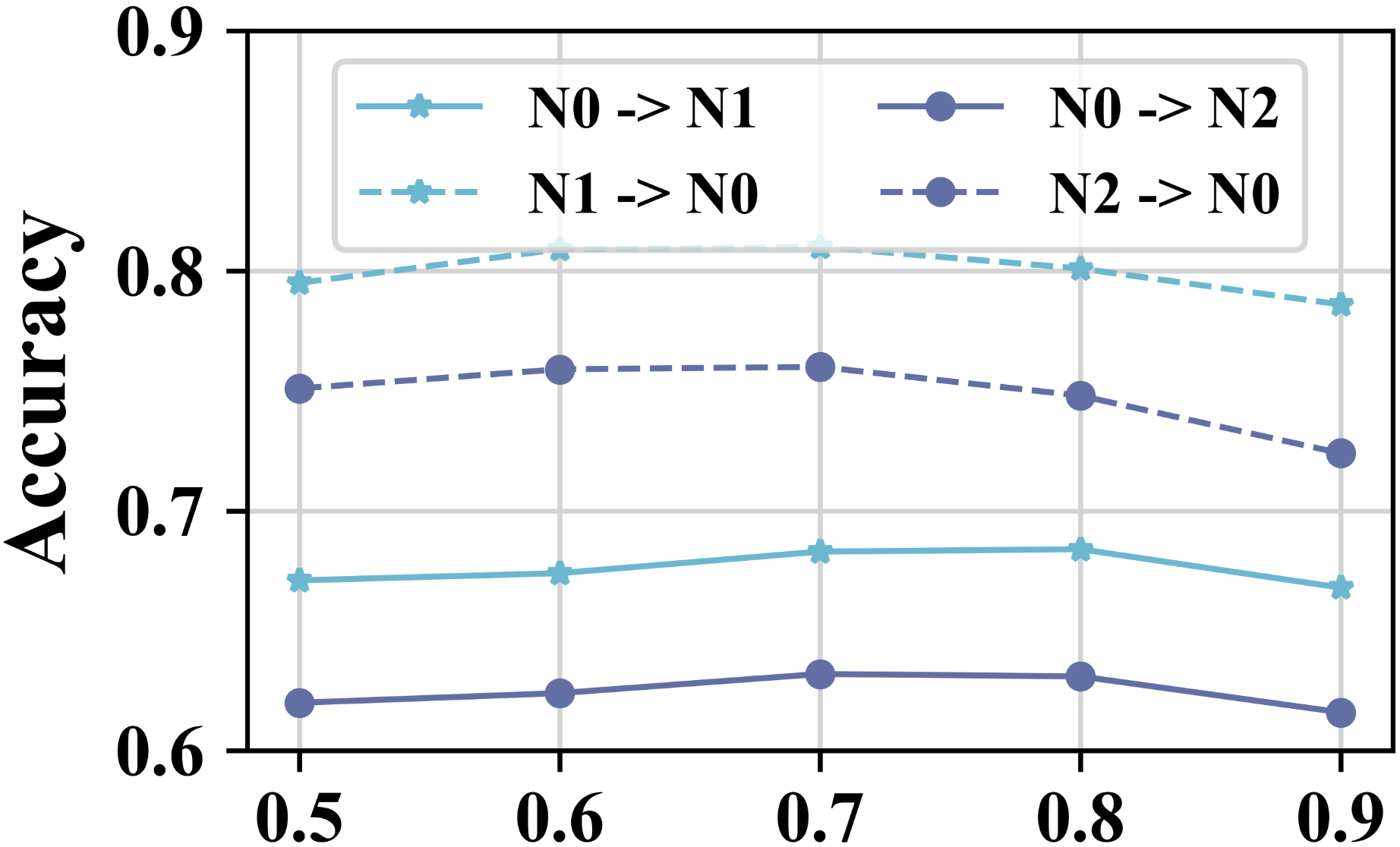}
        \caption{NCI1}
    \end{subfigure}
    \hfill
    \begin{subfigure}[t]{0.24\linewidth}
        \centering
        \includegraphics[width=\linewidth]{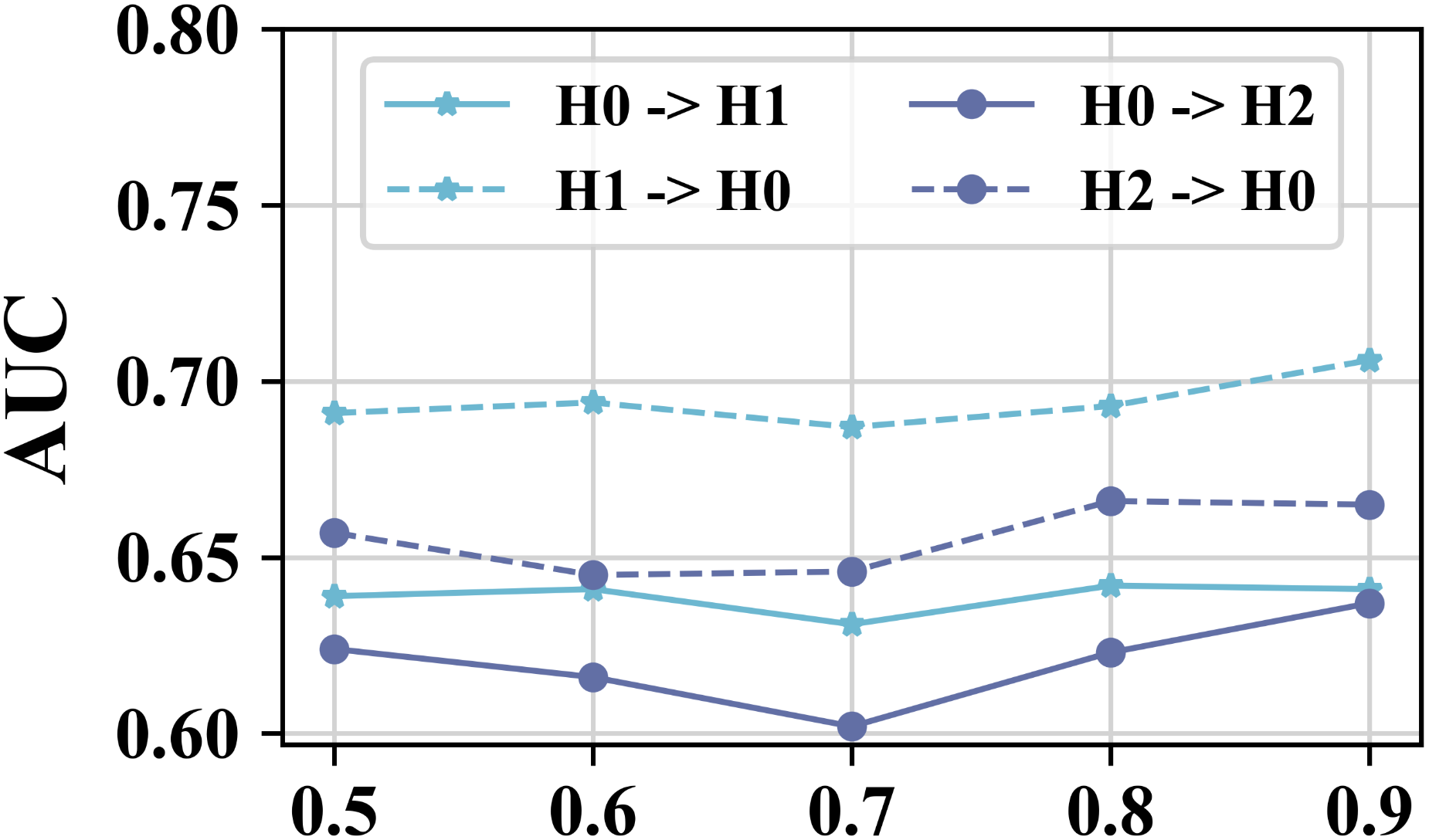}
        \caption{ogbg-molhiv}
    \end{subfigure}

    \vspace{-3pt}

    \caption{Sensitivity analysis of the neighborhood-consistency threshold $\rho_{\min}$ on DD, FRANKENSTEIN, NCI1, and ogbg-molhiv.}

    \vspace{-0.1cm}
    \label{fig:rho_sensitivity}
\end{figure*}

\begin{figure*}[t]
    \centering

    \begin{subfigure}[t]{0.19\linewidth}
        \centering
        \includegraphics[width=\linewidth]{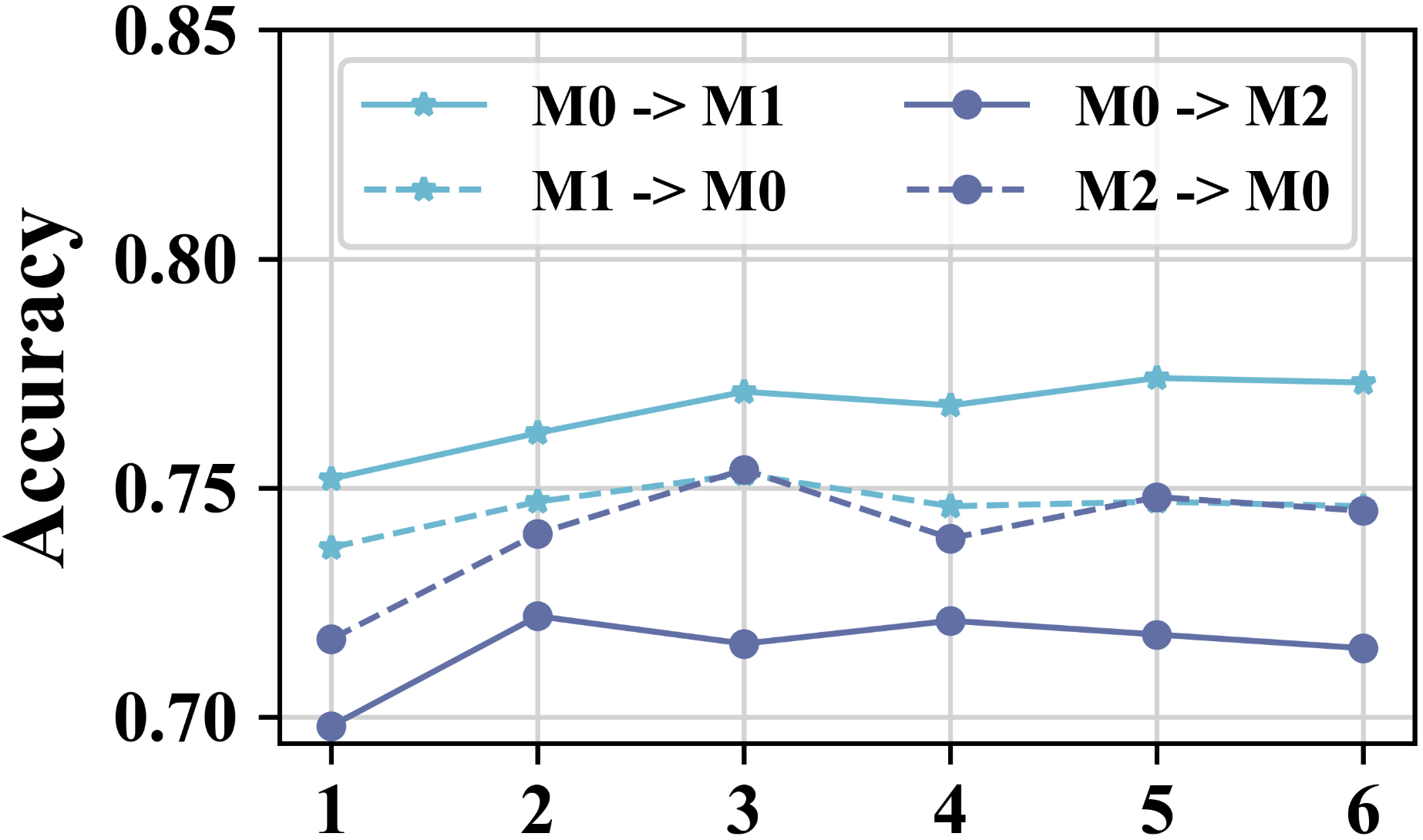}
        \caption{Mutagenicity}
    \end{subfigure}
    \hfill
    \begin{subfigure}[t]{0.19\linewidth}
        \centering
        \includegraphics[width=\linewidth]{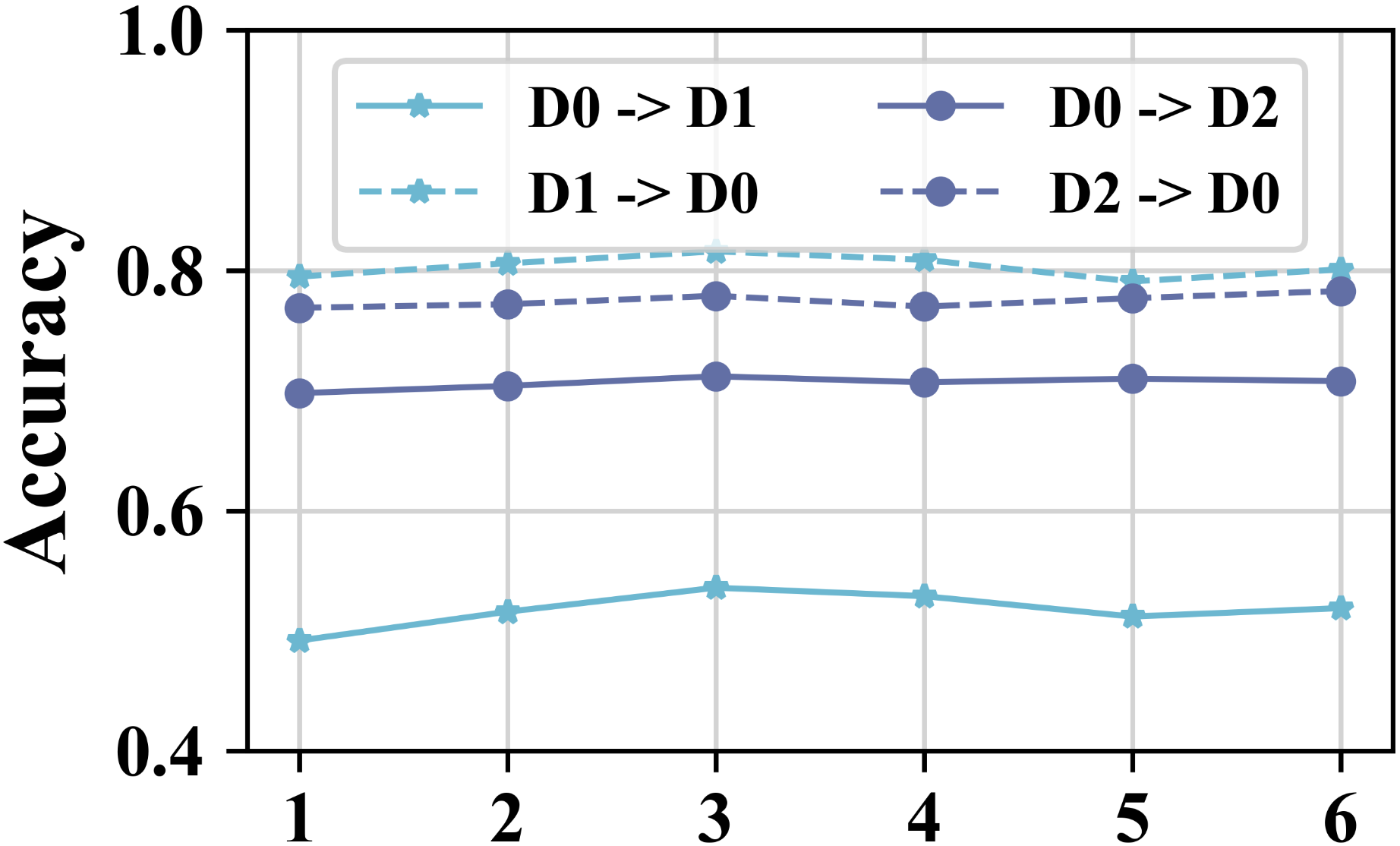}
        \caption{DD}
    \end{subfigure}
    \hfill
    \begin{subfigure}[t]{0.19\linewidth}
        \centering
        \includegraphics[width=\linewidth]{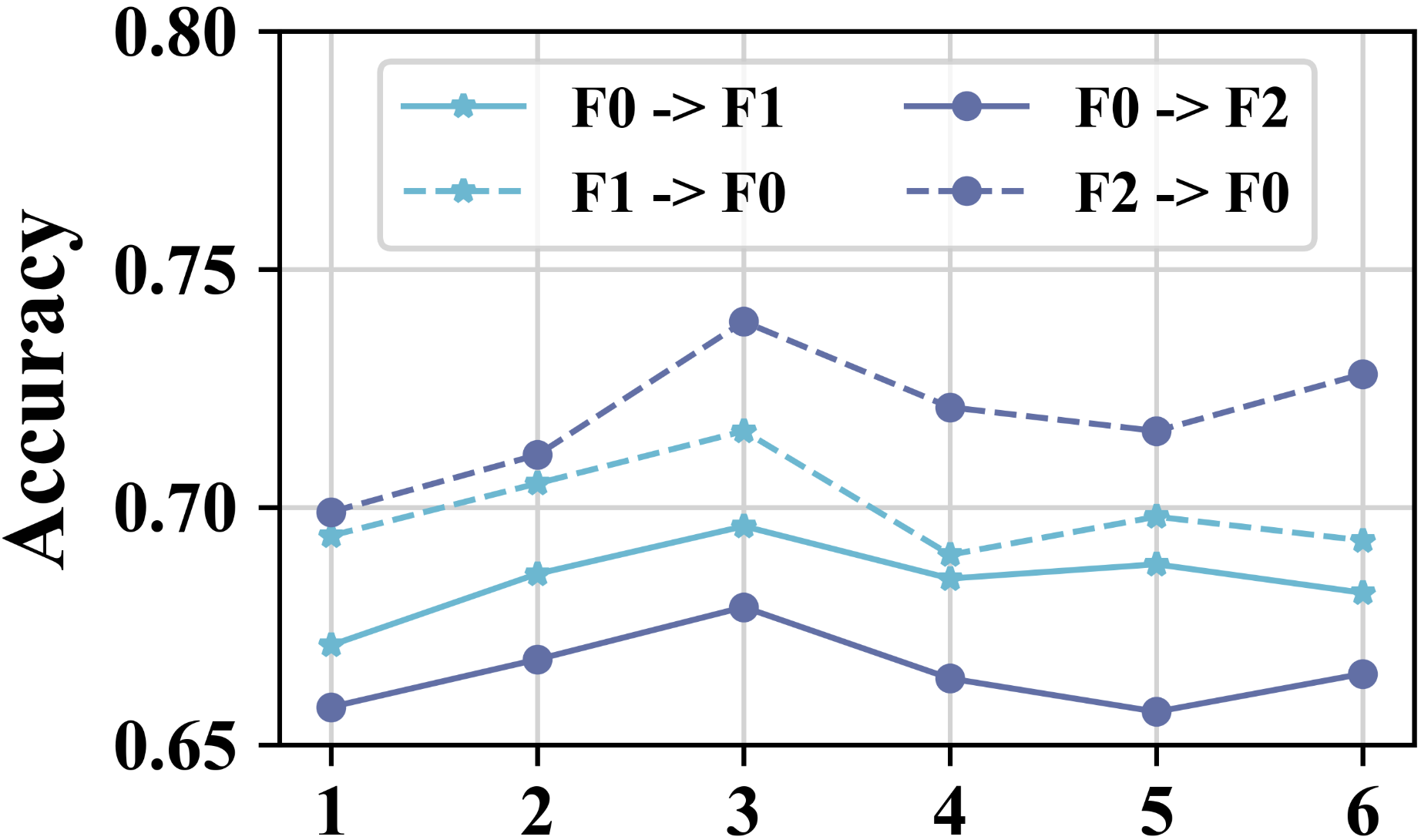}
        \caption{FRANKENSTEIN}
    \end{subfigure}
    \hfill
    \begin{subfigure}[t]{0.19\linewidth}
        \centering
        \includegraphics[width=\linewidth]{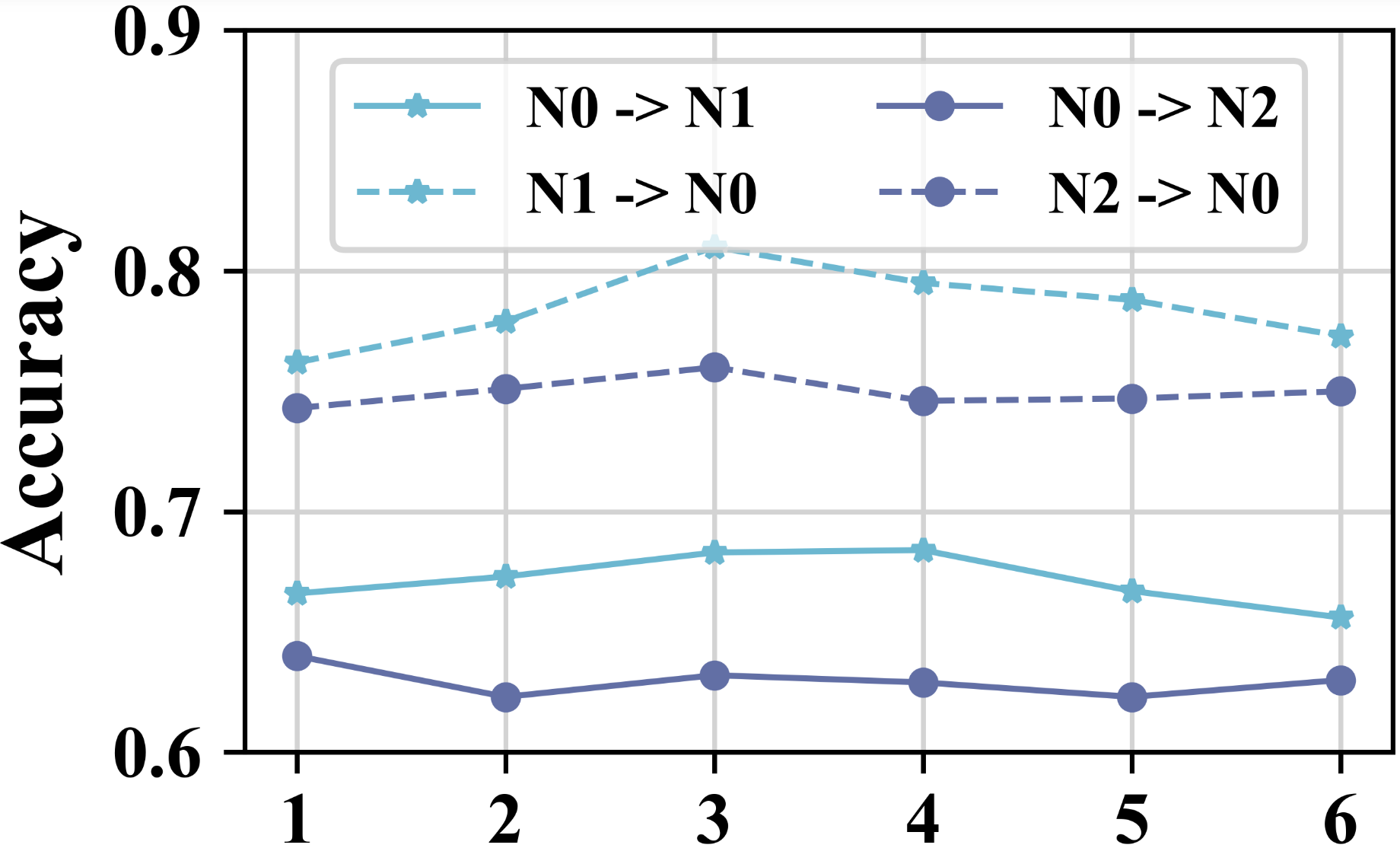}
        \caption{NCI1}
    \end{subfigure}
    \hfill
    \begin{subfigure}[t]{0.19\linewidth}
        \centering
        \includegraphics[width=\linewidth]{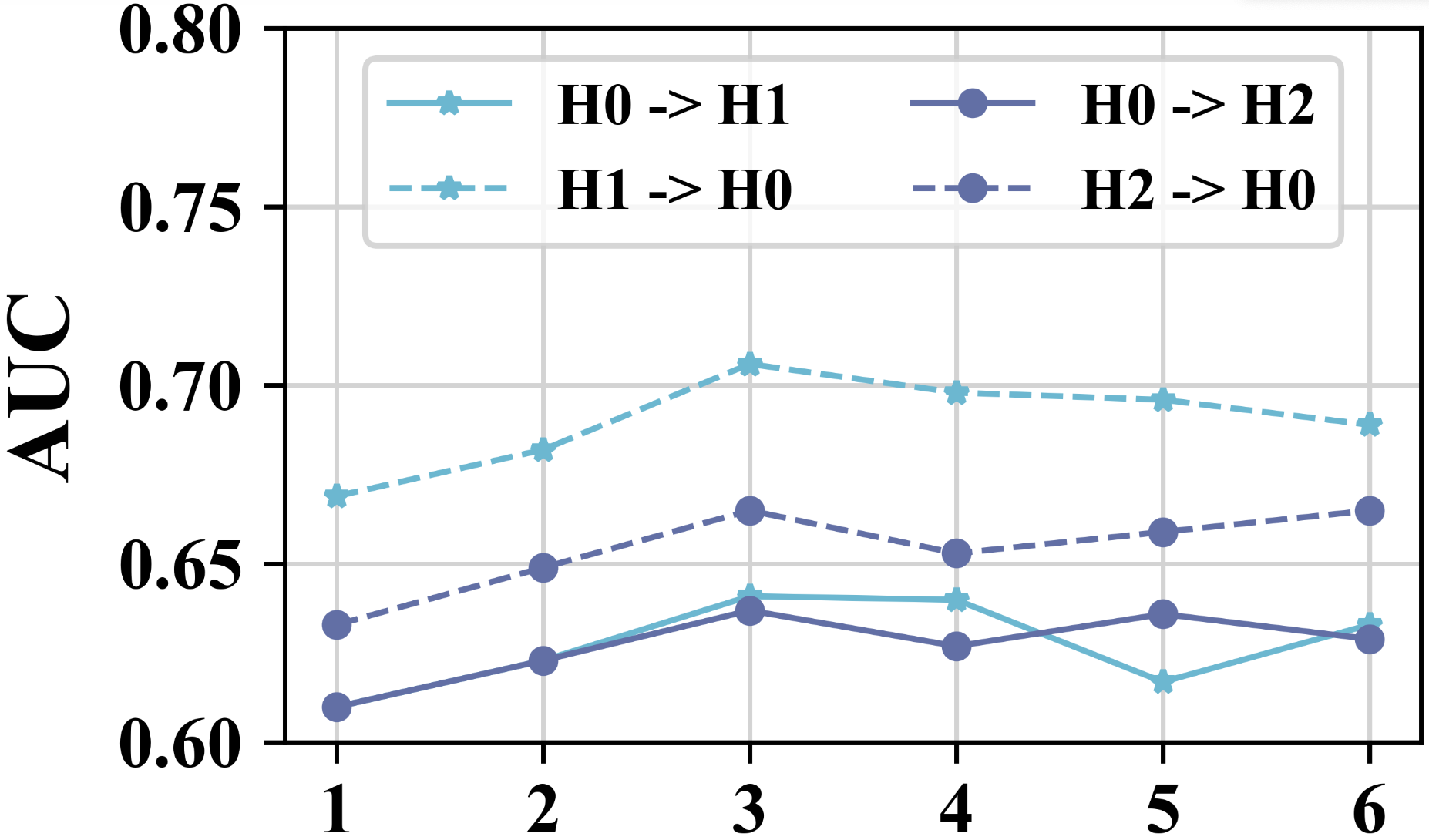}
        \caption{ogbg-molhiv}
    \end{subfigure}

    \caption{Sensitivity analysis of the number of source experts $K_e$ on Mutagenicity, DD, FRANKENSTEIN, NCI1, and ogbg-molhiv.}

    \label{fig:ke_sensitivity}

    \vspace{-0.3cm}
\end{figure*}

\subsubsection{More Sensitivity Analysis}
\label{sec:more_sensitivity}

We further analyze the sensitivity of \method{} to the confidence threshold $\zeta$ and the neighborhood-consistency threshold $\rho_{\min}$ on DD, FRANKENSTEIN, NCI1, and ogbg-molhiv. 
The results are shown in Fig.~\ref{fig:zeta_sensitivity} and~\ref{fig:rho_sensitivity}, and the observed trends are consistent with the analysis in Section~\ref{sec:sensitivity}. 
For each dataset, we vary one threshold at a time over $\{0.5,0.6,0.7,0.8,0.9\}$ while keeping the other hyperparameters fixed.

Overall, \method{} exhibits stable performance across a broad range of threshold values. 
For the confidence threshold $\zeta$, moderate values generally lead to better results because they remove low-confidence pseudo-labels while retaining sufficient safe samples for hard supervision. 
When $\zeta$ becomes too large, the selected safe subspace may become overly conservative, which slightly degrades performance on several transfer directions. 
For the neighborhood-consistency threshold $\rho_{\min}$, a similar quality--coverage trade-off can be observed: small values may admit locally inconsistent samples, whereas overly strict values may discard useful target samples under domain shift. 
Across DD, FRANKENSTEIN, NCI1, and ogbg-molhiv, the best or near-best performance is usually achieved in the middle range, especially around $0.6$--$0.8$. 
These results confirm that \method{} is not overly sensitive to the exact threshold choice and that its semantic and structural filters remain effective across both molecular and protein graph benchmarks.

We additionally examine the sensitivity of \method{} to the number of source experts $K_e$, as shown in Fig.~\ref{fig:ke_sensitivity}. 
We vary $K_e$ from 1 to 6 while keeping the other hyperparameters fixed. 
The source experts are added cumulatively in the following order: GMT, GIN, PathNN, GCN, SAGE, and RandomWalkNN. 
Thus, $K_e=1$ uses only GMT, $K_e=2$ uses GMT and GIN, $K_e=3$ uses GMT, GIN, and PathNN, and larger values further include GCN, SAGE, and RandomWalkNN. 
Different from $\zeta$ and $\rho_{\min}$, which control the strictness of sample selection, $K_e$ affects the reliability of multi-expert uncertainty estimation.

The results show that increasing $K_e$ from a single source hypothesis to a moderate-size heterogeneous committee generally improves performance, indicating that complementary experts provide more stable disagreement signals and help suppress source-biased pseudo-labels. 
However, the gains tend to saturate when $K_e$ becomes larger, and further increasing the number of experts does not consistently improve performance across datasets. 
This suggests that additional experts may introduce redundant or less informative disagreement signals, making safe-subspace selection overly conservative in some transfer directions. 
Overall, the results support the default choice of $K_e=3$, which uses GMT, GIN, and PathNN and achieves a favorable trade-off between reliability estimation, adaptation performance, and computational cost across Mutagenicity, DD, FRANKENSTEIN, NCI1, and ogbg-molhiv.

\begin{table*}[t]
\centering
\small
\caption{Time (in seconds) and GPU memory (in MB) consumption of different methods in the adaptation stage for each epoch.}
\label{tab:computational_cost}
\begin{tabular}{lcccc}
\toprule
Method & Source Hypotheses & Trainable Params & Adapt. Time & GPU Memory \\
\midrule
GraphATA & 1 & 164,738 & 0.9843 & 1,732 \\ 
GALA & 1 & 144,819 & 2.2843 & 5,530 \\
GraphCTA & 1 & 101,256 & 0.7563 & 1,312 \\
SOGA & 1 & 134,274 & 0.1873 & 1,678 \\
\method{}-Single & 1 & 445,828 & 1.7230 & 1,473 \\
\method{}-Homo3 & 3 & 445,828 & 2.0377 & 1,614 \\
\method{}-Hetero3 & 3 & 445,828 & 1.7617 & 5,665 \\
\bottomrule
\end{tabular}
\end{table*}
\subsubsection{Efficiency and Resource Consumption Analysis} 

In this section, we report the computational cost of \method{} and representative baselines, including GALA, GraphCTA, and SOGA, in Table~\ref{tab:computational_cost}. 
The goal is to examine whether the performance gains of \method{} come from a substantially larger adaptation budget or from the proposed reliability-guided refinement mechanism. 
All measurements are conducted on the Mutagenicity dataset under the same hardware setting. 
We report the per-epoch adaptation time, peak GPU memory, number of released source hypotheses, and trainable parameters during target adaptation. 
For \method{}, \method{}-Single uses one frozen GMT source hypothesis, \method{}-Homo3 uses three frozen GMT source hypotheses trained with different random seeds, and \method{}-Hetero3 uses a heterogeneous frozen source committee composed of GMT, GIN, and PathNN. 
During target adaptation, all released source hypotheses are kept frozen and are used only for forward-pass reliability estimation, including pseudo-label prediction and predictive-variance computation. 
Gradients are propagated only through the target-side trainable modules. 
Therefore, increasing the number of source hypotheses affects the number of stored checkpoints and the forward-pass cost, but does not increase the number of trainable parameters optimized during adaptation.

The trainable parameter count of \method{} comes from two target-side modules: the primary target model and the target-intrinsic structure encoder $g_{\psi}$. 
In our implementation, the primary target model is instantiated as GMT and is responsible for target prediction and pseudo-label supervision, while $g_{\psi}$ is instantiated as GIN and is used to learn target-intrinsic structural representations for graph contrastive learning and neighborhood-consistency verification. 
Both modules are updated during target adaptation. 
Thus, the trainable parameters of \method{} correspond to the combination of the trainable GMT target model and the trainable GIN-based structure encoder, together with minor overhead from task-specific heads or projection layers. 
Importantly, this parameter count remains identical for \method{}-Single, \method{}-Homo3, and \method{}-Hetero3, because the additional source hypotheses are frozen and never fine-tuned. 
The moderate runtime increase of the multi-hypothesis variants mainly comes from extra forward passes through frozen experts, while the higher memory usage of \method{}-Hetero3 is due to the heterogeneous frozen encoders used during reliability estimation. 
Overall, the results show that the gains of \method{} are not caused by additional trainable source models or test-time ensembling, but by using frozen source hypotheses to improve pseudo-label reliability, followed by target-structural verification and safe-subspace supervision.

\subsection{Limitations}
\label{sec:limitations}

The proposed method has two limitations. First, the main version assumes that the source provider can release multiple source hypotheses or checkpoints. This is compatible with the source-free setting because no source graph is accessed during adaptation, but it may increase storage and inference cost. Second, the theoretical guarantee is selective: it applies to the safe subspace and does not certify correctness on all target samples. This is why uncertain samples are handled by soft regularization rather than hard pseudo-labels. These limitations are reflected in the design of the single-expert variant and the risk decomposition in Proposition~\ref{prop:target_risk_bound}.

\begin{table*}[t]
\small
\centering
\caption{The results of ablation studies on the DD dataset (source $\rightarrow$ target).}
\vspace{-4pt}
\resizebox{1.0\textwidth}{!}{
\begin{tabular}{
>{\centering\arraybackslash}m{2.3cm}|
*{12}{>{\centering\arraybackslash}m{1.2cm}}
}
\toprule
{\bf Methods} 
& D0$\rightarrow$D1 
& D1$\rightarrow$D0 
& D0$\rightarrow$D2 
& D2$\rightarrow$D0 
& D0$\rightarrow$D3 
& D3$\rightarrow$D0 
& D1$\rightarrow$D2 
& D2$\rightarrow$D1 
& D1$\rightarrow$D3 
& D3$\rightarrow$D1 
& D2$\rightarrow$D3 
& D3$\rightarrow$D2 \\
\midrule

\method{} w/o ME
& 52.4 & 80.0 & 67.5 & 75.5 & 78.4 & 59.8 & 76.8 & 65.3 & 83.0 & 80.7 & 86.0 & 83.5 \\

\method{} w/o CF
& 51.3 & 78.5 & 68.8 & 74.3 & 78.1 & 59.0 & 79.2 & 66.4 & 85.1 & 77.1 & 86.5 & 85.0 \\

\method{} w/o TS
& 52.2 & 79.4 & 67.6 & 73.4 & 80.4 & 58.9 & 78.8 & 66.0 & 86.6 & 79.5 & 86.2 & 83.6 \\

\method{} w/o NC
& 52.4 & 79.9 & 65.4 & 74.6 & 79.2 & 60.4 & 79.0 & 67.7 & 86.0 & 79.1 & 85.9 & 84.1 \\

\method{} w/o SR
& 48.1 & 77.5 & 63.5 & 73.1 & 77.9 & 57.6 & 76.0 & 64.0 & 82.3 & 76.8 & 85.1 & 79.3 \\

\midrule

\method{} & 
\textbf{55.7} & \textbf{80.5} & \textbf{70.1} & \textbf{77.3} & \textbf{82.7} & \textbf{62.8} & \textbf{81.4} & \textbf{68.1} & \textbf{87.3} & \textbf{83.2} & \textbf{88.0} & \textbf{86.7} \\

\bottomrule
\end{tabular}
}
\vspace{-5pt}
\label{tab:ablation_study_dd}
\end{table*}
\begin{table*}[t]
\small
\centering
\caption{The results of ablation studies on the FRANKENSTEIN dataset (source $\rightarrow$ target).}
\vspace{-4pt}

\resizebox{1.0\textwidth}{!}{
\begin{tabular}{
>{\centering\arraybackslash}m{2.3cm}|
*{12}{>{\centering\arraybackslash}m{1.2cm}}
}
\toprule

{\bf Methods}
& F0$\rightarrow$F1
& F1$\rightarrow$F0
& F0$\rightarrow$F2
& F2$\rightarrow$F0
& F0$\rightarrow$F3
& F3$\rightarrow$F0
& F1$\rightarrow$F2
& F2$\rightarrow$F1
& F1$\rightarrow$F3
& F3$\rightarrow$F1
& F2$\rightarrow$F3
& F3$\rightarrow$F2 \\
\midrule

\method{} w/o ME
& 67.1 & 66.1 & 65.2 & 71.7 & 57.2 & 55.3 & 59.1 & 69.6 & 57.1 & 64.5 & 57.7 & 67.3 \\

\method{} w/o CF
& 67.7 & 68.2 & 64.6 & 70.2 & 55.1 & 56.2 & 66.7 & 68.4 & 56.9 & 62.4 & 59.4 & 66.9 \\

\method{} w/o TS
& 67.6 & 64.2 & 62.9 & 71.5 & 56.2 & 54.7 & 66.1 & 68.1 & 57.3 & 64.9 & 58.9 & 68.7 \\

\method{} w/o NC
& 67.8 & 67.5 & 63.7 & 69.9 & 53.1 & 54.0 & 67.5 & 67.5 & 55.2 & 62.2 & 58.1 & 67.1 \\

\method{} w/o SR
& 67.4 & 58.0 & 55.7 & 61.7 & 49.1 & 43.2 & 58.5 & 67.9 & 53.7 & 59.1 & 57.5 & 63.9 \\

\midrule

\method{}
& \textbf{69.6} & \textbf{71.6} & \textbf{67.9} & \textbf{73.9} & \textbf{59.9} & \textbf{57.1} & \textbf{67.6} & \textbf{71.4} & \textbf{61.1} & \textbf{66.2} & \textbf{60.7} & \textbf{69.5} \\

\bottomrule
\end{tabular}
}

\vspace{-5pt}
\label{tab:ablation_study_frank}
\end{table*}
\begin{table*}[t]
\small
\centering
\caption{The results of ablation studies on the NCI1 dataset (source $\rightarrow$ target).}
\vspace{-4pt}

\resizebox{1.0\textwidth}{!}{
\begin{tabular}{
>{\centering\arraybackslash}m{2.3cm}|
*{12}{>{\centering\arraybackslash}m{1.2cm}}
}
\toprule

{\bf Methods}
&N0$\rightarrow$N1 &N1$\rightarrow$N0 &N0$\rightarrow$N2 &N2$\rightarrow$N0 &N0$\rightarrow$N3 &N3$\rightarrow$N0 &N1$\rightarrow$N2 &N2$\rightarrow$N1 &N1$\rightarrow$N3 &N3$\rightarrow$N1 &N2$\rightarrow$N3 &N3$\rightarrow$N2 \\
\midrule

\method{} w/o ME
& 66.8 & 77.1 & 62.0 & 72.6 & 65.9 & 70.9 & 68.8 & 72.3 & 67.9 & 61.4 & 72.9 & 69.6 \\

\method{} w/o CF
& 67.1 & 78.6 & 62.3 & 74.6 & 65.2 & 69.1 & 71.1 & 75.5 & 68.0 & 62.1 & 76.2 & 72.7 \\

\method{} w/o TS
& 66.9 & 79.9 & 61.4 & 73.6 & 65.5 & 67.5 & 70.0 & 74.8 & 67.8 & 62.3 & 75.8 & 71.7 \\

\method{} w/o NC
& 66.0 & 78.3 & 61.8 & 72.0 & 66.8 & 68.6 & 70.3 & 75.3 & 66.9 & 61.6 & 75.5 & 71.5 \\

\method{} w/o SR
& 59.6 & 76.5 & 60.0 & 70.9 & 65.3 & 63.8 & 69.3 & 71.4 & 67.5 & 54.1 & 70.1 & 68.3 \\

\midrule

\method{}
& \textbf{68.3} & \textbf{81.0} & \textbf{63.2} & \textbf{76.0} & \textbf{67.2} & \textbf{71.7} & \textbf{72.2} & \textbf{77.5} & \textbf{69.1} & \textbf{63.6} & \textbf{77.3} & \textbf{73.2}\\

\bottomrule
\end{tabular}
}

\vspace{-5pt}
\label{tab:ablation_study_nci1}
\end{table*}
\begin{table*}[t]
\small
\centering
\caption{The results of ablation studies on the ogbg-molhiv dataset (source $\rightarrow$ target).}
\vspace{-4pt}

\resizebox{1.0\textwidth}{!}{
\begin{tabular}{
>{\centering\arraybackslash}m{2.3cm}|
*{12}{>{\centering\arraybackslash}m{1.2cm}}
}
\toprule

{\bf Methods}
& H0$\rightarrow$H1 & H1$\rightarrow$H0 & H0$\rightarrow$H2 & H2$\rightarrow$H0 & H0$\rightarrow$H3 & H3$\rightarrow$H0 & H1$\rightarrow$H2 & H2$\rightarrow$H1 & H1$\rightarrow$H3 & H3$\rightarrow$H1 & H2$\rightarrow$H3 & H3$\rightarrow$H2 \\
\midrule

\method{} w/o ME
& 60.3 & 68.5 & 59.0 & 63.5 & 62.7 & 62.3 & 68.6 & 73.7 & 67.7 & 69.8 & 65.0 & 70.7 \\

\method{} w/o CF
& 62.2 & 68.8 & 60.2 & 64.9 & 62.8 & 65.5 & 69.3 & 72.8 & 67.5 & 70.8 & 64.4 & 73.5 \\

\method{} w/o TS
& 61.6 & 67.9 & 61.0 & 64.6 & 61.1 & 64.8 & 67.0 & 71.2 & 65.0 & 70.5 & 63.8 & 72.4 \\

\method{} w/o NC
& 62.0 & 68.8 & 60.3 & 63.9 & 62.8 & 63.5 & 68.4 & 72.6 & 64.7 & 71.0 & 64.9 & 71.9 \\

\method{} w/o SR
& 58.0 & 58.6 & 60.0 & 60.2 & 59.7 & 59.9 & 64.5 & 63.1 & 62.7 & 66.3 & 57.3 & 59.6 \\

\midrule

\method{}
& \textbf{64.1} & \textbf{70.6} & \textbf{63.7} & \textbf{66.5} & \textbf{64.6} & \textbf{66.1} & \textbf{70.2} & \textbf{74.4} & \textbf{70.8} & \textbf{73.4} & \textbf{68.0} & \textbf{76.2} \\

\bottomrule
\end{tabular}
}

\label{tab:ablation_study_molhiv}
\end{table*}

\begin{table*}[t]
\small
	\caption{The classification results (in \%) on the DD dataset under node density domain shift (source $\rightarrow$ target).  D0, D1, D2, and D3 denote the sub-datasets partitioned with node density. \textbf{Bold} results indicate the best performance.}\resizebox{1.0\textwidth}{!}{
	\begin{tabular}{
>{\centering\arraybackslash}m{1.6cm}|
*{12}{>{\centering\arraybackslash}m{1.1cm}}
}
		\toprule                                                      
		\textbf{Methods}    & D0$\rightarrow$D1 & D1$\rightarrow$D0 & D0$\rightarrow$D2 & D2$\rightarrow$D0 & D0$\rightarrow$D3 & D3$\rightarrow$D0 & D1$\rightarrow$D2 & D2$\rightarrow$D1 & D1$\rightarrow$D3 & D3$\rightarrow$D1 & D2$\rightarrow$D3 & D3$\rightarrow$D2
 \\ \midrule
WL subtree  & 49.2 & 56.8 & 29.6 & 20.1 & 21.0 & 18.4 & 59.5 & 50.5 & 57.3 & 48.1 & 63.9 & 66.9 \\
GCN & 48.9\scriptsize{$\pm$2.8} & 59.0\scriptsize{$\pm$1.7} & 20.7\scriptsize{$\pm$2.0} & 27.3\scriptsize{$\pm$2.3} & 15.1\scriptsize{$\pm$1.8} & 26.9\scriptsize{$\pm$2.2} & 61.6\scriptsize{$\pm$1.9} & 53.6\scriptsize{$\pm$1.5} & 68.1\scriptsize{$\pm$1.6} & 52.9\scriptsize{$\pm$2.6} & 64.9\scriptsize{$\pm$2.1} & 69.7\scriptsize{$\pm$2.3} \\
GIN & 48.8\scriptsize{$\pm$1.9} & 24.7\scriptsize{$\pm$2.1} & 44.1\scriptsize{$\pm$1.8} & 22.4\scriptsize{$\pm$2.3} & 57.0\scriptsize{$\pm$2.1} & 18.4\scriptsize{$\pm$2.0} & 73.0\scriptsize{$\pm$1.8} & 52.5\scriptsize{$\pm$2.3} & 63.2\scriptsize{$\pm$1.6} & 53.6\scriptsize{$\pm$1.5} & 70.3\scriptsize{$\pm$2.6} & 69.4\scriptsize{$\pm$1.8} \\
GMT & 49.1\scriptsize{$\pm$1.9} & 32.9\scriptsize{$\pm$2.2} & 31.8\scriptsize{$\pm$1.8} & 27.3\scriptsize{$\pm$2.3} & 52.5\scriptsize{$\pm$1.5} & 27.6\scriptsize{$\pm$1.8} & 75.4\scriptsize{$\pm$1.9} & 53.2\scriptsize{$\pm$2.1} & 74.1\scriptsize{$\pm$2.6} & 57.9\scriptsize{$\pm$2.4} & 70.9\scriptsize{$\pm$1.8} & 71.1\scriptsize{$\pm$2.7} \\
CIN & 50.4\scriptsize{$\pm$1.8} & 18.4\scriptsize{$\pm$2.0} & 21.2\scriptsize{$\pm$2.1} & 36.8\scriptsize{$\pm$1.8} & 43.0\scriptsize{$\pm$2.1} & 22.9\scriptsize{$\pm$1.9} & 53.4\scriptsize{$\pm$1.7} & 56.5\scriptsize{$\pm$1.5} & 62.3\scriptsize{$\pm$1.6} & 53.3\scriptsize{$\pm$1.9} & 75.0\scriptsize{$\pm$2.1} & 69.3\scriptsize{$\pm$2.0} \\
\bottomrule
	\end{tabular}}
    \label{tab:dd_node}
\end{table*}

\begin{table*}[t]
\small
	\caption{The classification results (in \%) on the FRANKENSTEIN dataset under node density domain shift (source $\rightarrow$ target).  F0, F1, F2, and F3 denote the sub-datasets partitioned with node density. \textbf{Bold} results indicate the best performance.}\resizebox{1.0\textwidth}{!}{
	\begin{tabular}{
>{\centering\arraybackslash}m{1.6cm}|
*{12}{>{\centering\arraybackslash}m{1.1cm}}
}
		\toprule                                                      
		\textbf{Methods}    & F0$\rightarrow$F1 &F1$\rightarrow$F0 &F0$\rightarrow$F2 &F2$\rightarrow$F0 &F0$\rightarrow$F3 &F3$\rightarrow$F0 &F1$\rightarrow$F2 &F2$\rightarrow$F1 &F1$\rightarrow$F3 &F3$\rightarrow$F1 &F2$\rightarrow$F3 &F3$\rightarrow$F2
 \\ \midrule
WL subtree  & 54.6 & 40.1 & 47.8 & 62.0 & 49.6 & 56.0 & 49.7 & 51.2 & 41.8 & 42.9 & 47.3 & 45.1\\
GCN & 32.7\scriptsize{$\pm$2.0} & 43.2\scriptsize{$\pm$2.1} & 42.2\scriptsize{$\pm$1.7} & 43.2\scriptsize{$\pm$1.5} & 54.5\scriptsize{$\pm$1.8} & 43.2\scriptsize{$\pm$1.9} & 42.2\scriptsize{$\pm$2.0} & 32.7\scriptsize{$\pm$1.5} & 56.5\scriptsize{$\pm$1.2} & 32.7\scriptsize{$\pm$1.8} & 55.5\scriptsize{$\pm$1.8} & 42.2\scriptsize{$\pm$2.0}\\
GIN & 39.7\scriptsize{$\pm$1.8} & 51.4\scriptsize{$\pm$1.7} & 54.7\scriptsize{$\pm$1.3} & 51.4\scriptsize{$\pm$1.7} & 58.2\scriptsize{$\pm$1.9} & 41.8\scriptsize{$\pm$2.0} & 54.7\scriptsize{$\pm$2.3} & 47.1\scriptsize{$\pm$1.6} & 52.1\scriptsize{$\pm$1.2} & 34.3\scriptsize{$\pm$1.5} & 50.8\scriptsize{$\pm$1.7} & 43.8\scriptsize{$\pm$2.0}\\
GMT & 47.3\scriptsize{$\pm$2.2} & 56.8\scriptsize{$\pm$1.9} & 57.8\scriptsize{$\pm$1.4} & 56.8\scriptsize{$\pm$1.2} & 39.5\scriptsize{$\pm$2.3} & 44.1\scriptsize{$\pm$1.2} & 57.8\scriptsize{$\pm$2.0} & 58.3\scriptsize{$\pm$1.5} & 39.5\scriptsize{$\pm$1.6} & 39.3\scriptsize{$\pm$1.4} & 39.5\scriptsize{$\pm$2.0} & 44.1\scriptsize{$\pm$1.6} \\
CIN & 54.3\scriptsize{$\pm$1.8} & 51.6\scriptsize{$\pm$1.9} & 45.0\scriptsize{$\pm$1.7} & 52.5\scriptsize{$\pm$1.8} & 51.9\scriptsize{$\pm$2.1} & 43.2\scriptsize{$\pm$2.0} & 45.3\scriptsize{$\pm$1.3} & 60.4\scriptsize{$\pm$1.8} & 48.8\scriptsize{$\pm$1.6} & 32.7\scriptsize{$\pm$2.0} & 57.1\scriptsize{$\pm$1.2} & 42.2\scriptsize{$\pm$1.8}\\
\bottomrule
	\end{tabular}}
    \label{tab:frank_node}
\end{table*}
\begin{table*}[t]
\small
	\caption{The classification results (in \%) on the Mutagenicity dataset under node density domain shift (source $\rightarrow$ target).  M0, M1, M2, and M3 denote the sub-datasets partitioned with node density. \textbf{Bold} results indicate the best performance.}\resizebox{1.0\textwidth}{!}{
	\begin{tabular}{
>{\centering\arraybackslash}m{1.6cm}|
*{12}{>{\centering\arraybackslash}m{1.1cm}}
}
		\toprule                                                      
		\textbf{Methods}    &M0$\rightarrow$M1 &M1$\rightarrow$M0 &M0$\rightarrow$M2 &M2$\rightarrow$M0 &M0$\rightarrow$M3 &M3$\rightarrow$M0 &M1$\rightarrow$M2 &M2$\rightarrow$M1 &M1$\rightarrow$M3 &M3$\rightarrow$M1 &M2$\rightarrow$M3 &M3$\rightarrow$M2
 \\ \midrule
WL subtree  & 34.3 & 46.0 & 40.4 & 50.4 & 52.7 & 52.1 & 51.6 & 48.3 & 45.6 & 40.8 & 53.5 & 52.4\\
GCN & 64.1\scriptsize{$\pm$1.4} & 62.6\scriptsize{$\pm$2.3} & 65.5\scriptsize{$\pm$2.0} & 62.3\scriptsize{$\pm$1.1} & 56.9\scriptsize{$\pm$2.1} & 53.7\scriptsize{$\pm$1.8} & 66.1\scriptsize{$\pm$1.6} & 70.8\scriptsize{$\pm$2.1} & 47.7\scriptsize{$\pm$1.9} & 49.2\scriptsize{$\pm$2.1} & 52.1\scriptsize{$\pm$1.7} & 54.4\scriptsize{$\pm$2.1}\\
GIN & 66.5\scriptsize{$\pm$2.1} & 53.2\scriptsize{$\pm$2.3} & 52.0\scriptsize{$\pm$1.7} & 62.6\scriptsize{$\pm$2.3} & 53.7\scriptsize{$\pm$1.7} & 48.6\scriptsize{$\pm$2.3} & 71.5\scriptsize{$\pm$2.5} & 70.0\scriptsize{$\pm$2.1} & 52.2\scriptsize{$\pm$2.3} & 64.3\scriptsize{$\pm$2.1} & 50.9\scriptsize{$\pm$1.4} & 52.4\scriptsize{$\pm$1.5} \\
GMT & 65.7\scriptsize{$\pm$1.8} & 67.1\scriptsize{$\pm$1.9} & 62.1\scriptsize{$\pm$2.1} & 66.4\scriptsize{$\pm$1.2} & 59.0\scriptsize{$\pm$2.0} & 58.0\scriptsize{$\pm$2.7} & 69.9\scriptsize{$\pm$1.4} & 68.4\scriptsize{$\pm$2.3} & 60.4\scriptsize{$\pm$1.3} & 48.2\scriptsize{$\pm$2.0} & 59.1\scriptsize{$\pm$1.3} & 60.6\scriptsize{$\pm$1.7} \\
CIN & 65.1\scriptsize{$\pm$1.7} & 52.9\scriptsize{$\pm$2.3} & 66.0\scriptsize{$\pm$1.7} & 69.8\scriptsize{$\pm$2.0} & 55.2\scriptsize{$\pm$1.5} & 55.1\scriptsize{$\pm$1.3} & 66.0\scriptsize{$\pm$1.6} & 69.9\scriptsize{$\pm$2.1} & 54.9\scriptsize{$\pm$1.8} & 57.8\scriptsize{$\pm$1.5} & 46.7\scriptsize{$\pm$2.0} & 51.8\scriptsize{$\pm$1.1}\\
\bottomrule
	\end{tabular}}
    \vspace{-0.2cm}
    \label{tab:mutag_node}
\end{table*}
\begin{table*}[t]
\small
	\caption{The classification results (in \%) on the NCI1 dataset under node density domain shift (source $\rightarrow$ target).  N0, N1, N2, and N3 denote the sub-datasets partitioned with node density. \textbf{Bold} results indicate the best performance.}\resizebox{1.0\textwidth}{!}{
	\begin{tabular}{
>{\centering\arraybackslash}m{1.6cm}|
*{12}{>{\centering\arraybackslash}m{1.1cm}}
}
		\toprule                                                      
		\textbf{Methods}    &N0$\rightarrow$N1 &N1$\rightarrow$N0 &N0$\rightarrow$N2 &N2$\rightarrow$N0 &N0$\rightarrow$N3 &N3$\rightarrow$N0 &N1$\rightarrow$N2 &N2$\rightarrow$N1 &N1$\rightarrow$N3 &N3$\rightarrow$N1 &N2$\rightarrow$N3 &N3$\rightarrow$N2
 \\ \midrule
WL subtree  & 54.9 & 60.0 & 51.4 & 51.4 & 44.4 & 63.1 & 51.9 & 53.4 & 58.4 & 57.4 & 61.2 & 50.9\\
GCN & 54.0\scriptsize{$\pm$2.0} & 66.5\scriptsize{$\pm$1.9} & 47.3\scriptsize{$\pm$2.2} & 37.4\scriptsize{$\pm$1.8} & 38.4\scriptsize{$\pm$2.3} & 27.6\scriptsize{$\pm$1.5} & 59.0\scriptsize{$\pm$1.7} & 56.9\scriptsize{$\pm$2.0} & 56.5\scriptsize{$\pm$1.2} & 51.0\scriptsize{$\pm$1.7} & 60.5\scriptsize{$\pm$1.9} & 57.4\scriptsize{$\pm$1.6}\\
GIN & 57.0\scriptsize{$\pm$1.9} & 38.7\scriptsize{$\pm$2.3} & 45.7\scriptsize{$\pm$1.8} & 34.4\scriptsize{$\pm$1.9} & 43.6\scriptsize{$\pm$2.1} & 27.0\scriptsize{$\pm$2.0} & 54.7\scriptsize{$\pm$1.7} & 52.8\scriptsize{$\pm$2.1} & 59.0\scriptsize{$\pm$1.8} & 49.1\scriptsize{$\pm$2.5} & 59.4\scriptsize{$\pm$2.0} & 57.4\scriptsize{$\pm$1.3}\\
GMT & 53.0\scriptsize{$\pm$1.9} & 63.7\scriptsize{$\pm$1.6} & 48.9\scriptsize{$\pm$2.2} & 32.8\scriptsize{$\pm$2.0} & 41.5\scriptsize{$\pm$1.7} & 27.8\scriptsize{$\pm$1.6} & 56.4\scriptsize{$\pm$1.9} & 57.2\scriptsize{$\pm$1.7} & 54.1\scriptsize{$\pm$1.9} & 49.8\scriptsize{$\pm$1.7} & 61.6\scriptsize{$\pm$2.1} & 58.5\scriptsize{$\pm$1.1}\\
CIN & 54.5\scriptsize{$\pm$1.9} & 37.6\scriptsize{$\pm$1.9} & 53.2\scriptsize{$\pm$1.8} & 36.2\scriptsize{$\pm$1.4} & 42.1\scriptsize{$\pm$1.8} & 27.5\scriptsize{$\pm$2.0} & 57.2\scriptsize{$\pm$1.6} & 51.6\scriptsize{$\pm$1.9} & 59.3\scriptsize{$\pm$1.5} & 50.0\scriptsize{$\pm$1.4} & 62.0\scriptsize{$\pm$1.4} & 57.4\scriptsize{$\pm$2.3}\\
\bottomrule
	\end{tabular}}
    \label{tab:nci1_node}
\end{table*}
\begin{table*}[t]
\small
	\caption{ROC-AUC results (in \%) on the ogbg-molhiv dataset under node density domain shift (source $\rightarrow$ target).  H0, H1, H2, and H3 denote the sub-datasets partitioned with node density. \textbf{Bold} results indicate the best performance.}\resizebox{1.0\textwidth}{!}{
	\begin{tabular}{
>{\centering\arraybackslash}m{1.6cm}|
*{12}{>{\centering\arraybackslash}m{1.1cm}}
}
		\toprule                                                      
		\textbf{Methods}    & H0$\rightarrow$H1 & H1$\rightarrow$H0 & H0$\rightarrow$H2 & H2$\rightarrow$H0 & H0$\rightarrow$H3 & H3$\rightarrow$H0 & H1$\rightarrow$H2 & H2$\rightarrow$H1 & H1$\rightarrow$H3 & H3$\rightarrow$H1 & H2$\rightarrow$H3 & H3$\rightarrow$H2 \\ \midrule
WL  & 49.6 & 51.8 & 54.0 & 54.7 & 48.2 & 48.9 & 50.6 & 51.9 & 38.1 & 51.0 & 51.8 & 50.7\\
GCN         & 57.8\scriptsize{$\pm$1.9} & 57.4\scriptsize{$\pm$2.1} & 56.5\scriptsize{$\pm$1.3} & 54.2\scriptsize{$\pm$1.2} & 54.5\scriptsize{$\pm$1.7} & 51.9\scriptsize{$\pm$1.7} & 61.8\scriptsize{$\pm$2.0} & 55.2\scriptsize{$\pm$1.3} & 49.8\scriptsize{$\pm$1.4} & 56.6\scriptsize{$\pm$1.3} & 61.5\scriptsize{$\pm$1.1} & 56.8\scriptsize{$\pm$1.7}\\
GIN         & 59.1\scriptsize{$\pm$1.6} & 53.3\scriptsize{$\pm$1.6} & 50.0\scriptsize{$\pm$1.7} & 54.8\scriptsize{$\pm$1.8} & 56.8\scriptsize{$\pm$1.6} & 50.7\scriptsize{$\pm$1.4} & 55.2\scriptsize{$\pm$1.3} & 58.7\scriptsize{$\pm$1.4} & 57.0\scriptsize{$\pm$1.8} & 49.6\scriptsize{$\pm$1.7} & 53.1\scriptsize{$\pm$1.3} & 55.3\scriptsize{$\pm$1.6}\\
GMT         & 51.1\scriptsize{$\pm$1.6} & 57.2\scriptsize{$\pm$1.7} & 50.6\scriptsize{$\pm$1.2} & 55.9\scriptsize{$\pm$1.9} & 58.4\scriptsize{$\pm$1.4} & 53.4\scriptsize{$\pm$1.7} & 56.1\scriptsize{$\pm$1.6} & 57.6\scriptsize{$\pm$1.3} & 60.9\scriptsize{$\pm$1.7} & 54.5\scriptsize{$\pm$2.0} & 55.4\scriptsize{$\pm$1.6} & 51.2\scriptsize{$\pm$1.8}\\
CIN         & 52.3\scriptsize{$\pm$1.6} & 53.4\scriptsize{$\pm$1.5} & 54.0\scriptsize{$\pm$1.9} & 57.5\scriptsize{$\pm$1.8} & 56.5\scriptsize{$\pm$1.4} & 51.1\scriptsize{$\pm$1.9} & 55.8\scriptsize{$\pm$1.4} & 63.6\scriptsize{$\pm$1.8} & 48.7\scriptsize{$\pm$1.6} & 56.8\scriptsize{$\pm$1.6} & 49.3\scriptsize{$\pm$1.2} & 52.8\scriptsize{$\pm$1.3}\\
\bottomrule
	\end{tabular}}
    \label{tab:hiv_node}
\end{table*}

\begin{table*}[t]
\small
	\caption{The classification results (in \%) on the DD dataset under edge density domain shift (source $\rightarrow$ target).  D0, D1, D2, and D3 denote the sub-datasets partitioned with edge density. \textbf{Bold} results indicate the best performance.}\resizebox{1.0\textwidth}{!}{
	\begin{tabular}{
>{\centering\arraybackslash}m{1.6cm}|
*{12}{>{\centering\arraybackslash}m{1.1cm}}
}
		\toprule                                                      
		\textbf{Methods}    & D0$\rightarrow$D1 & D1$\rightarrow$D0 & D0$\rightarrow$D2 & D2$\rightarrow$D0 & D0$\rightarrow$D3 & D3$\rightarrow$D0 & D1$\rightarrow$D2 & D2$\rightarrow$D1 & D1$\rightarrow$D3 & D3$\rightarrow$D1 & D2$\rightarrow$D3 & D3$\rightarrow$D2
 \\ \midrule
WL subtree  & 49.2 & 56.8 & 29.6 & 20.1 & 21.0 & 18.4 & 59.5 & 50.5 & 57.3 & 48.1 & 63.9 & 66.9 \\
GCN & 48.9\scriptsize{$\pm$2.8} & 59.0\scriptsize{$\pm$1.7} & 20.7\scriptsize{$\pm$2.0} & 27.3\scriptsize{$\pm$2.3} & 15.1\scriptsize{$\pm$1.8} & 26.9\scriptsize{$\pm$2.2} & 61.6\scriptsize{$\pm$1.9} & 53.6\scriptsize{$\pm$1.5} & 68.1\scriptsize{$\pm$1.6} & 52.9\scriptsize{$\pm$2.6} & 64.9\scriptsize{$\pm$2.1} & 69.7\scriptsize{$\pm$2.3} \\
GIN & 48.8\scriptsize{$\pm$1.9} & 24.7\scriptsize{$\pm$2.1} & 44.1\scriptsize{$\pm$1.8} & 22.4\scriptsize{$\pm$2.3} & 57.0\scriptsize{$\pm$2.1} & 18.4\scriptsize{$\pm$2.0} & 73.0\scriptsize{$\pm$1.8} & 52.5\scriptsize{$\pm$2.3} & 63.2\scriptsize{$\pm$1.6} & 53.6\scriptsize{$\pm$1.5} & 70.3\scriptsize{$\pm$2.6} & 69.4\scriptsize{$\pm$1.8} \\
GMT & 49.1\scriptsize{$\pm$1.9} & 32.9\scriptsize{$\pm$2.2} & 31.8\scriptsize{$\pm$1.8} & 27.3\scriptsize{$\pm$2.3} & 52.5\scriptsize{$\pm$1.5} & 27.6\scriptsize{$\pm$1.8} & 75.4\scriptsize{$\pm$1.9} & 53.2\scriptsize{$\pm$2.1} & 74.1\scriptsize{$\pm$2.6} & 57.9\scriptsize{$\pm$2.4} & 70.9\scriptsize{$\pm$1.8} & 71.1\scriptsize{$\pm$2.7} \\
CIN & 50.4\scriptsize{$\pm$1.8} & 18.4\scriptsize{$\pm$2.0} & 21.2\scriptsize{$\pm$2.1} & 36.8\scriptsize{$\pm$1.8} & 43.0\scriptsize{$\pm$2.1} & 22.9\scriptsize{$\pm$1.9} & 53.4\scriptsize{$\pm$1.7} & 56.5\scriptsize{$\pm$1.5} & 62.3\scriptsize{$\pm$1.6} & 53.3\scriptsize{$\pm$1.9} & 75.0\scriptsize{$\pm$2.1} & 69.3\scriptsize{$\pm$2.0} \\
\bottomrule
	\end{tabular}}
    \label{tab:dd_idx}
\end{table*}

\begin{table*}[t]
\small
	 \caption{The classification results (in \%) on the FRANKENSTEIN dataset under edge density domain shift (source $\rightarrow$ target).  F0, F1, F2, and F3 denote the sub-datasets partitioned with edge density. \textbf{Bold} results indicate the best performance.}\resizebox{1.0\textwidth}{!}{
	\begin{tabular}{
>{\centering\arraybackslash}m{1.6cm}|
*{12}{>{\centering\arraybackslash}m{1.1cm}}
}
		\toprule                                                      
		\textbf{Methods}& F0$\rightarrow$F1 &F1$\rightarrow$F0 &F0$\rightarrow$F2 &F2$\rightarrow$F0 &F0$\rightarrow$F3 &F3$\rightarrow$F0 &F1$\rightarrow$F2 &F2$\rightarrow$F1 &F1$\rightarrow$F3 &F3$\rightarrow$F1 &F2$\rightarrow$F3 &F3$\rightarrow$F2
 \\ \midrule
WL subtree  & 44.6 & 54.8 & 47.0 & 56.1 & 58.2 & 41.8 & 50.8 & 58.1 & 45.7 & 47.9 & 45.3 & 42.7\\
GCN & 33.5\scriptsize{$\pm$2.1} & 40.9\scriptsize{$\pm$1.5} & 44.6\scriptsize{$\pm$1.3} & 40.1\scriptsize{$\pm$1.6} & 57.4\scriptsize{$\pm$2.0} & 43.1\scriptsize{$\pm$1.3} & 44.6\scriptsize{$\pm$2.0} & 33.5\scriptsize{$\pm$1.2} & 56.4\scriptsize{$\pm$1.1} & 36.6\scriptsize{$\pm$1.9} & 57.4\scriptsize{$\pm$2.0} & 47.8\scriptsize{$\pm$1.7}\\
GIN & 40.1\scriptsize{$\pm$1.2} & 52.4\scriptsize{$\pm$1.8} & 52.2\scriptsize{$\pm$2.2} & 52.7\scriptsize{$\pm$1.8} & 58.3\scriptsize{$\pm$1.5} & 42.8\scriptsize{$\pm$2.0} & 53.2\scriptsize{$\pm$2.3} & 47.1\scriptsize{$\pm$1.6} & 52.1\scriptsize{$\pm$2.2} & 39.1\scriptsize{$\pm$2.0} & 49.9\scriptsize{$\pm$1.8} & 43.5\scriptsize{$\pm$1.9}\\
GMT & 51.5\scriptsize{$\pm$1.0} & 57.0\scriptsize{$\pm$1.3} & 55.4\scriptsize{$\pm$1.6} & 59.0\scriptsize{$\pm$1.2} & 39.6\scriptsize{$\pm$1.6} & 44.1\scriptsize{$\pm$2.0} & 55.4\scriptsize{$\pm$2.1} & 66.5\scriptsize{$\pm$1.6} & 39.6\scriptsize{$\pm$1.9} & 36.5\scriptsize{$\pm$1.5} & 39.6\scriptsize{$\pm$1.4} & 42.8\scriptsize{$\pm$1.4}\\
CIN & 54.2\scriptsize{$\pm$1.1} & 52.2\scriptsize{$\pm$1.9} & 47.6\scriptsize{$\pm$1.8} & 52.5\scriptsize{$\pm$2.2} & 56.2\scriptsize{$\pm$1.3} & 40.0\scriptsize{$\pm$2.0} & 46.8\scriptsize{$\pm$1.3} & 59.9\scriptsize{$\pm$1.2} & 49.1\scriptsize{$\pm$1.5} & 33.5\scriptsize{$\pm$2.0} & 57.8\scriptsize{$\pm$1.4} & 45.9\scriptsize{$\pm$2.0}\\
\bottomrule
	\end{tabular}}
    \label{tab:frank_idx}
\end{table*}

\begin{table*}[t]
\small
	\caption{The classification results (in \%) on the Mutagenicity dataset under edge density domain shift (source $\rightarrow$ target).  M0, M1, M2, and M3 denote the sub-datasets partitioned with edge density. \textbf{Bold} results indicate the best performance.}\resizebox{1.0\textwidth}{!}{
	\begin{tabular}{
>{\centering\arraybackslash}m{1.6cm}|
*{12}{>{\centering\arraybackslash}m{1.1cm}}
}
		\toprule                                                      
		\textbf{Methods}    &M0$\rightarrow$M1 &M1$\rightarrow$M0 &M0$\rightarrow$M2 &M2$\rightarrow$M0 &M0$\rightarrow$M3 &M3$\rightarrow$M0 &M1$\rightarrow$M2 &M2$\rightarrow$M1 &M1$\rightarrow$M3 &M3$\rightarrow$M1 &M2$\rightarrow$M3 &M3$\rightarrow$M2
 \\ \midrule
WL subtree  & 34.4  & 39.6 & 47.6 & 43.6  & 52.7 & 46.4 & 59.8 & 40.9 & 53.0 & 65.1& 52.6  & 42.4 \\
GCN & 66.3\scriptsize{$\pm$1.7} & 65.1\scriptsize{$\pm$2.2} & 63.6\scriptsize{$\pm$1.4} & 67.6\scriptsize{$\pm$1.6} & 56.0\scriptsize{$\pm$1.4} & 49.6\scriptsize{$\pm$1.9} & 66.1\scriptsize{$\pm$2.0} & 71.8\scriptsize{$\pm$1.8} & 48.0\scriptsize{$\pm$2.3} & 47.8\scriptsize{$\pm$1.7} & 51.6\scriptsize{$\pm$1.8} & 56.6\scriptsize{$\pm$2.2} \\
GIN & 67.1\scriptsize{$\pm$1.7} & 56.1\scriptsize{$\pm$1.5} & 54.2\scriptsize{$\pm$2.6} & 65.6\scriptsize{$\pm$1.6} & 55.4\scriptsize{$\pm$1.9} & 50.8\scriptsize{$\pm$2.1} & 67.9\scriptsize{$\pm$2.1} & 69.4\scriptsize{$\pm$2.0} & 53.4\scriptsize{$\pm$1.6} & 66.7\scriptsize{$\pm$2.0} & 50.4\scriptsize{$\pm$1.1} & 55.4\scriptsize{$\pm$1.2}\\
GMT & 67.9\scriptsize{$\pm$1.3} & 63.9\scriptsize{$\pm$1.8} & 61.5\scriptsize{$\pm$1.8} & 68.5\scriptsize{$\pm$1.5} & 58.2\scriptsize{$\pm$2.4} & 54.3\scriptsize{$\pm$1.8} & 66.7\scriptsize{$\pm$1.9} & 76.1\scriptsize{$\pm$1.7} & 58.8\scriptsize{$\pm$1.5} & 44.5\scriptsize{$\pm$1.5} & 59.0\scriptsize{$\pm$1.4} & 63.7\scriptsize{$\pm$1.7}\\
CIN & 66.3\scriptsize{$\pm$1.8} & 56.7\scriptsize{$\pm$2.5} & 60.8\scriptsize{$\pm$1.7} & 67.4\scriptsize{$\pm$1.7} & 55.8\scriptsize{$\pm$2.4} & 52.3\scriptsize{$\pm$1.8} & 63.4\scriptsize{$\pm$2.3} & 69.3\scriptsize{$\pm$1.2} & 54.3\scriptsize{$\pm$1.2} & 60.7\scriptsize{$\pm$1.4} & 49.0\scriptsize{$\pm$1.3} & 54.4\scriptsize{$\pm$1.7}\\
\bottomrule
	\end{tabular}}
    \label{tab:mutag_idx}
\end{table*}

\begin{table*}[t]
\small
	\caption{The classification results (in \%) on the NCI1 dataset under edge density domain shift (source $\rightarrow$ target).  N0, N1, N2, and N3 denote the sub-datasets partitioned with edge density. \textbf{Bold} results indicate the best performance.}\resizebox{1.0\textwidth}{!}{
	\begin{tabular}{
>{\centering\arraybackslash}m{1.6cm}|
*{12}{>{\centering\arraybackslash}m{1.1cm}}
}
		\toprule                                                      
		\textbf{Methods}    &N0$\rightarrow$N1 &N1$\rightarrow$N0 &N0$\rightarrow$N2 &N2$\rightarrow$N0 &N0$\rightarrow$N3 &N3$\rightarrow$N0 &N1$\rightarrow$N2 &N2$\rightarrow$N1 &N1$\rightarrow$N3 &N3$\rightarrow$N1 &N2$\rightarrow$N3 &N3$\rightarrow$N2
 \\ \midrule
WL subtree  & 50.1 & 59.9 & 50.9 & 51.0 & 51.8 & 54.5 & 56.7 & 57.0 & 52.5 & 57.6 & 54.4 & 46.9 \\
GCN & 53.7\scriptsize{$\pm$2.1} & 68.2\scriptsize{$\pm$1.9} & 50.6\scriptsize{$\pm$1.9} & 45.8\scriptsize{$\pm$1.7} & 35.5\scriptsize{$\pm$2.4} & 28.2\scriptsize{$\pm$2.2} & 55.9\scriptsize{$\pm$1.7} & 56.8\scriptsize{$\pm$1.9} & 55.7\scriptsize{$\pm$1.5} & 53.2\scriptsize{$\pm$2.1} & 60.4\scriptsize{$\pm$1.8} & 54.8\scriptsize{$\pm$0.9} \\
GIN & 51.4\scriptsize{$\pm$1.7} & 47.5\scriptsize{$\pm$2.1} & 50.4\scriptsize{$\pm$1.6} & 32.7\scriptsize{$\pm$1.6} & 47.0\scriptsize{$\pm$2.1} & 29.4\scriptsize{$\pm$2.3} & 52.4\scriptsize{$\pm$1.9} & 55.7\scriptsize{$\pm$2.0} & 64.3\scriptsize{$\pm$1.8} & 51.1\scriptsize{$\pm$1.2} & 60.3\scriptsize{$\pm$1.3} & 53.5\scriptsize{$\pm$1.6}\\
GMT & 53.0\scriptsize{$\pm$1.9} & 58.3\scriptsize{$\pm$1.7} & 49.0\scriptsize{$\pm$1.9} & 34.3\scriptsize{$\pm$1.6} & 39.7\scriptsize{$\pm$1.5} & 27.5\scriptsize{$\pm$2.4} & 55.8\scriptsize{$\pm$1.6} & 60.4\scriptsize{$\pm$2.1} & 51.3\scriptsize{$\pm$1.0} & 51.5\scriptsize{$\pm$1.6} & 60.0\scriptsize{$\pm$2.6} & 54.5\scriptsize{$\pm$1.4}\\
CIN & 54.5\scriptsize{$\pm$2.1} & 36.8\scriptsize{$\pm$1.9} & 53.0\scriptsize{$\pm$1.3} & 32.7\scriptsize{$\pm$2.1} & 44.9\scriptsize{$\pm$1.4} & 27.4\scriptsize{$\pm$1.1} & 58.4\scriptsize{$\pm$1.6} & 54.2\scriptsize{$\pm$2.3} & 63.9\scriptsize{$\pm$1.6} & 51.2\scriptsize{$\pm$1.4} & 64.7\scriptsize{$\pm$1.1} & 53.7\scriptsize{$\pm$2.1}\\
\bottomrule
	\end{tabular}}
    \label{tab:nci_idx}
\end{table*}

\begin{table*}[t]
\small
	\caption{ROC-AUC results (in \%) on the ogbg-molhiv dataset under edge density domain shift (source $\rightarrow$ target).  H0, H1, H2, and H3 denote the sub-datasets partitioned with edge density. \textbf{Bold} results indicate the best performance.}\resizebox{1.0\textwidth}{!}{
	\begin{tabular}{
>{\centering\arraybackslash}m{1.6cm}|
*{12}{>{\centering\arraybackslash}m{1.1cm}}
}
		\toprule                                                      
		\textbf{Methods}    & H0$\rightarrow$H1 & H1$\rightarrow$H0 & H0$\rightarrow$H2 & H2$\rightarrow$H0 & H0$\rightarrow$H3 & H3$\rightarrow$H0 & H1$\rightarrow$H2 & H2$\rightarrow$H1 & H1$\rightarrow$H3 & H3$\rightarrow$H1 & H2$\rightarrow$H3 & H3$\rightarrow$H2
 \\ \midrule
 WL  & 45.4 & 54.9 & 44.9 & 46.7 & 54.4 & 52.2 & 53.7 & 46.0 & 56.1 & 45.3 & 51.1 & 50.8\\
GCN         & 57.5\scriptsize{$\pm$1.8} & 58.0\scriptsize{$\pm$1.5} & 56.1\scriptsize{$\pm$2.1} & 54.3\scriptsize{$\pm$1.2} & 54.1\scriptsize{$\pm$1.4} & 50.9\scriptsize{$\pm$1.4} & 61.2\scriptsize{$\pm$1.6} & 54.8\scriptsize{$\pm$1.9} & 54.9\scriptsize{$\pm$2.1} & 56.6\scriptsize{$\pm$1.4} & 65.4\scriptsize{$\pm$2.0} & 58.1\scriptsize{$\pm$1.6}\\
GIN         & 58.1\scriptsize{$\pm$1.1} & 55.3\scriptsize{$\pm$1.2} & 49.7\scriptsize{$\pm$1.5} & 55.2\scriptsize{$\pm$2.2} & 55.5\scriptsize{$\pm$1.9} & 50.5\scriptsize{$\pm$2.0} & 55.1\scriptsize{$\pm$1.3} & 58.8\scriptsize{$\pm$1.2} & 56.6\scriptsize{$\pm$1.6} & 50.0\scriptsize{$\pm$1.1} & 56.8\scriptsize{$\pm$1.8} & 56.2\scriptsize{$\pm$1.6} \\
GMT         & 50.7\scriptsize{$\pm$1.0} & 57.5\scriptsize{$\pm$1.7} & 51.1\scriptsize{$\pm$2.1} & 58.8\scriptsize{$\pm$1.5} & 59.8\scriptsize{$\pm$1.4} & 52.0\scriptsize{$\pm$1.4} & 57.0\scriptsize{$\pm$1.8} & 55.4\scriptsize{$\pm$1.0} & 54.0\scriptsize{$\pm$1.7} & 54.9\scriptsize{$\pm$1.0} & 58.1\scriptsize{$\pm$1.3} & 52.0\scriptsize{$\pm$1.3} \\
CIN         & 51.8\scriptsize{$\pm$1.4} & 54.1\scriptsize{$\pm$1.7} & 53.6\scriptsize{$\pm$1.4} & 57.3\scriptsize{$\pm$1.6} & 56.9\scriptsize{$\pm$1.7} & 53.9\scriptsize{$\pm$1.3} & 54.5\scriptsize{$\pm$1.7} & 63.6\scriptsize{$\pm$2.1} & 54.2\scriptsize{$\pm$1.2} & 56.9\scriptsize{$\pm$1.8} & 57.6\scriptsize{$\pm$1.3} & 58.7\scriptsize{$\pm$1.2}\\
\bottomrule
	\end{tabular}}
    \label{tab:hiv_idx}
\end{table*}

\end{document}